\documentclass[reqno,11pt,letterpaper]{amsart}

\usepackage[utf8]{inputenc}
\usepackage[T1]{fontenc}

\usepackage[top=1in, bottom=1.25in, left=1in, right=1in]{geometry}

\usepackage{amsmath,amssymb,amsfonts}
\usepackage{amsthm}
\usepackage{mathtools}
\usepackage[mathscr]{euscript}
\usepackage{latexsym}
\usepackage{bm}

\usepackage{graphicx}
\usepackage[dvipsnames]{xcolor} 
\usepackage{tikz}
\usepackage{epsfig} 
\usepackage{pgfplots}
\pgfplotsset{compat=1.18}
\graphicspath{{./images/}}

\usepackage{booktabs,tabularx,makecell,array,ragged2e}
\newcolumntype{Y}{>{\RaggedRight\arraybackslash}X}

\usepackage{subcaption}  

\usepackage{enumitem}
\usepackage{verbatim}
\usepackage{placeins}

\usepackage[noend]{algorithmic}

\usepackage{dsfont}
\usepackage{bbm}

\usepackage{nicefrac}
\usepackage{microtype}
\usepackage{url}

\usepackage[toc,page]{appendix}

\usepackage[numbers,sort&compress]{natbib}

\usepackage[colorlinks=true,linkcolor=Red,citecolor=Green]{hyperref}

\theoremstyle{plain}
\newtheorem{theorem}{Theorem}
\newtheorem{prop}{Proposition}
\newtheorem{lemma}{Lemma}
\newtheorem{corollary}{Corollary}

\theoremstyle{definition}
\newtheorem{definition}{Definition}

\theoremstyle{remark}
\newtheorem{remark}{Remark}

\DeclareMathOperator{\II}{\mathrm{I\!I}}

\newcommand{\rgrad}{\nabla_{\!g}}
\newcommand{\mb}{\mathbf}

\newcommand{\wa}{\nabla\kern-2mm\nabla}
\newcommand{\expp}{\mathrm{exp}_{p}}

\newcommand{\wass}{\mathrm{GF}}

\DeclareMathOperator*{\argmin}{arg\,min}

\newcommand{\wagrad}{\nabla \frac{\delta J}{\delta \rho}}
\newcommand{\wadel}{\frac{\delta J}{\delta \rho}}

\newcommand{\Fimp}{\mathcal{R}}
\newcommand{\jko}{\mathrm{JKO}}
\newcommand{\vol}{\mathrm{v}_{g}}
\newcommand{\jkoreg}{\chi_{\mathrm{jko}}}

\newcommand{\id}{\mathbf{I}}
\newcommand{\abs}[1]{\left|#1\right|}

\newbox\fixbox
\renewcommand{\algorithmicdo}{\setbox\fixbox\hbox{\ {} }\hskip-\wd\fixbox}

\newcommand{\gives}{\rightarrow}
\newcommand{\R}{\mathbb R}
\newcommand{\lr}[1]{\left(#1\right)}
\newcommand{\norm}[1]{\left\lVert#1\right\rVert}
\newcommand{\set}[1]{\left\{#1\right\}}

\newcommand{\tr}{\mathrm{tr}}

\title{Implicit Bias of the JKO Scheme}
\author[P. Halmos]{Peter Halmos$^*$}\thanks{$\,^*$Princeton Computer Science, ph3641@princeton.edu; PH is supported by NIH/NCI grant U24CA248453 (PI: Benjamin J. Raphael).}
\author[B. Hanin]{Boris Hanin$^\dagger$}\thanks{$\,^\dagger$Princeton ORFE, bhanin@princeton.edu;  BH is supported by a 2024 Sloan Fellowship in Mathematics, NSF CAREER grant DMS-2143754, and NSF grant DMS-2133806, and DARPA AIQ grant (HR001124S0029)}

\begin{document}

\maketitle

\begin{abstract}
Wasserstein gradient flow provides a general framework for minimizing an energy functional $J$ over the space of probability measures on a Riemannian manifold $(M,g)$. Its canonical time-discretization, the Jordan-Kinderlehrer-Otto (JKO) scheme, produces for any step size $\eta>0$ a sequence of probability distributions $\rho_k^\eta$ that approximate to first order in $\eta$ Wasserstein gradient flow on $J$. But the JKO scheme also has many other remarkable properties not shared by other first order integrators, e.g. it preserves energy dissipation and exhibits unconditional stability for $\lambda$-geodesically convex functionals $J$. To better understand the JKO scheme we characterize its implicit bias at second order in $\eta$. We show that $\rho_k^\eta$ are approximated to order $\eta^2$ by Wasserstein gradient flow on a \emph{modified} energy
\[
J^{\eta}(\rho) = J(\rho) - \frac{\eta}{4}\int_M \Big\lVert \nabla_g \frac{\delta J}{\delta \rho} (\rho) \Big\rVert_{2}^{2} \,\rho(dx),
\]
obtained by subtracting from $J$ the squared metric curvature of $J$ times $\eta/4$. The JKO scheme therefore adds at second order in $\eta$ a \textit{deceleration} in directions where the metric curvature of $J$ is rapidly changing. This corresponds to canonical implicit biases for common functionals: for entropy the implicit bias is the Fisher information, for KL-divergence it is the Fisher-Hyv{\"a}rinen divergence, and for Riemannian gradient descent it is the kinetic energy in the metric $g$. To understand the differences between minimizing $J$ and $J^\eta$ we study \emph{JKO-Flow}, Wasserstein gradient flow on $J^\eta$, in several simple numerical examples. These include exactly solvable Langevin dynamics on the Bures-Wasserstein space and Langevin sampling from a quartic potential in 1D. 
\end{abstract}

\section{Introduction}

Many problems in statistics, physics, and machine-learning can be formulated as finding a minimizer 
\begin{align}\label{eq:energy_Wasserstein_min}
\rho_* \in \argmin_{\rho\, \in P_{ac}(M)} J(\rho),\qquad J: P_{ac}(M) \to \mathbb{R}
\end{align}
for an energy $J$ on the space of probability measures $P_{ac}(M)$ over a Riemannian manifold $(M,g)$ with finite second moments and density relative to the Riemannian volume form. For sufficiently smooth energies $J$ it is often natural to study the static optimization problem \eqref{eq:energy_Wasserstein_min} by analyzing the dynamics of \emph{Wasserstein gradient flow} on $J$. That is, given an initial condition $\rho_0$, one seeks a solution to the dissipative PDE
\begin{equation}\label{eq:W-grad-flow}
\partial_t \rho_t  = \mathrm{div}_g\lr{\rho_t \nabla_g \frac{\delta J}{\delta \rho}(\rho_t)},   
\end{equation}
in which the right hand side is the steepest descent direction at $\rho_t$ for the energy $J$ with respect to the Wasserstein-2 metric. In \eqref{eq:W-grad-flow}, $\mathrm{div}_g, \nabla_g$ are the Riemannian divergence, gradient and $\frac{\delta J}{\delta \rho}(\rho_t)$ is the \textit{first variation} of the functional $J$ evaluated at $\rho_t$ (see \cite{calcvar} for a precise definition). Examples of well-known Wasserstein gradient flows for various energy functionals $J$ include both PDEs and SPDEs such as the heat equation and McKean-Vlasov equations as well as important algorithms in statistics and machine learning such as score-matching and Langevin dynamics. 

Both for the purposes of theoretical analyses of Wasserstein gradient flows and to solve \eqref{eq:W-grad-flow} numerically, a number of time discretizations
Wasserstein gradient flow \eqref{eq:W-grad-flow} have been proposed. One such prescription commonly employed in practice is to fix a finite step size $\eta >0$ and consider the forward-Euler scheme
\begin{equation}\label{eq:FE}
    \rho^{(n+1)} = \rho^{(n)} - \eta \,\nabla_{g} \cdot ( \rho^{(n)} v^{(n)} ),\qquad v^{(n)} =  \nabla_{g} \frac{\delta J}{\delta \rho} (\rho^{(n)}),
\end{equation}
which updates the density $\rho^{(n)}$ using the instantaneous velocity $v^{(n)}$. This strategy is easy to implement numerically after approximating $\rho^{(n)}$ by a sum of delta functions. However, it has significant deficiencies:
\begin{enumerate}
    \item The iterate $\rho^{(n+1)}$ may no longer belong to $P_{ac}(M)$: it is neither guaranteed to be positive nor to have unit mass.
    \item The update $- \eta \,\nabla_{g} \cdot ( \rho^{(n)} \nabla_{g} \delta J (\rho^{(n)}) )$ merely represents an additive perturbation, not a proper push-forward (i.e. of a transport map $\nabla\phi_{\sharp} \rho^{(n)}$ along a Wasserstein geodesic).
    \item The forward-Euler scheme is stable only for sufficiently small $\eta$.
    \item Forward-Euler \textit{does not} guarantee the energy-dissipation relation $J(\rho^{n+1}) \leq J(\rho^{n})$.
\end{enumerate} 
In addition to these limitations, a body of work (e.g. \cite{FE_Bad}) has identified major regularity issues in forward-Euler schemes \eqref{eq:FE} even in the regime of small step-size (cf \S \ref{sec:numerics}). In contrast, the seminal work Jordan, Kinderlehrer, and Otto (JKO) \cite{Jordan1998} initiated the study of the implicit-Euler discretization of \eqref{eq:W-grad-flow} given by
\begin{equation}\label{eq:JKO}
    \rho_{k+1}^{\mathrm{JKO}, \eta} =
    \argmin_{\rho \in P_{ac}(M)} \,\, J(\rho) + \frac{1}{2\eta} W_{g}^{2}(\rho_k^{\mathrm{JKO}, \eta}, \rho) \,\, .
\end{equation}
The update \eqref{eq:JKO} is a proximal-point problem in the squared Wasserstein-2 metric \cite{monge1781memoire, kantorovich1942english} $W_{g}^{2}(\cdot, \cdot)$, where by definition (see \cite{AGS2008}):
\begin{align*}
W_g^2(\mu, \nu):=&\inf_{\gamma \in \Gamma(\mu, \nu)}  \int_{M\times M} d_g^2(x,y) \gamma(dx,dy),
\end{align*}
$d_g(x,y)$ denotes the Riemannian distance between a pair of points $x,y \in M$ and the infimum is over all couplings between $\mu$ and $\nu$, the set of all $\gamma\in P_{ac}(M \times M)$ with marginals $\int_{M} \gamma(x,dy) = \mu(x)$ and $\int_{M} \gamma(dx,y) = \nu(y)$ denoted by $\Gamma(\mu, \nu)$.

\begin{definition}[Order-$k$ Integrator for the JKO Scheme.]\label{def:order_k_JKO}
    Suppose $J: P_{ac} \to \mathbb{R}$ is an energy functional for which the JKO Scheme is well-posed and stable with piecewise interpolation $\rho_{\lfloor t/\eta \rfloor}^{\mathrm{JKO}}$. A curve $(\rho'_t)_{t \geq 0} \subset P_{ac}$ is called an order-$k$ integrator for the JKO scheme if for any $T>0$
    \[
    \sup_{t\in [0,T]} W_2 (\rho'_t, \rho_{\lfloor t/\eta \rfloor}^{\mathrm{JKO}, \eta}) = O_T(\eta^{k}).
    \]
\end{definition}

A key result in \cite{AGS2008} is that the proximal point scheme \eqref{eq:JKO}, which is often called the JKO scheme, is a first order integrator of the Wasserstein gradient flow of $J$.

\begin{theorem}[\cite{AGS2008} Theorem 4.0.4, 11.2.1]\label{thrm:JKO_first_order}
    Suppose $J$ is proper, l.s.c., coercive, and $\lambda$-geodesically convex along generalized geodesics. Then the JKO Scheme is well-posed and stable, with piecewise interpolation $\rho_{\lfloor t/\eta \rfloor}^{\mathrm{JKO}}$. Moreover, the Wasserstein gradient flow $(\rho_{t})_{t \ge 0}$ defined by
    \[
    \partial_{t} \rho_{t} = \mathrm{div} \left(\rho_{t} \nabla \frac{\delta J}{\delta \rho} (\rho_{t}) \right),
    \]
    exists, is unique, and satisfies for any $T>0$
    \[
    \sup_{t\in [0,T]} W_2 (\rho_t, \rho_{\lfloor t/\eta \rfloor}^{\mathrm{JKO}, \eta}) = O_T(\eta).
    \]
\end{theorem}

Theorem~\ref{thrm:JKO_first_order} is usually interpreted in the forward direction: the JKO Scheme approximates the Wasserstein gradient flow with $O_T(\eta)$ accuracy. In reverse, in the sense of Definition~\ref{def:order_k_JKO}, the Wasserstein gradient flow $(\rho_{t})_{t\geq0}$ is itself an order-1 integrator for the JKO Scheme.

In this work, we take this first-order consistency estimate as a black-box: we assume that the energy functional $J$ satisfies the hypotheses guaranteeing the existence of the gradient flow, the well-posedness of the JKO Scheme, and the stability and energy-dissipation estimates of \cite{AGS2008} (e.g.\ the assumptions of Theorem~11.2.1). Under these conditions, the JKO Scheme provides an $O_{T}(\eta)$ approximation to the continuous flow. Our main result builds on this first-order relationship by identifying a second-order modified flow matching the JKO iterates to $O_{T}(\eta^{2})$.

In contrast to the forward-Euler discretization \eqref{eq:FE}, the JKO scheme \eqref{eq:JKO} exhibits many desirable theoretical properties: 
\begin{enumerate}
    \item For $J$ a $\lambda$-convex energy functional in $(P_{ac},W_{2})$ with $\lambda > 0$ the JKO scheme is \emph{unconditionally-stable}: for any $\eta > 0$ the scheme produces a sequence $(\rho_k^{_{\mathrm{JKO},  \eta}})_k \subset P_{2}(M)$ that converges in the Wasserstein metric to the unique minimizer of $J$. 
    \item The JKO scheme \eqref{eq:JKO} exhibits discrete energy dissipation:
    \[
        J(\rho_{k+1}^{\mathrm{JKO, \eta}}) + \frac{1}{2\eta}W_{2}^{2}(\rho_{k+1}^{\mathrm{JKO}, \eta}, \rho_k^{\mathrm{JKO},\eta}) \leq J(\rho_k^{\mathrm{JKO},\eta}).
    \] 
    \item The iterates of the JKO scheme converges as $\eta \gives 0$ to the unique solution of the associated Wasserstein gradient flow  for a broad class of $J$ (see e.g. Theorem 5.1 in \cite{Jordan1998}, Theorem 4.0.4 in \cite{AGS2008}, and Theorem 2 in \cite{Serfaty2011}. 
\end{enumerate}
While \eqref{eq:JKO} is a canonical method for analyzing the optimization problem \eqref{eq:energy_Wasserstein_min}, it is not constructive since it requires  solving at each iteration an infinite-dimensional optimization problem. Indeed, \eqref{eq:JKO} admits exact solutions only in special cases such as mean and covariance updates for the linear Fokker-Planck equation \cite{Halder2017} and Gaussian convolution steps for the heat-equation \cite{Burger2012, Erbar2010}. 

It is therefore unclear how the geometry of $P_{ac}(M)$ and the local structure of $J$ interact to determine the JKO updates $\rho_k^{_{\mathrm{JKO}, \eta}}$ beyond the fact that they are order $\eta$ approximations to Wasserstein gradient flow (Theorem \ref{thrm:JKO_first_order}). The central result of the present article is to provide additional insights on the structure of JKO iterates by computing the second order \textit{implicit bias} of the JKO scheme with respect to the Wasserstein gradient flow \eqref{eq:W-grad-flow}.


\subsection{Main Result: Implicit Bias of JKO}



Our main result Theorem \ref{thm:main} identifies, for any sufficiently smooth energy $J:P_{ac}(M)\gives \R$ and step size $\eta>0$ a modified energy $J^\eta: P_{ac}(M)\gives \R$ so that Wasserstein gradient flow on $J^\eta$ matches the JKO scheme updates on $J$ to order $\eta^2$, providing a richer understanding of the JKO scheme than Theorem \ref{thrm:JKO_first_order}. To state our main result we recall the following

\begin{definition}[Metric Slope]\label{eq:metric_slope}
    Let $J: P_{ac}(M) \to \mathbb{R} \cup \infty$ be an energy on $({P}_{ac}(M), W_{g})$. The \emph{metric slope} at a point $\rho \in {P}_{ac}(M)$ is the quantity
    \[
    |\partial J(\rho)| = \left(\int_{M} \left\lVert
    \nabla_g \frac{\delta J}{\delta \rho}
    \right\rVert_{g}^{2} \rho(dx) \right)^{1/2}.
    \]
\end{definition}
The squared-metric slope is a natural quantity in the analysis of Wasserstein gradient flows \cite{AGS2008, Villani2003}, as it quantifies the rate of energy dissipation $\frac{d}{dt} J(\rho_{t}) = - |\partial J(\rho_{t})|^{2}$, an equality commonly known as the \emph{energy dissipation equality} for a purely dissipative gradient flow.

\begin{theorem}[Implicit Bias of the JKO Scheme]\label{thm:main}
    Fix an energy $J: P_{2,ac}(M) \to \mathbb{R} \cup \infty$, $\rho_0\in P_{2,ac}(M)$, a step size $\eta$, and $T>0$. Denote by $\rho_k^{\mathrm{JKO}, \eta}$ the $k$-th iterate of the JKO scheme on $J$ with step size $\eta$ starting at $\rho_0^{\mathrm{JKO},\eta}=\rho_0.$ Suppose that:
    \begin{itemize}
           \item[\textbf{(A1)}] $J$ is proper and l.s.c. in $(P_2,W_2)$.
    \item[\textbf{(A2)}] The $W_{2}$-gradient flow 
    \[
    \partial_{t}\rho^{\wass}_{t} +\nabla \cdot ( -\nabla \psi(\rho_{t}^{\wass})\rho_{t}^{\wass}) =0,\qquad \psi = \frac{\delta J}{\delta \rho}
    \]
    exists and admits classical solution $\rho^{\wass}\in C_{t,x}^{2}$ that is twice differentiable on $[0,T]\times M$.
    \item[\textbf{(A3)}] The JKO Scheme is a first order integrator for the Wasserstein gradient flow on $[0,T]$. That is, there exists $C>0$ so that for all $0\leq k\leq T/\eta$, we have
    \[
    \sup_{t\eta\in [k,k+1]} W_2\lr{\rho_t^\wass, \rho_k^{\mathrm{JKO},\eta}} \leq C \eta.
    \]
    \item[\textbf{(A4)}] For all $t\in [0,T]$, the measure $\rho_{t}$ enjoys sufficiently fast decay to zero at infinity so that integration by parts is valid:
    \[
    \int_M \varphi \partial_{t} \rho_{t} = - \int_M \varphi\, \mathrm{div}(\rho_{t} v_{t}) = \int_M \langle \nabla \varphi , v_{t} \rangle \rho_{t}\qquad \forall \varphi \in C_c^\infty(M),\quad v_t = \nabla \frac{\delta J}{\delta \rho}(\rho_t)
    \]
    \item[\textbf{(A5)}] 
There exists $m\ge 2$ and an open set
\[
U \subset \Big\{\rho\in C^{m}(M): \rho\ge c_0>0,\ \int_M \rho\,d\mathrm{vol}=1\Big\}
\]
containing $\{\rho_t: t\in[0,T]\}$ such that:

(i) (First variation) For every $\rho\in U$ there exists
$\psi(\rho)\in C^{m}(M)$ with
\[
\frac{d}{d\varepsilon} J(\rho+\varepsilon\sigma)\Big|_{\varepsilon=0}
= \int_M \psi(\rho)\,\sigma\,d\mathrm{vol}
\quad \forall \sigma\in C^{m}(M),\ \int_M\sigma=0.
\]

(ii) (Second variation) The map $\rho\mapsto \psi(\rho)$ is differentiable on
$\mathcal U$: there exists a continuous linear map
$D\psi(\rho)[\cdot]$ such that
\[
D^2J(\rho)[\sigma_1,\sigma_2]
:=\int_M \big(D\psi(\rho)[\sigma_2]\big)\,\sigma_1\,d\mathrm{vol}
\]
defines a continuous symmetric bilinear form on $\{\sigma:\int\sigma=0\}$.

(iii) Along $t\mapsto\rho_t$ we have $\psi(\rho_t)\in C^{1,2}_{t,x}$ and
$\nabla\psi(\rho_t),\ \partial_t\nabla\psi(\rho_t)\in L^2(\rho_t)$.
    \end{itemize}
    Define the functional $J^\eta:P_{ac}(M)\gives \R$ via 
    \begin{equation}\label{eq:JKO-bias}
    J^\eta(\rho) := J(\rho)-\frac{\eta}{4}\abs{\partial J(\rho)}^2.
    \end{equation}
    Denoted by $\rho_t^\eta$ the Wasserstein gradient flow $\rho_t^\eta$ on $J^\eta$. Then there exists $C>0$ so that 
    \begin{equation}\label{eq:JKO-approx}
        \sup_{t\in [0,T]} W_2\lr{\rho_t^\eta, \rho_{\lfloor t/\eta \rfloor}^{\mathrm{JKO}, \eta}} \leq C \eta^2.
    \end{equation}

    \end{theorem}

\begin{remark}
    Assumptions (A1)-(A5) hold for many common energy functionals. For example, for the free-energy $J(\rho) = \int E(x) \rho(dx) + \beta^{-1} \int \rho(dx) \log\rho(x)$, $\lambda$-convexity guarantees the existence and well-posedness of the Wasserstein gradient flow and JKO Scheme, and $C^{3}$ smoothness of the potential $E$ and initial density $\rho_{0}$.
\end{remark}

Theorem~\ref{thm:main}, is a local, second-order result. It reverses the usual direction of the analysis: instead of showing the JKO Scheme approximates a given flow, it identifies a modified energy $J^{\eta}$ whose Wasserstein gradient flow coincides with the discrete JKO updates to $O_{T}(\eta^{2})$. This construction only requires local smoothness and well-posedness assumptions on $J$ and on the continuous and discrete flows; it does not explicitly require $\lambda$-geodesic convexity.

Although $\lambda$-convexity (with properness, l.s.c., and coercivity) is not needed for this second-order analysis, it is a standard sufficient condition which \cite{AGS2008} guarantees existence, uniqueness, stability, and first-order consistency of both the Wasserstein gradient flow and JKO scheme (Theorem~\ref{thrm:JKO_first_order}).

At a high level, Theorem \ref{thm:main} shows that the behavior of the backward-Euler scheme, to second order in the step size, amounts to \emph{subtracting} from the energy $J$ its squared metric slope with a weight depending on the step-size. The resulting Wasserstein gradient flow is more conservative (``sticky'') in regions of high metric slope for $J$. Intuitively, this offers stability relative to routines such as forward-Euler discretizations which are prone to overshoot minima and quickly leave regions of $P_{ac}(M)$ where $J$ is \textit{sharp} \cite{smith2021on, barrett2021implicit}. 

We prove Theorem~\ref{thm:main} in Section~\ref{sec:IB_JKO}. Before doing so, we offer a number of perspectives and examples of this implicit regularization in the discussion below. In particular:
\begin{itemize}
    \item In Section~\ref{sec:JKO_examples} we give examples of the implicit bias of JKO for many of the canonical energy functionals minimized with Wasserstein gradient flow. As we shall see below, the implicit regularization from Theorem \ref{thm:main} recovers a variety of interesting divergences / regularizers.
    \item In Section~\ref{sec:GD_IB} we discuss how Theorem \ref{thm:main} recovers previous results on the implicit bias of Euclidean gradient descent \cite{smith2021on, barrett2021implicit}.
    \item In Section \ref{sec:RGD_ImplicitBias} we provide a novel generalization of the results in \cite{smith2021on, barrett2021implicit} from gradient flows on flat a Euclidean to gradient flows on general Riemannian manifolds. 
    \item In Section~\ref{sec:free_energy_Langevin_IB} we state the implicit bias of JKO on the free-energy functional, finding the bias to correspond to two symmetric operators acting on the energy and entropy. We find that the implicit bias of the potential energy generates a kinetic energy term, while the implicit bias of the entropy generates a stabilizing \emph{quantum drift-diffusion} (See, e.g. \cite{Gianazza2008}).
    \item In Section~\ref{sec:Theoretical_Res} we summarize key steps for showing Theorem~\ref{thm:main} and provide an Eulerian reformulation of Theorem \ref{thm:main}. 
    \item Lastly, in Section~\ref{sec:numerics} we show numerical simulations exploring properties of gradient flow on the deformed objective $J^\eta$, which we term the \emph{\textrm{JKO}-Flow}. We highlight in particular its approximation to the solution of JKO and its improved regularity relative to forward-Euler.
\end{itemize}

%

\subsection{Illustrative Examples}\label{sec:JKO_examples}

Below we catalog the implicit bias 
\[
H^{\eta}(\rho):=J(\rho)-J^\eta(\rho)
\]
of the JKO scheme for several common energy functionals $J$:

\begin{itemize}[leftmargin=*]
    \item \textbf{Potential Energy.} For a smooth function $E:M\gives \R$ consider the associated potential energy functional:
\begin{align*}
J(\rho) = \int_M E(x)\rho(dx) , \qquad \frac{\delta J}{\delta \rho} = E, \qquad \nabla_g \frac{\delta J}{\delta \rho} = \nabla_g E.
\end{align*}
Hence, 
\begin{align}
    \label{eq:potential-ib}H^{\eta}(\rho)= \frac{\eta}{4} \int_M \lVert \nabla E(x) \rVert_{2}^{2} \,\rho(dx) 
\end{align}
This corresponds to the Dirichlet energy of the potential $E$ under $\rho$.
\item \textbf{Entropy.} The entropy functional, its first variation, and its Wasserstein gradient are given by
    \begin{align*}
    J(\rho) = \int_M  \rho(dx)\log \rho(x) , \quad \frac{\delta J}{\delta \rho} = \log \rho +1, \quad \nabla_g \frac{\delta J}{\delta \rho} = \nabla_g \log \rho
    \end{align*}
    The gradient-flow on $J$ corresponds to the heat equation $\partial_{t}\rho = \Delta_g \rho$ \cite{Jordan1998}, and the implicit bias of JKO is the classical Fisher information functional \cite{Gianazza2008, Amari2016}
\begin{align}
    &H^{\eta}(\rho)= \frac{\eta}{2} \int_M \frac{1}{2}\frac{\lVert \nabla_g  \rho (x)\rVert_{2}^{2}}{\rho} \,dx,
\end{align}
or equivalently the trace of the Fisher information matrix $(\eta/4) \cdot\tr I^{\rho}$  equals \cite{Amari2016}
\begin{align}
    &\frac{\eta}{4} \int_M \lVert \nabla_g \log \rho (x)\rVert_{2}^{2} \,\rho(dx) = \frac{\eta}{4}\mathbb{E}_{\rho}\left[\lVert \nabla_g \log \rho \rVert_{2}^{2}\right]
\end{align}
\item \textbf{Kullback-Leibler Divergence.} For a fixed $\pi \in P_{ac}(M)$ if $J$ is the KL divergence to $\pi$, then we have
\begin{align*}
J(\rho) = \int_M \rho(dx) \log \frac{\rho(x)}{\pi(x)}, \quad \frac{\delta J}{\delta \rho} = \log \rho- \log \pi +1, \quad \nabla_g \frac{\delta J}{\delta \rho} = \nabla_g \log \rho - \nabla \log \pi
\end{align*}
Hence, the implicit regularization \begin{align}
    &H^{\eta}(\rho) = \frac{\eta}{4} \int_M \lVert \nabla_g \log \rho(x) - \nabla_g \log \pi(x) \rVert_{2}^{2} \,\rho(dx)
\end{align}
of the JKO scheme is given by the Hyv{\"a}rinen or Fisher divergence \cite{Otto2000} of information geometry and score-matching \cite{hyvarinen05a}.
\end{itemize}

We further remark on other implicit biases of interest in Section~\ref{sec:JKO_IB_other}, including interaction energies, the Porous-Medium equation, and $K$-Wasserstein barycenter. Notably, with the implicit Euler signage the correction $-\frac{\eta}{4} |\partial J (\rho)|^2$ is \emph{concave} in $W_{2}$. A natural question, then, is to determine for which $J$ and ranges of $\eta$ the modified Wasserstein gradient flow generated by $J^{\eta}$ remains well-posed. While we compute the induced regularization of the JKO Scheme for several examples, we do not attempt here to prove full well-posedness for each case.


\subsection{Recovering the Backward Euler Implicit Bias for Euclidean GD}\label{sec:GD_IB}

Theorem \ref{thm:main} generalizes well-known results on the implicit bias of gradient descent \cite{smith2021on, barrett2021implicit}. To explain this, let us restrict to the case when $M = \R^d$ equipped with the standard Euclidean metric. Given a smooth objective function $E:\R^d \gives \R$ the gradient flow
\[
\dot{x}_t = - \nabla E(x_t)
\]
can be discretized via either a forward or backward Euler scheme:
\begin{align}
    \label{eq:forward} x_{k+1} &= x_k - \eta \nabla E(x_k) \tag{Forward Euler}\\
    \label{eq:backward} x_{k+1} &=\argmin_{x} \set{E(x) + \frac{1}{2\eta}\norm{x-x_k}^2} \tag{Backward Euler}
\end{align}
A fundamental result about the implicit bias of these discretizations is the following
\begin{theorem}[\cite{smith2021on, barrett2021implicit}]\label{thm:GD}
    Let $x_k^{GD,\eta}$ be the iterates of either the forward or backward Euler schemes above. Define
    \[
    E^\eta(x) = \begin{cases}
        E(x)+\frac{\eta}{4}\norm{\nabla E(x)}^2,&\quad \text{Forward Euler}\\
        E(x)-\frac{\eta}{4}\norm{\nabla E(x)}^2,&\quad \text{Backward Euler}
    \end{cases}
    \]
    and write
    \begin{align}\label{eq:GD_ODE}
    \dot{x}_t^\eta = - \nabla E^\eta(x_t^\eta)
    \end{align}
    for the gradient flow on $E^\eta$ starting at the same initial condition as the gradient flow for $E$. Then for any $T>0$ and all $\eta$ sufficiently small 
    \[
    \sup_{t\in [0,T]} \norm{x_t^\eta - x_{\lfloor t/\eta \rfloor}^{GD,\eta}}= O_T(\eta^2).
    \]
\end{theorem}

\begin{remark}
    The articles \cite{smith2021on, barrett2021implicit} focus on the case of gradient descent (i.e. the forward Euler scheme) but their proofs carry over almost verbatim to analyzing the backward Euler scheme as well. 
\end{remark}

\cite{barrett2021implicit, smith2021on} compute this asymptotic bias of gradient descent relative to gradient flow using backward error analysis (BEA) \cite{Reich1999, CHansen2011}. Their result can be recovered from our Theorem \ref{thm:main} by considering the case when 
\[
J(\rho) = \int_M E(x) \rho(x)
\]
is a pure potential. Indeed, as in \eqref{eq:potential-ib}, the Wasserstein gradient flow on $J^\eta$ is
\[
\partial_t \rho_t  = \mathrm{div}\lr{\rho_t \nabla \frac{\delta J^{\eta}}{\delta \rho}(\rho_t)} = -\mathrm{div}\lr{\rho_t (-\nabla E +\frac{\eta}{4} \nabla \left\lVert \nabla E\right\rVert_{2}^{2}) ,}
\]
which is the weak form of the particle dynamics \eqref{eq:GD_ODE}
\[
\dot{x}^{\eta}_{t} = -\nabla E(x^{\eta}_{t}) +\frac{\eta}{4} \nabla \left\lVert \nabla E(x^{\eta}_{t}) \right\rVert_{2}^{2}.
\]
As we shall see next, Theorem \ref{thm:main} allows us not only to recover the results of \cite{smith2021on, barrett2021implicit} but also generalize them to gradient flows on any Riemannian manifold.

\subsection{Implicit Bias of Riemannian Gradient Descent}\label{sec:RGD_ImplicitBias}

Fix a smooth, complete Riemannian manifold $(M,g)$ of dimension $d<\infty$ and a smooth energy $E:M\gives \R$. In this section we present a novel expression for the implicit bias of both the forward and backward Euler discretizations 
\[
x_{k+1} = \begin{cases}
    \exp_{x_k}\lr{-\eta \nabla_g E(x_k)},&\quad\text{explicit Euler}\\
    \argmin_{x\in M} \set{E(x) + \frac{1}{2\eta}d_g(x,x_k)^2},&\quad \text{implicit Euler}
\end{cases},
\]
of the continuous time gradient flow $ \dot{x}_t = - \nabla_g E(x_t).$ Specifically, we will give an independent derivation of the following result 
\begin{prop}[Riemannian Gradient Descent Bias]\label{prop:implicit_riemann_cov}
    Let $(M, g)$ be a smooth, complete Riemannian manifold, and let $E:M\gives \R$ be in $C^3(M)$. For $\eta>0$ define $E^\eta:M\gives \R$ by
    \[
    E^{\eta}(x) = \begin{cases}
        E(x) + \frac{\eta}{4}\norm{\nabla_g E(x)}^2,&\quad \text{forward Euler}\\
        E(x) - \frac{\eta}{4}\norm{\nabla_g E(x)}^2,&\quad \text{backward Euler}
    \end{cases}.
    \]
    Write $x_t^\eta$ for the gradient flow on $E^\eta$ and $x_k^{GD,\eta}$ for either the forward or backward Euler discretization of $x_t^\eta$ with $x_0^\eta = x_0^{GD, \eta}$. Then for every $T>0$ and all $\eta$ we have
    \[
        \sup_{t\in [0,T]}d_g(x_{t}^\eta, x_{\lfloor t/\eta\rfloor}^{GD, \eta}) = O_T(\eta^2).
    \]
\end{prop}
We provide a short proof of Proposition \ref{prop:implicit_riemann_cov} at the end of this section. While formally our arguments can be extended to infinite dimensional Riemannian manifolds such at $P_{ac}(M)$ equipped with the Wasserstein-2 metric, precise derivations in infinite dimension require more care. Before turning to the proof, we make a few remarks:

\textbf{(Novelty)} To our knowledge Proposition \ref{prop:implicit_riemann_cov} is the first generalization of \cite{smith2021on,barrett2021implicit} to general Riemannian manifolds.

\textbf{(Implicit Bias as Geodesic Acceleration)} Note that
    \[
        \nabla_g E^\eta - \nabla_g E =\pm \frac{\eta}{2}\mathrm{Hess}_g E [\nabla_g E],
    \]
    where $\mathrm{Hess}_g E$ is the Riemannian Hessian. This expression reflects that at order $\eta$ the implicit bias of forward and backward Euler is to accelerate and decelerate, respectively, gradient flow on $E$ in directions where the gradient $\nabla_g E$ is rapidly changing, as measured by the Hessian of the objective. Equivalently, taking the view of the base-flow $\tilde{x}$
    \begin{align}\label{eq:EL_Riemann}
        \nabla_g E^\eta - \nabla_g E =\pm \frac{\eta}{4} \lr{\hat{\mathcal{E}} \big\lVert\dot{\tilde{x}}\big\rVert_{g}^2}\bigg|_{\dot{\tilde{x}} = \dot{x}_t} = \pm \frac{\eta}{2}\nabla_{\nabla_g E} \nabla_g E,
    \end{align}
where $\hat{\mathcal{E}} = g^{-1}(\partial_{\tilde{x}} - D_{t} \partial_{\dot{\tilde{x}}})$ is the Riemannian Euler-Lagrange operator. Applying the operator $\hat{\mathcal{E}}$ to the kinetic energy $\frac{1}{2}||\dot{\tilde{x}}||^2$ returns the geodesic acceleration $\nabla_{\dot{\tilde{x}}} \dot{\tilde{x}}$. Thus, the implicit bias of  forward and backward Euler discretization for Riemannian gradient flow is to accelerate or slow down in directions of large geodesic acceleration.

The form of the implicit regularization in the Riemannian case, when written in local coordinates, offers an interesting parallel to the Euclidean case \cite{smith2021on}
\begin{align}
    &\frac{\eta}{4}g^{-1}(\partial_{\tilde{x}} - D_{t} \partial_{\dot{\tilde{x}}}) \left\lVert \dot{\tilde{x}} \right\rVert_{g}^{2}, \quad
    &\text{(Riemannian Implicit Bias)} \label{riem:IB_EL}\\
    -&\frac{\eta}{4} D_{t} \partial_{\dot{\tilde{x}}} \left\lVert \dot{\tilde{x}}  \right\rVert_{2}^{2} = -\frac{\eta}{4} \nabla \left\lVert \dot{\tilde{x}}  \right\rVert_{2}^{2}, \quad
    &\text{(Euclidean Implicit Bias)}\label{Euclidean:IB_EL}
\end{align}
In the Euclidean case, \eqref{Euclidean:IB_EL} adds a Ridge-like Euclidean kinetic penalty. Meanwhile, the presence of non-trivial curvature in the Riemannian case \eqref{riem:IB_EL} requires the variational Euler-Lagrange operator $(\partial_{\tilde{x}} - D_{t} \partial_{\dot{\tilde{x}}})$ applied to a Riemannian kinetic energy in the metric. The action of this operator can be viewed as a steepest descent on the action integral \cite{gill2023, calcvar} of the Riemannian kinetic energy.

Collecting the zero order potential $E$, we see that the dynamics of the Riemannian case are governed by a \emph{Step-Dependent Lagrangian},
\begin{align*}
\mathcal{L}^{\eta}(\tilde{x}, \dot{\tilde{x}}) &=  \frac{\eta}{4} g_{ij} \Dot{\tilde{x}}^{i} \Dot{\tilde{x}}^{j}- E(\tilde{x}):= \textrm{Kinetic Energy }(T) - {\textrm{Potential Energy }(E) 
}. 
\end{align*}
The modified flow is thus the steepest descent on a Lagrangian action, where the components satisfy
\[
\Dot{\tilde{x}}^{k} = \,g^{km} \,\left(
\frac{\partial}{\partial \tilde{x}^{m}}
-\frac{d}{dt}\,\frac{\partial}{\partial \dot{\tilde{x}}^{m}}
\right)\,
\mathcal{L}^{\eta}(\tilde{x}, \dot{\tilde{x}})\, \bigg|_{\dot{\tilde{x}} = \dot{x}_{t}}
+\;O(\eta^{2})
\]
This reveals the curvature correction required by discretization manifests physically as \emph{inertia}. The modified objective retains a kinetic energy, and the optimizer behaves as if the point exhibited a ``mass'' proportional to the step-size $\eta$.



\textbf{(Backward Euler Slows Near Sharp Minima)} Suppose the potential $E$ possesses a smooth manifold $\mathcal{M}\subset M$ of minima and at any $x\in \mathcal{M}$ the Hessian $\mathrm{Hess}_g E(x)$ has kernel equal to the tangent space $T_{x}\mathcal{M}$ and is strictly positive on the normal space $N_{x}\mathcal{M}$.  Such potentials  commonly arise as loss functions in overparameterized machine learning models \cite{liu2022loss}. In a tubular neighborhood of $\mathcal M$ let us introduce coordinates $s$ along $\mathcal M$ and $n$ normal to $\mathcal M$. We then have
    \begin{equation}\label{eq:M-exp}
        E(s,n) = E_0 + \frac{1}{2}\, n^\top H(s)\, n + O(\norm{n}^3),
    \qquad H(s) \succ 0,    
    \end{equation}
    where $H(s)$ denotes the positive definite restriction of $\mathrm{Hess}_g E$ to the normal space to $\mathcal M$ at $s$. A short computation gives 
    \[
        E^\eta(s,n) = E_0 + \frac{1}{2} n^\top H_{\mathrm{eff}}(s) n + O(\norm{n}^3),
    \]
    where
    \begin{equation}\label{eq:Hess-eff-def}
        H_{\mathrm{eff}}(s) :=(I\pm 2\eta H(s))H(s)
    \end{equation}
    is an \textit{effective} Hessian and $+,-$ is chosen depending on whether one is analyzing forward or backward Euler discretization, respectively. Hence, forward and backward Euler schemes accelerate or slow down  dynamics near high curvature minima respectively.

\begin{proof}[Proof of Proposition \ref{prop:implicit_riemann_cov}]
    We begin by analyzing the backward Euler discretization, and then offer a somewhat different derivation in the case of the  simpler forward discretization. Since for any fixed $x_k$ 
    \[
        \nabla_g \set{\frac{1}{2} d_g(x,x_k)^2 }= - \exp_x^{-1}(x_k), 
    \]
    the optimality condition for the iteration
    \[
         x_{k+1}=\argmin_{x\in M} \set{E(x) + \frac{1}{2\eta}d_g(x,x_k)^2}
    \]
    reads
    \[
         \exp_{x_{k+1}}^{-1}(x_k) = \eta \nabla_g E(x_{k+1}). 
    \]
    By taking norms of both sides we find
    \[
    d_g(x_k, x_{k+1}) = O(\eta).
    \]
    Denoting by $||_{x_{k+1}\gives x_k}$ the parallel transport operator from $T_{x_{k+1}}M$ to $T_{x_k}M$ and using that 
    \[
    \exp_{x_k}^{-1}(x_{k+1}) = -||_{x_{k+1}\gives x_k} \exp_{x_{k+1}}^{-1}(x_k) 
    \]
    gives
    \begin{equation}\label{eq:RGF-exact-IB}
        x_{k+1}=\exp_{x_k}\lr{ -\eta ||_{x_{k+1}\gives x_k} \nabla_g E(x_{k+1})}.
    \end{equation}
    The operator $||_{x_{k+1}\gives x_k} \nabla_g E(x_{k+1})$ is precisely the pull-back along the geodesic from $x_{k}$ to $x_{k+1}$ of the vector field $\nabla_g E$. Using standard formulas for covariant differentiation of a vector field along a curve (see Theorem 4.34 in \cite{lee2018introduction}) we have 
    \[
    ||_{x_{k+1}\gives x_k} \nabla_g E(x_{k+1}) = \nabla_g E(x_k) - \frac{\eta}{2}\nabla_{\nabla_g E} \nabla_g E(x_k) + O(\eta^2).
    \]
    Substituting this into \eqref{eq:RGF-exact-IB} gives
    \[
    x_{k+1} = \exp_{x_k} \lr{-\eta E^\eta(x_k) + O(\eta^3)}.
    \]
    Appealing to standard local-to-global estimates \cite[Theorem 3.2 II.3]{hairer2003geometric} for the difference between the gradient flow on $E^\eta$ and a one step method that matches the vector field $-\eta \nabla_g E(x_k)$ to order $\eta^3$ completes the proof in the case of the backward Euler discretization. 
    
    For the forward Euler case, we give a derivation mirroring \cite{barrett2021implicit, smith2021on} but applied to Riemannian gradient flow $\dot{x}_{t} = - \nabla_{g} E(x_{t})$. Namely, we directly solve for the correction $R(\cdot)$ required for the the second-order flow $\dot{x}_{t}^{\eta} = - \nabla_{g} E(x_{t}^{\eta}) + \eta R(x_{t}^{\eta})$, to integrate to the exponential map up to errors of size $O(\eta^{3})$. We have
    \begin{align*}
    \exp_{x_{k}}(-\eta \nabla_{g} E(x_{k})) &= x_{k} + \eta \,\dot{x}_{k}^{\eta} + \frac{\eta^{2}}{2} \ddot{x}_{k}^{\eta} + o(\eta^{2}) \\
    &= x_{k} - \eta \nabla_{g} E(x_{k}) + \eta^{2} \left( R(x_{k}) + \frac{1}{2} \ddot{x}_{k} \right) \\
    &=x_{k} - \eta \nabla_{g} E(x_{k}) + \eta^{2} \left( R(x_{k}) - \frac{1}{2} \nabla \nabla_{g} \,E \, [ \dot{x}_{k} ] \right).
    \end{align*}
    On the other hand, one may appeal directly to the form of the Taylor expansion of the exponential map \cite{Monera2013} to find
    \begin{align*}
    &\exp_{p}(-t v)^{i} = \gamma^{i}(t) = p^{i} + v^{i} t + \frac{t^{2}}{2} a^{i} + \frac{t^{3}}{3} b^{i} + \cdots 
    \end{align*}
    Where the geodesic $\gamma$ is given by varying $t$ in  $\exp_{x_{k}}(-t \nabla_{g} E(x_{k}))$ and satisfies $p^{i} = x_{k}$ and $v^{i} = -\nabla_{g} E(x_{k})$. Hence,
    \begin{align*}
    \exp_{x_{k}}(-\eta \nabla_{g}E(x_{k}))^{i} &= x_{k}^{i} - \eta\nabla_{g} E(x_{k}) - \frac{\eta^{2}}{2} \,\Gamma_{ij}^{s} \nabla_{g} E(x_{k})^{i} \nabla_{g} E(x_{k}) ^{j} \\
    &+ \frac{\eta^{3}}{6} \, \left(
    -\frac{\partial}{\partial x_{\gamma}}\Gamma^{s}_{\alpha\beta}(\theta)
    + 2\cdot \Gamma^{s}_{\alpha\delta}(\theta) \Gamma^{\delta}_{
\beta\gamma
    } (\theta)\right)\\
    &\qquad \qquad \times \nabla_{g} E(x_{k})^{\alpha} \nabla_{g} E(x_{k})^{\beta} \nabla_{g} E(x_{k})^{\gamma}
    \end{align*}
    plus higher order terms in $\eta$, where $\Gamma^{\lambda}_{\mu\nu} = (1/2) g^{\lambda\alpha}
    (\partial_{\mu} g_{\alpha\nu} + \partial_{\nu} g_{\alpha \mu} - \partial_{\alpha} g_{\mu})$ denote the Christoffel symbols of the second kind. Then, one simply cancels terms at lower order and matches at $\eta^{2}$, yielding the $O_{T}(\eta^{2})$ correction
    \begin{align*}
        &R(x_{k})^{\alpha} =  -\frac{1}{2} \left( \,\Gamma_{ij}^{\alpha} \dot{x}_{k}^{i} \dot{x}_{k}^{j} +  \nabla \dot{x}_{k}^{\alpha, \beta} \, [ (\dot{x}_{k})_{\beta} ]  \right) 
    \end{align*}
    This may alternatively be written as
    \begin{align}
        R(x_{k})= \frac{1}{4}g^{-1}(\partial_{x_{k}} - D_{t} \partial_{\dot{x}_{k}}) \left\langle \dot{x}_{k}, \dot{x}_{k} \right\rangle_{g},
    \end{align}
    coinciding with the backward-Euler derivation in covariant form, with an opposite sign (see \eqref{eq:EL_Riemann}).
\end{proof}

\subsection{Implicit Bias of JKO on the Free-Energy Functional and Langevin Dynamics}\label{sec:free_energy_Langevin_IB}

An important example of an  objective for Wasserstein gradient flow is the  \emph{free-energy functional} at inverse temperature $\beta$
\begin{align}\label{eq:Langevin_E}
    &J(\rho) = \int E(x)\rho(dx) + \beta^{-1} \int  \rho(dx)\log \rho(x), 
\end{align}
which is the weighted sum of a linear potential energy and the differential entropy. Gradient flow on $J$ is the 
Fokker-Planck equation 
\[
\partial_{t} \rho = \nabla \cdot (\rho \nabla E) + \beta^{-1} \Delta \rho,
\]
which describes the evolution of the law of the Langevin SDE
\begin{align}
    dX_{t} = \nabla E(X_{t}) dt + \sqrt{2\beta^{-1}} dW_{t}.
\end{align}
Since 
\begin{align}\label{eq:WGF_Langevin}
    \frac{\delta J}{\delta \rho} = E + \beta^{-1} \log \rho, \quad \nabla \frac{\delta J}{\delta \rho} = \nabla E + \beta^{-1} \nabla \log \rho,
\end{align}
we find from Theorem \ref{thm:main} that the implicit bias of the JKO scheme applied to the free energy \eqref{eq:Langevin_E} is
\begin{align*}
    H^{\eta}(\rho)&=  \frac{\eta}{4} \int \left\lVert \nabla E + \beta^{-1} \nabla \log \rho \right\rVert_{2}^{2} \rho(dx) \\
    &= \frac{\eta}{4} \int \lVert \nabla E \rVert_{2}^{2} \,\rho(dx)+ \frac{\eta}{2 \beta} \int \langle \nabla E , \nabla \log \rho \rangle \,\rho(dx)+\frac{\eta}{4 \beta^{2}}\int \lVert \nabla \log \rho \rVert_{2}^{2} \,\rho(dx) 
\end{align*}
Remarkably, the Riemannian gradient of $H^\eta$, which is the implicit bias on the level of velocity fields, can be rewritten as a symmetric action of two differential operators, $\nabla \Delta (\cdot)$ and $\nabla \left\lVert \nabla (\cdot) \right\rVert_{2}^{2}$, acting on the energy  $E$ and entropy $\beta^{-1} \log \rho$. As a result the JKO bias is analogous to adding an $\eta$-weighted \textit{Quantum drift-diffusion}, which is a non-local penalty on the curvature of the density (cf Remark~\ref{rem:JKO_Quantum} and \cite{Gianazza2008}). 

\begin{prop}[Implicit Bias of the JKO Scheme on the Free-Energy (Langevin).]
    Wasserstein-gradient flow on the correction of JKO for the free-energy functional \eqref{eq:Langevin_E} is given by the dynamics
    \begin{align}\label{eq:WGF_H_Langevin}
    &\nabla \frac{\delta  H^{\eta}}{\delta \rho}(\rho) = \frac{\eta}{4} \left( \nabla \lVert \nabla E \rVert_{2}^{2} -   \frac{1}{\beta^{2}} \nabla \left\lVert \nabla \log \rho \right\rVert_{2}^{2} \right)
     -\frac{\eta}{2\beta} \left( 
    \nabla \Delta ( E + \frac{1}{\beta} \log \rho)
    \right)
\end{align}
\end{prop}
\begin{proof}
$H^{\eta}(\rho)$ is a weighted combination of the Dirchlet-energy $D_{\rho}[E]$, Fisher information functional $I[\rho]$, and a cross-term measuring the alignment of $\nabla \rho$ with the potential gradient $\nabla E$,
\begin{align}
    H^{\eta}(\rho)= \frac{\eta}{4}D_{\rho}[E] + \frac{\eta}{2 \beta} \int \langle \nabla E , \nabla \rho\rangle\, dx + \frac{\eta}{4 \beta^{2}} I[\rho].
\end{align}
The first term $D_{\rho}[E]$ is a simple linear potential, yielding the variation $\frac{\delta}{\delta \rho} D_{\rho}[E] = \lVert \nabla E \rVert_{2}^{2}$ -- recalling the deterministic regularization of gradient descent. Continuing for the entropic terms, one supposes compact variations of the form $\rho_{(s)} = \rho + s \dot\phi$ with test function $\dot\phi$ on compact support or vanishing boundary $\partial \,\Omega$. In this case, one identifies the variation of the second term as the Laplacian of $E$, $\frac{\delta}{\delta \rho}\int \langle \nabla E , \nabla \rho\rangle = - \Delta E$. This arises immediately by taking $\frac{d}{ds} \int \langle \nabla E, \nabla \rho_{(s)} \rangle dx |_{s=0}$ and applying integration by parts to find $\int \langle \nabla E, \nabla \dot\phi \rangle dx = - \int \dot\phi \,\Delta E$. Finally, as we recall in \ref{proof:fv_fisher} following Remark~\ref{rem:Fisher_var} the first variation of the Fisher functional (Theorem 4.2 of \cite{Gianazza2008}) is given by
\begin{align}\label{eq:fish_FV}
    \frac{\delta I}{\delta \rho} [\rho] = -4\, \frac{\Delta \sqrt{ \rho}}{\sqrt{\rho}}. 
\end{align}
Hence, the implicit bias of JKO for Langevin dynamics is 
\begin{align*}
    \frac{\delta  H^{\eta}}{\delta \rho}(\rho) = \frac{\eta}{4}\lVert \nabla E \rVert_{2}^{2} - \frac{\eta}{2 \beta} \Delta E - \frac{\eta}{4 \beta^{2}}\, \frac{\Delta \sqrt{ \rho}}{\sqrt{\rho}} .
\end{align*}
With associated Wasserstein gradient
\begin{align*}
    \nabla \frac{\delta  H^{\eta}}{\delta \rho}(\rho) = \frac{\eta}{4} \nabla \lVert \nabla E \rVert_{2}^{2} - \frac{\eta}{2 \beta} \nabla \Delta E - \frac{\eta}{ \beta^{2}}\, \nabla \left( \frac{\Delta \sqrt{ \rho}}{\sqrt{\rho}} \right)
\end{align*}
Since one has that $\Delta \log u = \frac{\Delta u}{u} - \lVert \nabla \log u \rVert_{2}^{2}$
for $u = \sqrt{\rho}$, from the equality $4\frac{\Delta \sqrt{ \rho}}{\sqrt{\rho}} = 2 \Delta \log \rho + \lVert \nabla \log\rho \rVert_{2}^{2}$, one can show that
\[
\frac{\delta I}{\delta \rho} [\rho] =- 2 \Delta \log \rho - \lVert \nabla \log \rho \rVert_{2}^{2}.
\]
This yields a more familiar score-based Wasserstein gradient
\begin{align*}
    \nabla \frac{\delta  H^{\eta}}{\delta \rho}(\rho) = \frac{\eta}{4} \nabla \lVert \nabla E \rVert_{2}^{2} - \frac{\eta}{2 \beta} \nabla \Delta E - \frac{\eta}{ 2\beta^{2}}\, \nabla  \Delta \log \rho - \frac{\eta}{4\beta^{2}} \left\lVert \nabla \log \rho \right\rVert_{2}^{2}, 
\end{align*}
with two differential operators -- the Laplacian and the gradient of the squared-norm -- applied to the energy and score.
\end{proof}

Rewriting \eqref{eq:WGF_H_Langevin} in terms of the score and (observed) Fisher-information matrix $\hat I^{\rho}$, we also have
\begin{align*}
 \nabla \frac{\delta  H^{\eta}}{\delta \rho}(\rho) &= \frac{\eta}{4} \nabla \lVert \nabla E \rVert_{2}^{2} - \frac{\eta}{2 \beta} \nabla \Delta E - \frac{\eta}{ \beta^{2}}\,  \left(  \frac{1}{2} \nabla\,\tr\, [\nabla^{2} \log \rho] + \frac{1}{2} \nabla^{2} \log \rho\, [\nabla \log \rho] \right) \\
 &=\frac{\eta}{4} \nabla \lVert \nabla E \rVert_{2}^{2} - \frac{\eta}{2 \beta} \nabla \Delta E - \frac{\eta}{ \beta^{2}}\,  \left(  \frac{1}{2} \nabla\,\tr\, \hat I^{\rho} + \frac{1}{2} \hat I^{\rho}\, (\nabla \log \rho) \right).
 \end{align*}
The entropic terms are analogous to a non-local regularization on the scores $\nabla \log \rho$ with an observed Fisher-information precondition. Unfortunately, ~\eqref{eq:fish_FV} is not directly amenable to particle-discretizations and requires third-order differentiation.

\textbf{Remark (Stabilizing Quantum Drift-Diffusion)}\label{rem:JKO_Quantum} The Bohm Quantum potential $Q[\rho]=-\frac{\hbar^{2}}{2m}\frac{\Delta \sqrt{ \rho}}{\sqrt{\rho}}$ remarkably arises as a scaling of the first variation of the Fisher functional, a connection central to Quantum-Drift-Diffusion (QDD) \cite{Gianazza2008}. This offers insight into the correction term~\eqref{eq:fish_FV}: it is proportional to the Quantum potential, with a \emph{negative} sign. $Q[\rho]$ is a global, non-local penalty on curvature of the density which differentiates between classical and quantum mechanics: with a positive sign it is repulsive and prevents collapse of $\rho$ at zero-point energy, and with a negative sign it adds a ``quantum cohesion'' which resists entropic expansion. Viewed in this sense, the implicit bias of JKO introduces a global, non-local regularization of the curvature of the density with ``quantum-mechanical'' behavior.

\section{Strategy of Proof of Theorem~\ref{thm:main}}\label{sec:Theoretical_Res}

Theorem~\ref{thm:main} asserts for any fixed $T>0$ and sufficiently small $\eta>0$ that
\begin{equation}\label{eq:thm-main-goal}
\sup_{t\in[0,T]} W_2\big(\rho_t^\eta,\rho_{\lfloor t/\eta\rfloor}^{\mathrm{JKO},J,\eta}\big)
= O_T(\eta^2),    
\end{equation}
where $\rho_t^\eta$ denotes the Wasserstein gradient flow of the modified energy $J^\eta$ from \eqref{eq:JKO-bias}, and $\rho_k^{\mathrm{JKO},J,\eta}$ are the JKO iterates with step size $\eta$. The first key observation is that to show the global estimate \eqref{eq:thm-main-goal}, it is enough to construct a \emph{modified velocity field}
\begin{align}\label{eq:JKO_vf_ansatz}
v^\eta(\rho) = v(\rho) + \eta\,\bm j(\rho),
\qquad
v(\rho) := -\nabla_g \frac{\delta J}{\delta \rho}(\rho),
\end{align}
with the following two \textit{local} properties:

\begin{enumerate}[leftmargin=*]
    \item \emph{Stability of the associated flow.} Write $\Phi_t^\eta (\cdot) : P_{ac}(M) \times \mathbb{R}_{+} \to P_{ac}(M)$ for the time $t$ flow maps
    \[
    \Phi_t^\eta(\rho_0) =\rho_t^\eta
    \]
    of the continuity equation
    \[
    \partial_t \rho_t^\eta = -\nabla_g \cdot\big(\rho_t^\eta\,v^\eta(\rho_t^\eta)\big).
    \]
    We require $\Phi_t^\eta$ is Lipschitz in $W_2$ for short time $t=\eta$:
    \begin{equation}\label{eq:lip-flow}
    W_2\big(\Phi_\eta^\eta(\mu),\Phi_\eta^\eta(\nu)\big)
    \le (1 + L\eta)\,W_2(\mu,\nu),
    \end{equation}
    where $L$ is independent of $\eta$.
    
    \item \emph{One-step matching with JKO to order $\eta^2$.} For each $\rho$, let
    \begin{equation}
    \label{eq:JKO-def-outline}
    \mathcal J_\eta(\rho)
    := \argmin_{\rho'} \Big\{J(\rho') + \frac{1}{2\eta}W_2^2(\rho,\rho')\Big\}    
    \end{equation}
    be the one step JKO update starting from $\rho$. We require that that one step of JKO starting from $\rho$ coincides with the flow $\Phi_\eta^\eta(\rho)$ to order $\eta^2$:
    \begin{equation}\label{eq:onestep-defect}
    W_2\big(\Phi_\eta^\eta(\rho),\,\mathcal J_\eta(\rho)\big) \le C\,\eta^2
    \end{equation}
    for a constant $C$ independent of $\eta$ and $\rho$.
\end{enumerate}
It is not immediately clear from the proposed vector field \eqref{eq:JKO_vf_ansatz} whether the $\eta \,\bm{j}(\rho)$ component arises from a Wasserstein gradient of some other energy $H^{\eta}(\rho)$. We will show that the the modified vector field \eqref{eq:JKO_vf_ansatz} is actually the Wasserstein gradient of the squared metric slope 
\[
{\bm j}(\rho) = \nabla_g \frac{\delta \abs{\partial J}^2}{\delta \rho}(\rho),\qquad \abs{\partial J}^2 =\frac{1}{4} \int_M \norm{\nabla_g \frac{\delta J}{\delta \rho}}^2 \rho(dx).
\]
Let us first explain why \eqref{eq:lip-flow} and \eqref{eq:onestep-defect} imply \eqref{eq:thm-main-goal} via a simple discrete Grönwall argument. Set
\[
e_k := W_2\big(\rho_{k\eta}^\eta,\rho_k^{\mathrm{JKO},J,\eta}\big)
\]
where $\rho_t^\eta$ solves the continuity equation with velocity $v^\eta$, and $\rho_k^{\mathrm{JKO},J,\eta}$ are the JKO iterates. For $k\ge0$, let's define
\[
\tilde\rho_{(k+1)\eta} := \Phi_\eta^\eta(\rho_k^{\mathrm{JKO},J,\eta}).
\]
We have 
\begin{align*}
e_{k+1}
&= W_2\big(\rho_{(k+1)\eta}^\eta,\rho_{k+1}^{\mathrm{JKO},J,\eta}\big)\\
&\le W_2\big(\rho_{(k+1)\eta}^\eta,\tilde\rho_{(k+1)\eta}\big)
    + W_2\big(\tilde\rho_{(k+1)\eta},\rho_{k+1}^{\mathrm{JKO},J,\eta}\big)\\
&= W_2\big(\Phi_\eta^\eta(\rho_{k\eta}^\eta),\Phi_\eta^\eta(\rho_k^{\mathrm{JKO},J,\eta})\big)
    + W_2\big(\Phi_\eta^\eta(\rho_k^{\mathrm{JKO},J,\eta}),\mathcal J_\eta(\rho_k^{\mathrm{JKO},J,\eta})\big)\\
&\le (1 + L\eta)\,e_k + C\,\eta^2.
\end{align*}
Since $e_0 = 0$, solving this recurrence yields
\[
e_k \le C_T\,\eta^2
\qquad\text{whenever } k\eta \le T,
\]
for $C_T$ independent of $\eta$. Hence, 
\[
W_2\big(\rho_t^\eta,\rho_{\lfloor t/\eta\rfloor}^{\mathrm{JKO},J,\eta}\big)
\le C_T'\,\eta^2,
\]
as desired. Our regularity conditions on $J$ ensure that the ${\bm j} (\rho)$ we will identify to satisfy \eqref{eq:onestep-defect} will also be Lipschitz and hence \eqref{eq:lip-flow} holds (See~\ref{cor:W2_order_2_distance} and (A3) of \ref{prop:implicit_JKO}). To satisfy \eqref{eq:onestep-defect} the starting point is the JKO variational characterization \eqref{eq:JKO-def-outline}.

We compare this with the point $\Phi_\eta^\eta(\rho)$ obtained by evolving the continuity equation with velocity $v^\eta$ for time $\eta$ starting at $\rho$. Using the dynamic (Benamou--Brenier) formulation of $W_2^2$ and the first and second variations of $J$ and $W_2^2$ along Wasserstein geodesics, we expand the Euler--Lagrange condition for the JKO functional in powers of $\eta$ along the base Wasserstein gradient flow of $J$.

The first-order terms in $\eta$ reproduce the known fact that JKO is a first-order integrator of the gradient flow of $J$ (Theorem~\ref{thrm:JKO_first_order}); they cancel automatically. The leading nontrivial condition thus appears at order $\eta^2$. Given \eqref{eq:JKO_vf_ansatz}, the energy dissipation identity eliminates the $\eta^2$ terms integrated against the Wasserstein gradient base flow. This leaves a component involving $\eta \,\bm{j}(\rho)$ and an integral on the density-acceleration of $\rho$. We then identify
\[
\bm{j}(\rho_t') = - \wagrad (\rho'_{t})(x) - \frac{\eta}{2} \nabla \,\mathcal{H}[\phi],
\]
where, $\mathcal{H}[\phi] := \partial_{t} \phi \,\, - \frac{1}{2} \left\lVert \nabla \phi \right\rVert_{2}^{2}$ for $\phi = (\delta J/\delta \rho )(\rho)$. This corresponds to a convective acceleration or a reverse-time Hamilton-Jacobi operator \cite{reverseHJ}, (See \eqref{eq:mod_flow_JKO}). The final step is to verify 
\begin{align*}
\nabla \frac{\delta}{\delta \rho} \left( \frac{\eta}{4} \int \lVert \nabla \phi(\rho) \rVert_{2}^{2} d \rho
    \right) =
    \frac{\eta}{2} \left( -\nabla \partial_{t} \phi_{t} + \frac{1}{2} \nabla \lVert \nabla \phi_{t} \rVert_{2}^{2} \right).
\end{align*}
This is shown in \eqref{eq:JKO_PDE_to_Loss}.\\

\noindent \textbf{Outline of Rest of the Paper.} Before proceeding, we summarize the structure of the remainder of the paper and the content of the appendices. In \S \ref{sec:numerics} below, we offer numerical validation of the implicit bias of JKO. In \S \ref{sec:background_BEA} we introduce background on the framework of \emph{backward-error analysis}, which we adapt to the Wasserstein setting for showing Theorem~\ref{thm:main}. Then, in \S \ref{sec:IB_JKO} we prove Theorem~\ref{thm:main} directly, using helper lemmas which follow from standard results contained in \S \ref{sec:supp_res}. In \S ~\ref{sec:JKO_IB_other} we show the implicit bias for a number canonical energy functionals omitted from the main text as well as deriving the Bures-Wasserstein correction from its established analytic solution. In \S \ref{sec:Riemann_IR} we discuss and derive the implicit bias of Riemannian gradient descent separately and offer a Lagrangian perspective on this bias. Lastly, in \S \ref{sec:bw_numerics} we offer details on the numerical evaluations.

\section{Numerical Evaluation}\label{sec:numerics}
In this section we explore numerically in several simple settings the validity and potentially utility of the implicit bias of JKO from Theorem \ref{thm:main}. Namely, 
\begin{itemize}
    \item In \S \ref{sec:Bures-Wasserstein-intro} study the exactly solvable case of energy $J$ whose Wasserstein gradient flow is the overdamped Langevin equation on a quadratic potential in the Bures-Wasserstein space of $d$-dimensional Gaussians. We compare the exact JKO scheme, Wasserstein gradient flow on $J$, and JKO-flow (i.e. Wasserstein gradient flow on $J^\eta$). As expected, we find that the JKO-flow is a better approximation to the JKO scheme (see Figure \ref{fig:jko-bw}).
    \item In \S \ref{sec:xu-li-ex}, we revisit a 1d example from \cite{FE_Bad}, in which $J$ is the KL divergence to a distribution $\pi(x)\propto e^{-U(x)}$ for quartic $U$ and even a single step of forward Euler starting for standard Gaussian produces a distribution with no density. We show that a single forward Euler step on $J^\eta$ avoids this issue (see Figure \ref{fig:quartic}) and that a particle discretization of the JKO-flow outperforms a particle discretization of the Wasserstein gradient flow for some values of $\eta$ (see Figure \ref{fig:jko-wgf}). 
\end{itemize}


\subsection{JKO in Bures-Wasserstein Space.}\label{sec:Bures-Wasserstein-intro}

\begin{figure}[ht]
\centering\includegraphics[width=1.0\linewidth]{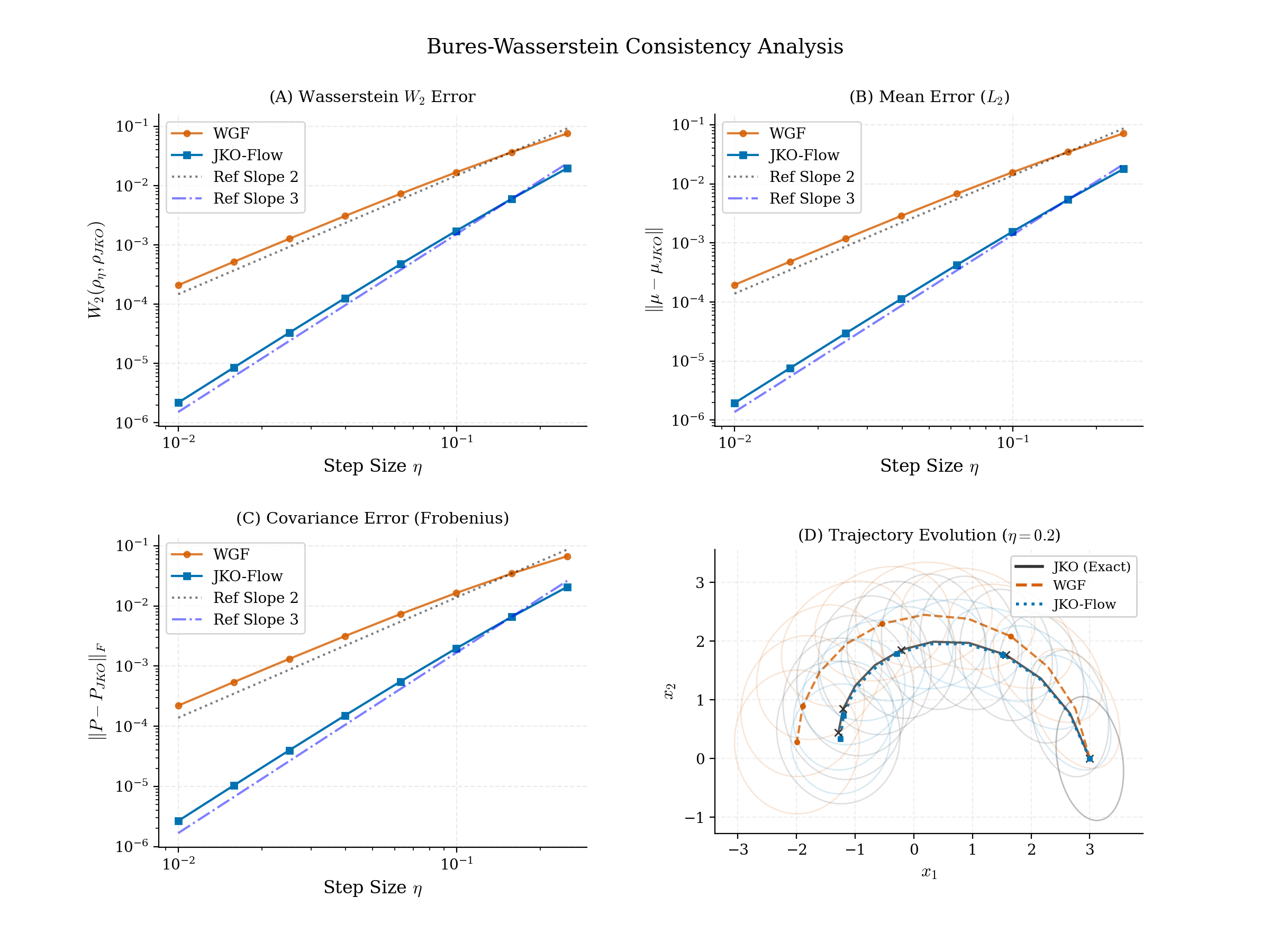}
    \caption{Wasserstein distance of $(\bm{\mu}, \Sigma)$ to analytic $(\bm{\mu}_{\mathrm{JKO}}, \Sigma_{\mathrm{JKO}})$ from \cite{Halder2017} of Wasserstein gradient flow and the modified second-order JKO-flow (Top left). Mean (Top right) and Covariance (Bottom left) error of Wasserstein gradient flow and second-order modified flow to JKO. All plots additionally have as reference the ideal $\eta^{2}$ and $\eta^{3}$ deviation curves. (Bottom right) Plot of mean and $2\sigma$-covariance isocontours of an example flow.
    }
    \label{fig:jko-bw}
\end{figure}

The case of Langevin on a Gaussian well potential, e.g. $E(x):= -x^{\top} H x$ in \eqref{eq:WGF_H_Langevin}, offers an analytically tractable test-bed. The work \cite{Halder2017} derives an analytical expression for the solution of JKO for such quadratic energies as a set of updates on a mean and covariance $(\mu, \Sigma)$ defining a Gaussian in Bures-Wasserstein space $\mathrm{BW}(\mathbb{R}^{d}) \subset P_{ac}(\mathbb{R}^{d})$. In Proposition~\ref{prop:JKO_FP}, we show the implicit bias as derived from the established analytical form for JKO in Bures-Wasserstein space \cite{Halder2017}. The exact same bias can be computed from the Langevin correction \eqref{eq:WGF_H_Langevin}, derived from our Theorem~\ref{thm:main}, in a few lines. In particular, observe in this case that from the expression
\begin{align}
    &\nabla \frac{\delta  H^{\eta}}{\delta \rho}(\rho) = \frac{\eta}{4} \left( \nabla \lVert \nabla E \rVert_{2}^{2} -   \frac{1}{\beta^{2}} \nabla \left\lVert \nabla \log \rho \right\rVert_{2}^{2} \right)
     -\frac{\eta}{2\beta} \left( 
    \nabla \Delta ( E + \frac{1}{\beta} \log \rho)
    \right)
\end{align}
that (1) the Laplacian component is zero $\nabla\Delta E = 0$, (2) by Gaussianity \cite{Halder2017} one has $\nabla \log \rho = -\Sigma^{-1}(x - \mu)$ so $\Delta \nabla \log \rho =0$, and thus
\begin{align}\label{eq:BW_wagrad}
    \nabla \frac{\delta  H^{\eta}}{\delta \rho}(\rho) &= \frac{\eta}{4} \left( \nabla \lVert \nabla E \rVert_{2}^{2} -   \frac{1}{\beta^{2}} \nabla \left\lVert \nabla \log \rho \right\rVert_{2}^{2} \right) \\
    &=  \frac{\eta}{2} H^{2} x - \frac{\eta}{2\beta^2} \Sigma^{-2}(x-\mu)
\end{align}
Thus the mean correction evolves as the expectation of \eqref{eq:BW_wagrad}
\begin{align}
   \mathbb{E}\nabla \delta  H^{\eta}(\rho) =\frac{\eta}{2} H^{2} \mu
\end{align}
and the covariance correction, computed from the classical Bures-Wasserstein (Lyapunov) differential equation \cite{chewi2024statistical}
\begin{align}
   -\dot{\Sigma} = \Sigma \, \mathbb{E}\nabla^{2}\delta J +\mathbb{E}\nabla^{2} \delta J\, \Sigma
\end{align}
is given from \eqref{eq:BW_wagrad} to be
\begin{align}
    \dot{\Sigma}_{t}[{H^{\eta}}] :&= \Sigma_{t} \left(\frac{\eta}{2} H^{2}- \frac{\eta}{2\beta^2} \Sigma_{t}^{-2} \right) + \left(\frac{\eta}{2} H^{2}- \frac{\eta}{2\beta^2} \Sigma_{t}^{-2} \right) \Sigma_{t} \\
    &= \frac{\eta}{2} \Sigma_{t} H^{2}  + \frac{\eta}{2} H^{2}\Sigma_{t} - \frac{\eta}{\beta^2} \Sigma_{t}^{-1} 
\end{align}
In addition, in Remark~\ref{rem:BW_from_PDE} we show one can also easily recover this Bures-Wasserstein flow correction from the modified velocity in the PDE form of the correction.

We numerically validate the second-order accuracy of the proposed JKO correction (Section~\ref{sec:bw_numerics}). We specialize to $M=\mathbb{R}^{3}$ and form a random negative-definite drift $A \prec 0$. We choose $A$ to be real, symmetric, and strictly negative definite so the system dynamics \eqref{eq:SDE_FokkPlanck} are stable. We also prevent alignment of $A$ with the initial covariance $P_{0}$, which is more representative of a non-commutative flow in high-dimensions. The covariance $P_{0}$ is randomly-sampled and ensured to be positive-definite $P_{0} \succ 0$. We set $\beta = 1$ in \eqref{eq:SDE_FokkPlanck} and initialize $\bm{\mu}_{0}$ as a standard Gaussian. We integrate the Wasserstein-gradient flow and the second-order flow for $200$ RK4 steps. We compare these to the JKO step in $W_2$ distance, in $\ell_{2}$ error on the mean, and Frobenius-norm error on the covariance (Figure~\ref{fig:jko-bw}). Values for this experiment are reported in Table~\ref{tab:jko_error_scaling}. We additionally demonstrate non-symmetric and rotational dynamics between the integrators (Figure~\ref{fig:jko-bw}).

In agreement with Proposition~\ref{prop:JKO_FP}, we find that the second-order JKO integrator coincides more closely to the analytical JKO step in both Wasserstein distance and in its mean-covariance parameters $(\bm{\mu}, \Sigma)$. In particular, from the Log-Log fit, we find that the second-order modified flow captures an additional order of $\eta^{0.925}$ in $W_2$, $\eta^{0.996}$ in mean-error, and $\eta^{0.995}$ in covariance-error, near the theoretical prediction of an addition order of $\eta$ captured relative to the Wasserstein gradient flow. 

\subsection{Regularity of JKO-Flow}\label{sec:xu-li-ex}
\begin{figure}
    \centering
    \includegraphics[width=0.85\linewidth]{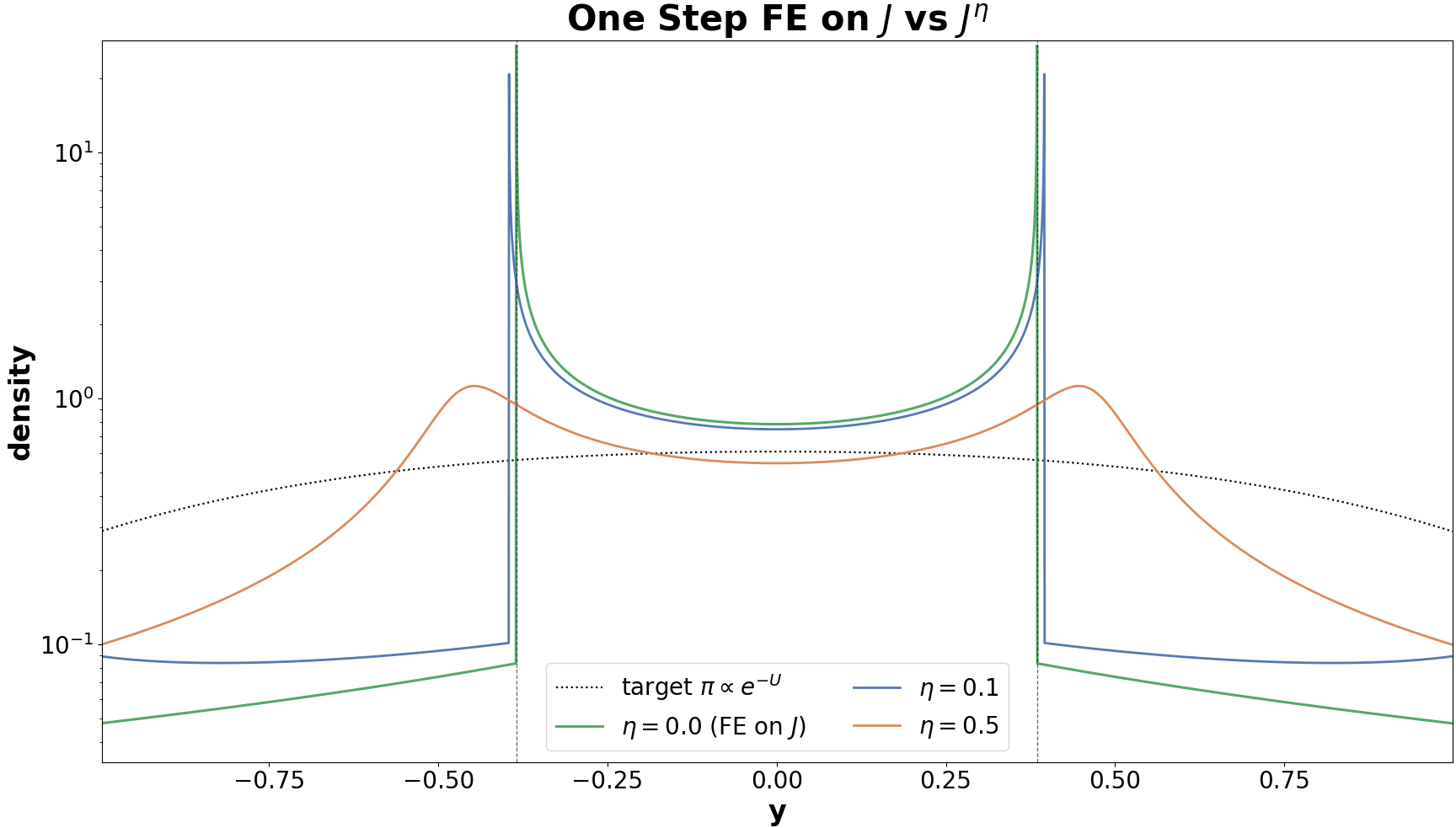}
    \caption{Densities of one step of forward Euler discretization of Wasserstein gradient flow on $J=\mathrm{KL}(\rho||\pi)$ with $-\log \pi(x) = U(x)=x^2/2+ x^4/4+\text{const}$. Comparison with  densities after one step of forward Euler on $J^\eta$ with $\eta\in \set{0.1, 0.5}$. Initial condition in all cases is the standard Gaussian. Theory predicts that the density obtained using $J^\eta$ is smooth if and only if $\eta > .3$.}
    \label{fig:quartic}
\end{figure}
In this section, we demonstrate the potential numerical benefits of using the JKO-corrected energy $J^\eta$ on a simple example from \cite{FE_Bad} known to be numerically instable for the forward Euler discretization of Wasserstein gradient flow. Namely, we take $M = \R$ and $J$ to be the KL divergence with a smooth convex quartic potential
\[
J(\rho)=\mathrm{KL}(\rho\|\pi),\qquad \pi(x)\propto e^{-U(x)},\qquad
U(x)=\tfrac12 x^2+\tfrac14 x^4.
\]
The article \cite{FE_Bad} considers one step of the forward Euler discretization of the gradient flow on $J$ starting with the standard Gaussian $\rho_0=(2\pi)^{-1/2}e^{-x^2/2}$. While the continuous time Wasserstein gradient flow  preserves smoothness and absolute continuity the forward Euler scheme does not. Indeed, a single step of size $h>0$ updates by the pushforward
\[
\rho_1=(T_h)_\#\rho_0,\qquad \rho_1^{\mathrm{FE}}(y)=\sum_{x\in T_h^{-1}(y)} \frac{\rho_0(x)}{\lvert T_h'(x)\rvert}
\]
where
\[
T_h(x)=x - h\partial_x \phi_0,\qquad \phi_0(x) = \frac{\partial J}{\partial \rho}(x,\rho_0)=x^3.
\]
Note that $T_h(x)$ is not monotone for any $h>0$ and has critical points at $x_\star=\pm(3h)^{-1/2}$. The density of $\rho_1^{FE}$ therefore has a jump discontinuity as $y$ crosses $T(x_\star)$, despite the smooth log-concave initial data. Consider in contrast one step of the forward Euler discretization on the JKO-deformed objective 
\[
J^\eta(\rho) = \mathrm{KL}(\rho|| \pi) -\frac{\eta}{4}\int_{\R} \abs{\partial_x(\log \rho(x) - U(x))}^2 \rho(dx).
\]
Setting $\rho = \rho_0$ gives
\[
\partial_x\phi_0^\eta(x)  = \partial_x\frac{\partial J^\eta }{\partial \rho}(x,\rho_0) = x^3 - \frac{3\eta}{2}x^5.
\]
Thus, the JKO-corrected transport map is
\[
T_h^\eta(x) = x - h\lr{x^3 - \frac{3\eta}{2}x^5},
\]
which is monotone provided $\eta \geq 3 h /10$. See Figure \ref{fig:quartic}. We see further numerical improvements by numerically integrating the JKO-flow (i.e. the Wasserstein gradient flow on $J^\eta$) up to time $t=h$ (see Figure \ref{fig:JKO-flow-vs-1-step}).

\begin{figure}
    \centering
    \includegraphics[width=1.0\linewidth]{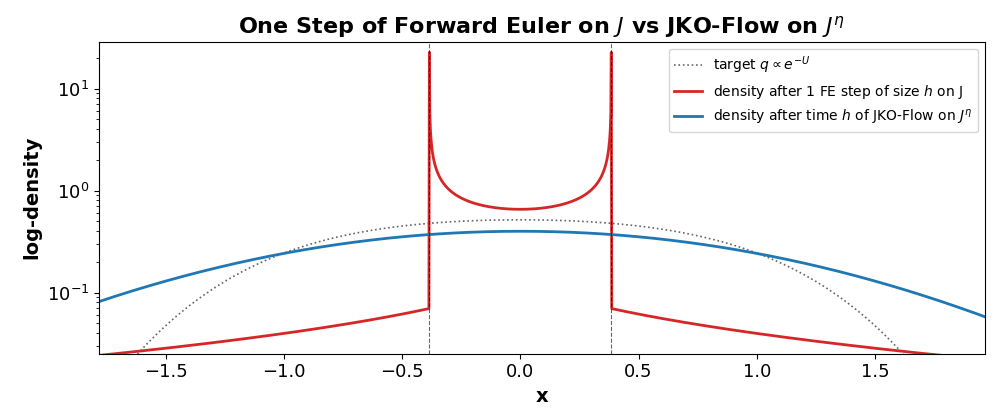}
    \caption{Comparison of 1 FE step of size $h$ on $J$ and time $t=h$ JKO-Flow on $J^\eta$ for $h=\eta=1.0$}
    \label{fig:JKO-flow-vs-1-step}
\end{figure}
Finally, we compare the performance of numerically integrating the JKO-Flow on $J^\eta$ and the Wasserstein gradient flow on $J$ for a small step size $h=2e-3$ and a range of small to moderate values of the JKO strength $\eta$ in Figure \ref{fig:jko-wgf}. We find that a range of values of $\eta>0$ provide lower $\mathrm{KL}$ divergences when compared with the vanilla $\eta \downarrow0$ Wasserstein gradient flow. Taken together we believe these results suggest that implementing the JKO bias in practice may be numerically useful in certain cases.

\textbf{Acknowledgements.} 
We are grateful to Katy Craig, Sinho Chewi, and Philippe Rigollet for comments on an earlier draft of this manuscript. BH is supported by a 2024 Sloan Fellowship in Mathematics, NSF CAREER grant DMS-2143754, and NSF grant DMS-2133806, and DARPA AIQ grant (HR001124S0029). PH is supported by NIH/NCI grant U24CA248453 (PI: Benjamin J. Raphael).
\newpage

\bibliographystyle{abbrvnat}   
\bibliography{WGF}

\appendix

\newpage

\section{Background on Backward Error-Analysis}\label{sec:background_BEA}

\emph{Backward Error-Analysis} (BEA) \cite{Reich1999} is a classical theoretical framework for analyzing the numerical error of integration schemes for differential equations. In particular, let $V \in \mathfrak{X}(\mathbb{R}^{d})$ be a vector field on $\mathbb{R}^{d}$ and suppose one considers the following ordinary differential equation for $t\in[0,1]$ and its integral solution (flow-map)
\begin{align}
    \dot{\theta}(t)  = V_{\theta(t)}, \quad \Phi_{V}^{t}(\theta):\, \mathbb{R}^{d} \times [0,1] \to \mathbb{R}^{d}
\end{align}
and define a discrete integrator $\hat{\Phi}_{V}^{\eta}$ \cite{CHansen2011} to be a function smooth in the step-size $\eta > 0$ satisfying $\hat{\Phi}_{V}^{0} = \mathrm{id}$ and $\frac{d}{d\eta}\hat{\Phi}_{V}^{\eta} \mid_{\eta=0}\,(p) = V_{p}$. An integrator is consistent with the flow to order $p$ if for sufficiently small $\eta$ if there exists a constant $C > 0$ so that
\begin{align}
\left\lVert
\hat{\Phi}_{V}^{\eta}(\theta) - \Phi_{V}^{\eta}(\theta)
    \right\rVert \leq C \, \eta^{p+1}
\end{align}
In BEA, one supposes access to solution to the discretized integration scheme $\hat{\Phi}_{V}^{\eta}(\theta)$ up to a time $(n+1)\eta=t_{n+1}$ from an initial condition at $n\eta=t_{n}$. For a given order $p$, BEA asks whether there exists a modified vector field $\Tilde{V} \in \mathfrak{X}(\mathbb{R}^{d})$ perturbing $V$ which ensures the continuous flow and discrete integrator coincide: \begin{align}\label{eq:backward_error_integrator_diff}
\left\lVert
\hat{\Phi}_{V}^{\eta}(\theta) - \Phi_{\Tilde{V}}^{\eta}(\theta)
\right\rVert \leq C \,\eta^{p+1}\end{align}
Since $\hat{\Phi}_{V}^{\eta}(\theta) \sim \Phi_{V}^{\eta}(\theta)$ differ to order $\eta^{2}$, any integrator is consistent to order one. While this is the case for gradient-descent and gradient-flow, for any finite $\eta$ the two differ beyond order one, with \cite{barrett2021implicit} identifying that the modified flow $\Tilde{V}_{\theta(t_{n}+\eta)}$ coinciding with gradient-descent to order two is given by
\begin{align}\label{eq:Corrected_ODE}
\Tilde{V} := 
{V} + \eta G:= - \nabla_{\theta} J - \frac{\eta}{4} \nabla_{\theta}\left\lVert \nabla_{\theta} J \right\rVert_{2}^{2} + O(\eta^{2})
\end{align}
Thus, while gradient-flow follows $- \nabla_{\theta} J$, forward-Euler gradient-descent behaves like the modified continuous flow $\Tilde{V}_{\theta(t_{n}+\eta)}$, revealing an implicit bias towards flatter minima.

\section{Implicit Bias of the JKO Scheme}\label{sec:IB_JKO}

Before we state the next lemma, let us recall Gr\"{o}nwall's inequality, which states that for a differentiable $\phi:[0,\eta] \to \mathbb{R}$ which satisfies $\dot{\phi}(t) \leq c \phi(t)$ for all $t\in [0,\eta]$, then $\phi(t) \leq e^{c t} \phi(0)$ for all $t$.

\begin{lemma}[Stability Lemma]\label{lem:stability}
Suppose we have two densities $\rho_{t}', \,\,\rho_{t}$ satisfying the continuity equation from the same initial datum $\rho_{0}$
\begin{align*}
    & \partial_{t}\rho_{t}' = - \nabla \cdot (\rho'_{t} v_{t} ), \quad \partial_{t}\rho_{t} = - \nabla \cdot (\rho_{t} w_{t} )
\end{align*}
Suppose that $v_{t},w_{t}$ satisfy for $t \in [0, \eta]$ both $L^{2}(\rho_{t})$ closeness a.e. $t$ and that $w_{t}$ (or $v_{t}$) are $L$-Lipschitz
\begin{align*}
    \lVert v_{t} - w_{t} \lVert_{L^{2}(\rho_{t})}  \leq C \eta^{k}, \quad \lVert w_{t}(x) - w_{t}(y) \rVert_{2} \leq L \lVert x - y \rVert_{2}
\end{align*}
Then, for any $t_{0} \in [0, \eta]$ one has the bound
\begin{align*}
    W_{2} (\rho_{t}, \rho_{t}') \leq \left( W_{2} (\rho_{t_0}, \rho_{t_0}' )  + \frac{C\eta^{k}}{L} \right) e^{L(t-t_{0})} - \frac{C\eta^{k}}{L}
\end{align*}
\end{lemma}

\begin{proof}
From Theorem 8.4.7 \cite{AGS2008}, we have from the generic differentiability of the Wasserstein-p distance that for $\rho_{t}$ absolutely continuous and $v_{t}(x) \in T_{x}{P}_{ac}(\mathbb{R}^{d})$ its tangent that
\begin{align*}
    \frac{d}{dt} W_{2}^{2} (\rho_{t}, \sigma ) \,\,  = 2 \iint \langle x - y, v_{t}(x) \rangle d\gamma.
\end{align*}
From \cite{AGS2008} one has for our pair of measures $\rho_{t}, \rho_{t}'$ that
\begin{align*}
    &\frac{1}{2}\frac{d}{dt} W_{2}^{2} (\rho_{t}, \rho_{t}') \leq \iint \langle v_{t}(x) - w_{t}(y)
    , x - y \rangle   d\gamma_{t}.
\end{align*}
The right hand side of the preceding expression can be rewritten 
\begin{align*}
    \iint \langle v_{t}(x) - w_{t}(y) + w_{t}(x) - w_{t}(x)
    , x - y \rangle   d\gamma_{t}.
\end{align*}
By Cauchy-Schwarz we thus obtain 
\begin{align*}
   \frac{1}{2}\frac{d}{dt} W_{2}^{2} (\rho_{t}, \rho_{t}') &=
    \iint \langle v_{t}(x) - w_{t}(x), x - y \rangle   d\gamma_{t}   + \iint \langle w_{t}(x) - w_{t}(y) , x - y \rangle   d\gamma_{t} \\
    &\leq
    W_{2} (\rho_{t}, \rho_{t}' ) \left( \left( \int \lVert v_{t}(x) - w_{t}(x) \rVert_{2}^{2} d\rho_{t}  \right)^{1/2}
     + \left( \iint \lVert w_{t}(x) - w_{t}(y) \rVert_{2}^{2} d\gamma_{t}  \right)^{1/2} \right) 
\end{align*}
Using that $w_{t}(x)$ is $L$-Lipschitz we find
\begin{align*}
   \frac{1}{2}\frac{d}{dt} W_{2}^{2} (\rho_{t}, \rho_{t}') &\leq   W_{2} (\rho_{t}, \rho_{t}' )\left(
    \lVert v_{t} - w_{t} \rVert_{L^{2}(\rho_{t})} + L \left( \iint \lVert x - y \rVert_{2}^{2} d\gamma_{t}  \right)^{1/2} \right) \\
    &= W_{2} (\rho_{t}, \rho_{t}' )\left(
    \lVert v_{t} - w_{t} \rVert_{L^{2}(\rho_{t})} + L W_{2} (\rho_{t}, \rho_{t}' ) \right)
\end{align*}
Thus, the time-change of the Wasserstein distance is bounded above by
\begin{align*}
    \frac{1}{2}\frac{d}{dt} W_{2}^{2} (\rho_{t}, \rho_{t}') \leq 
    C \eta^{k}  W_{2} (\rho_{t}, \rho_{t}' ) + L W_{2}^{2} (\rho_{t}, \rho_{t}' ). 
\end{align*}
Thus, letting $\phi(t):=W_{2}(\rho_{t}, \rho_{t}') $ we have
\begin{align*}
    &\frac{1}{2}\frac{d}{dt} \phi(t)^{2} =  \phi \dot{\phi}
    \leq C\eta^{k} \,\phi(t) + L \phi(t)^{2} .
\end{align*}
Supposing the non-degenerate case where $\{ t>0 : W_{2}^{2} (\rho_{t}, \rho_{t}') > 0 \} = (0,1]$ we have
\begin{align*}
    \dot{\phi} \leq C\eta^{k} + L \phi
\end{align*}
Since this differential equation is solved, for $C\eta^{k} ,\,L > 0$, by 
\begin{align*}
    \phi(t) = \left(
    \phi(t_{0}) + \frac{C \eta^{k}}{L}
    \right)e^{L (t-t_{0})} - \frac{C\eta^{k}}{L}
\end{align*}
we find the following linear differential inequality by application of Gr\"{o}nwall's
\begin{align*}
    &W_{2} (\rho_{t}, \rho_{t}') \leq \left( W_{2} (\rho_{t_0}, \rho_{t_0}' )  + \frac{C\eta^{k}}{L} \right) e^{L(t-t_{0})} - \frac{C\eta^{k}}{L}.
\end{align*}
\end{proof}
\begin{corollary}[Order of 2-Wasserstein Distance Between Order-k Integrators for step size $\eta$.]\label{cor:W2_order_2_distance}
    Suppose one defines a flow satisfying the conditions of Lemma~\ref{lem:stability} with an order-$\eta^{k}$ velocity gap. For a step of size $\eta >0$, one has that the 2-Wasserstein distance is of order $\eta^{k+1}$:
    \begin{align}
        W_{2} (\rho_{t}, \rho_{t}') = O(\eta^{k+1})
    \end{align}
\end{corollary}
\begin{proof}
    When $t_{0}=0$, observe we have $W_{2} (\rho_{0}, \rho_{0}' )= W_{2} (\rho_{0},\rho_{0}) = 0$, and by continuity have
\begin{align*}
    \lim_{t \downarrow 0} W_{2}(\rho_{t}, \rho_{t}') \leq  \frac{C\eta}{L} (e^{L\cdot0} - 1) = 0
\end{align*}
Thus, at $t=0$ we find $W_{2} (\rho_{0}, \rho_{0}') = 0$ and for $t>0$ have that $W_{2} (\rho_{t}, \rho_{t}')$ is bounded above by an exponential in $L$. Expanding in $t$ we find
\begin{align*}
    0 &\leq W_{2} (\rho_{t}, \rho_{t}') \leq  \frac{C\eta^{k}}{L} (e^{Lt} - 1) = \frac{C\eta^{k}}{L} (1 + (Lt) + \frac{(Lt)^{2}}{2} + o(t^{2}) - 1) = \frac{C \eta^{k}}{L}(Lt + \frac{(Lt)^{2}}{2} + o(t^{2}))
\end{align*}
So that at $t=\eta$ one has
\begin{align*}
    W_{2} (\rho_{t}, \rho_{t}') \leq C \eta^{k+1} + o(\eta^{k+1}):= O(\eta^{k+1}).
\end{align*}
\end{proof}

\noindent Thus, the 2-Wasserstein distance $W_{2} (\rho_{t}, \rho_{t}')$ of a perturbed flow with velocity gap of order $\eta$ is at order $\eta^{2}$.

\begin{prop}[Implicit Bias of the Variational JKO Scheme]\label{prop:implicit_JKO}

Suppose $n \in [N-1]$ and step size $\eta>0$. Let energy functional $J: {P}_{ac}(\mathbb{R}^{d}) \to \mathbb{R}$ and denote $\psi(\rho_{t}) := \wadel$. Assume A1-A5 as in Theorem~\ref{thm:main}.
Then, there exists a modified flow $\rho_t^\eta$ with boundary condition $\rho_t^\eta(n\eta)=\rho_{n}^{\mathrm{JKO}, \eta}$ satisfying the corrected continuity equation
    \begin{equation}
        \partial_{t} \rho_{t}^{\eta} = - \nabla \cdot \left( \rho_{t}^{\eta} \left(
    - \wagrad (\rho_{t}^{\eta}) + \frac{\eta}{2} \nabla \,\mathcal{H} \psi
    \right) \right)
    \end{equation}
Where there exists $C>0$ so that 
\begin{equation}\label{eq:JKO-approx}
        \sup_{t\in [0,T]} W_2\lr{\rho_t^\eta, \rho_{\lfloor t/\eta \rfloor}^{\mathrm{JKO}, \eta}} \leq C \eta^2.
\end{equation}
Here $\mathcal{H} \psi$ denotes the reverse-time Hamilton-Jacobi operator \cite{reverseHJ} applied to $\psi =\wadel$, $\mathcal{H}\psi := \partial_{t} \psi \,\, - (1/2) \left\lVert \nabla \psi \right\rVert^{2}$.

\end{prop}

\begin{remark}
    We note that requiring $J$ to be $\lambda$-geodesically convex (with sufficient coercivity) in addition to (A1) is not necessary for this local result to hold. However, if one does assume $\lambda$-convexity, the theory of metric gradient flows \cite{AGS2008} guarantees $J$ has a unique, well-behaved Wasserstein gradient flow, and that the JKO scheme is well-posed, stable, and converges to this flow. These properties ensure the regularity assumptions (A2) and (A4).
\end{remark}



\begin{proof}
Consider the discrete Jordan-Kinderlehrer-Otto step from iteration $\rho^{(n)}$ to $\rho^{(n+1)}$ iterations with a chosen step-size $\eta: \, \eta > 0$
\begin{align*}
\rho^{(n+1,\eta)} =:
\argmin_{\rho \in {P}_{ac}(\mathbb{R}^{d})} \,\, J(\rho) + \frac{1}{2 \eta} W_{2}^{2}(\rho^{(n, \eta)}, \rho)
\end{align*}
By Assumption (A1), the functional $J(\rho)$ is proper and lower semi-continuous, guaranteeing existence of a discrete minimizer $\rho^{(n+1,\eta)}$. Recall, $\rho^{\wass}$ has time-evolution given by the continuity equation following the gradient of the first variation of the functional $J: {P}_{ac}(\mathbb{R}^{d}) \to \mathbb{R}$:
$$
\partial_{t} \rho_{t} = - \nabla \cdot (\rho_{t} \mathbf{v}_{t}) = + \nabla \cdot \left(\rho_{t} \nabla \frac{\delta J (\rho)}{\delta \rho} \right)
$$
For $t = \eta n$ let the continuous solution have boundary $\rho^{(n,\eta)}$ from step $n$ of JKO. Let us take the Ansatz that $\rho'((n+1)\eta)$ coincides with the discrete flow $\rho^{(n+1,\eta)}$ up to order 2, so that
\begin{align*}
    &\rho^{(n,\eta)}(\cdot) = \rho'(\cdot, t:=\eta n), \quad \rho^{(n+1,\eta)}(\cdot) = \rho'(t+\eta)(\cdot) + O(\eta^{2}) 
\end{align*}
By Assumption (A2), the reference Wasserstein gradient flow admits a classical solution in $C^{2}_{t,x}$. We presume the corrected flow inherits this regularity by our Ansatz, justifying a Taylor expansion of the corrected flow which is partial in time $t$ \[\rho'(t+\eta, \cdot) = \rho'(t, \cdot) + \eta \frac{\partial }{\partial t}\rho' + \frac{\eta^{2}}{2} \frac{\partial^{2} }{\partial t^{2}} \rho' + o(\eta^{2})
\] 
with an $o(\eta^{2})$ perturbation $ \xi:=\rho^{(n+1,\eta)} - \rho'(t+\eta, \cdot)$ in the space of signed measures $\xi \in M(\mathbb{R}^{d})$. We seek a modified flow $\rho'$ which, when solved, matches the JKO solution $\rho^{(n+1,\eta)}$ to $o(\eta^{2})$. Essentially, this flow defined in an Eulerian sense with an expansion with respect to partial time derivatives $\partial_{t}^{\alpha}$, with each point in the domain $x\in \mathbb{R}^{d}$ viewed as fixed.

Following \cite{barrett2021implicit}, we presume an expansion exists on the time-partial of the modified flow $\partial_{t} \rho'_{t}$, in orders of $\eta$, which expresses it in terms of the continuous flow on $\rho$, so that
\[\partial_{t} \rho'_{t} = \partial_{t} \rho^{\wass}_{t} + \eta \jkoreg + \eta^{2} g_{2}(\rho) + o(\eta^{2}) 
\]
As we have the boundary condition that $\rho'(t) = \rho^{(n,\eta)}$ and $\int \rho^{(n,\eta)}(dx) = 1$ it follows that $\rho'$ must conserve mass for every $s \geq t = \eta n$. Let the Ansatz for the $O(\eta^{2})$ perturbation be given as $\xi = \jkoreg + \frac{1}{2} \partial_{t}^{2} \rho^{\wass}$
\begin{align*}
    \rho^{(n+1,\eta)} &= \rho^{(n,\eta)} + \eta \left( \partial_{t} \rho^{\wass}_{t} + \eta \jkoreg + \eta^{2} g_{2}(\rho) + o(\eta^{2}) \right) + \frac{\eta^{2}}{2} \frac{\partial^{2} }{\partial t^{2}} \rho' + o(\eta^{2}) \\
    & = \rho^{(n,\eta)} + \eta \partial_{t} \rho^{\wass}_{t} + \eta^{2} \left( \jkoreg + \frac{1}{2} \frac{\partial^{2} }{\partial t^{2}} \rho^{\wass} \right) + o(\eta^{2})
\end{align*}
Since $\rho_{t}'$ must satisfy mass-conservation itself, it holds that
\begin{align*}
&\int \partial_{t} \rho_{t}' (dx)= 0,\quad \int \left(
\partial_{t} \rho_{t} + \eta \jkoreg(\rho_{t}) + \eta^{2} g_{2}(\rho_{t}) + \cdots
\right) (dx) = 0 \,\, + \eta\int \jkoreg(\rho_{t})(dx) +o(\eta^{2})
\end{align*}
Thus, $\eta\int \jkoreg(\rho_{t}) \simeq_{o(\eta^{2})} 0$ so that the perturbation $\jkoreg(\rho_{t})$ is mass-conserving to this order. Recall by Assumption (A4) that $\rho_{t}$ vanishes at the boundaries, and that by Stokes' theorem, $\int_{M} \nabla \cdot (\rho \bm{j}) \mathrm{d} \mathrm{\mb{vol}} = \int_{\partial M} \langle \rho \bm{j}, \bm{n} \rangle\, \mathrm{d}A = 0$, so that a perturbation of the form $- \nabla \cdot (\rho \bm{j})$ is naturally feasible. Moreover, in the Otto-calculus any such zero-mass variation may be formally identified with an element of the Wasserstein tangent space $\bm{j} \in T_{\rho}P_{ac}(\mathbb{R}^{d}) = \overline{\{ \nabla\psi: \psi \in C^{\infty}_{0}\}}^{L^{2}(\rho)}.$ Thus, we may naturally take the Ansatz for this measure-preserving flow to be a continuity equation driven by a residual vector field $\bm{j}_{t}$. This captures the deviation between the Jordan-Kinderlehrer-Otto solution \cite{Jordan1998, Halder2017} from the base Wasserstein gradient flow \cite{AGS2008}, $\jkoreg(\rho_{t}) = - \nabla \cdot (\rho_{t} \bm{j}_{t})$ to ensure $\int \jkoreg(\rho_{t}) dx = 0$.

Let $\bm{j} : [0,1] \times \mathbb{R}^{d} \to \mathbb{R}^{d}$ denote the time-varying vector field driving the $O(\eta)$ JKO residual correction. By linearity of the divergence operator, we may group the base Wasserstein gradient-flow velocity and this correction field. In particular, expanding the time derivative of the modified flow $\rho'$ we find
\begin{align*}
\partial_{t} \rho_{t}' &= \partial_{t} \rho^{\wass}_{t} + \eta \jkoreg(\rho_{t}^{\wass}) + O(\eta^{2})\\
&= -\nabla \cdot \left(\rho_{t}^{\wass} \left( -\wagrad (\rho_{t})(x) \right)
\right)\, - 
\eta \,\, \nabla \cdot ( 
\rho_{t}^{\wass} \bm{j}_{t} ) + O(\eta^{2})
\\
&\simeq_{o(\eta^{2})} - \nabla \cdot \left[\rho^{\wass}_{t} \left( -\wagrad (\rho_{t}^{\wass})(x) + \eta \, \bm{j}_{t} 
\right) \right]\,\, + O(\eta^{2}) \\
&=  - \nabla \cdot \left(\rho^{\wass}_{t} \mb{v}^{\eta}_{t} \right) + O(\eta^{2}).
\end{align*}
Where $\mathbf{v}^{\eta}_{t} := -\wagrad (\rho'_{t}) + \eta \bm{j}_{t}$ represents the modified velocity field whose continuous gradient flow coincides with the discrete JKO solution up to local error $O(\eta^2)$. 

Observe that for a fixed $x$ in our domain $|\,\rho'_{t}(x) - \rho_{t}^{\wass}(x) \,|  = O(\eta^{2}) \,\,$ by the Taylor series expansion in orders of $\eta$. Additionally, given Assumption (A3), we may suppose there exists $B \in \mathbb{R}$ such that
\[
\int \lVert -\nabla \delta J/\delta \rho - \mb{v}^{\eta}_{t} \rVert_{2}^{2}d\rho_{t} = 
\eta^{2} \int \left\lVert \bm{j}_{t} \right\rVert_{2}^{2}d\rho_{t} < \eta^{2} B^{2}.
\]
Following Lemma~\ref{lem:stability} and Corollary~\ref{cor:W2_order_2_distance}
$L^{2}(\rho_{t})$ order-$\eta$ closeness of the vector fields implies, integrating up to time $t=\eta$, the following upper-bound:
\[
W_2 (\rho'_t , \rho^{\wass}_t) \leq B\eta^{2} + o(\eta^{2})
\]
when initialized from a common datum of $\rho_{0}:=\rho^{(n,\eta)}$. Thus, from the continuity equation we find that we may notate with respect to the reference flow ($\rho_{t}^{\wass} \to \rho_{t}'$) without loss up the order under consideration
\begin{align}
    - \nabla \cdot ( \rho_{t}' \mathbf{v}_{t}^{\eta}) &= - \nabla \cdot \left(\left(\rho_{t}^{\wass} + O(\eta^{2}) \right) \left(- \wagrad (\rho_{t})(x) + \eta \bm{j}_{t} \right) \right) \\
    &= - \nabla \cdot \left(
    - \wagrad (\rho_{t})(x)\rho_{t}^{\wass} + \eta \bm{j}_{t} \rho_{t}^{\wass} +O(\eta^{3}) \bm{j}_{t} + \wagrad (\rho_{t})(x)O(\eta^{2}) \right) \\
    & \simeq_{O(\eta^{2})} - \nabla \cdot \left(
    - \wagrad (\rho_{t})(x)\rho_{t}^{\wass} + \eta \bm{j}_{t} \rho_{t}^{\wass} \right) \\
    &= - \nabla \cdot \left( \rho_{t}^{\wass} \left(
    - \wagrad (\rho_{t})(x) + \eta \bm{j}_{t} \right) \right)
\end{align}
So that for this order, we may interchangeably work with the flow on $\rho_{t}'$ or $\rho_{t}^{\wass}$ as appropriate, as either are sufficient to identify the next order $\bm{j}_{t}$ of the correction.

In particular, since we have access to the Wasserstein gradient flow measure $\rho_{t}^{\wass}$ which follows $\partial_{t} \rho_{t}^{\wass} = + \nabla \cdot (\rho_{t}^{\wass}  \wagrad (\rho_{t})(x))$ through the tangent vector $-\wagrad (\rho_{t})(x) \in T_{\rho_{t}} {P}_{ac}(\mathbb{R}^{d})$, it suffices to work with this measure to identify the implicit correction term of the JKO scheme $\jkoreg$ and its velocity $\bm{j}$. Our aim is to match orders of $\eta$ in our derivation to find the modified field to first order. 

Now, let us evaluate the JKO loss at the optimal solution at step $(n+1)$ and take the expansion of $\rho^{(n+1,\eta)}$ 
\begin{align*}
 \Phi_{\eta} &= J(\rho^{(n+1,\eta)} ) + \frac{1}{2 \eta} W_{2}^{2}(\rho^{(n, \eta)}, \rho^{(n+1,\eta)} ) \\
 & = J(\rho^{(n+1,\eta)} ) + \frac{1}{2 \eta} W_{2}^{2}\left(\rho^{(n, \eta)}, \rho^{(n,\eta)} + \eta \frac{\partial }{\partial t}\rho' + \frac{\eta^{2}}{2} \frac{\partial^{2} }{\partial t^{2}} \rho' + \xi \right) \\
 & = J(\rho^{(n+1,\eta)} ) + \frac{1}{2 \eta} W_{2}^{2}\left(\rho^{(n, \eta)}, \rho^{(n,\eta)} + \eta \frac{\partial }{\partial t} \rho^{\wass}_{t} + \eta^{2} \left( \jkoreg + \frac{1}{2} \frac{\partial^{2} }{\partial t^{2}} \rho'  \right) + \eta^{3} g_{2}(\rho) + o(\eta^{2})  \right).
\end{align*}
Recall that the first variation of a functional $\mathcal{F}: {P}_{ac}(\mathbb{R}^{d}) \to \mathbb{R}$ at $\rho$, denoted $\delta \mathcal{F}(\rho): \mathbb{R}^{d} \to \mathbb{R}$, is defined by
\[
\lim_{\epsilon \to 0} \frac{\mathcal{F}(\rho + \epsilon \chi) - \mathcal{F}(\rho)}{\epsilon} = 
\delta\mathcal{F} [\rho, \chi] =\int_{\Omega} \frac{\delta \mathcal{F}(\rho)}{\delta\rho} d\chi
\]
for all signed measures $\chi$ such that $\rho + \epsilon \chi$ is closed in ${P}_{ac}(\mathbb{R}^{d})$. Higher-order variations are defined similarly by the formula
\[
 \frac{\partial^{k} }{ \prod_{i=1}^{k} \partial \epsilon_{i}} \mathcal{F}(\rho+\epsilon_{1}\chi_{1} + \epsilon_{2} \chi_{2} + \cdots ) \mid_{\epsilon_{i} = 0} = \delta^{k} \mathcal{F}[\rho,\chi_{1},\cdots,\chi_{k}] = \int \frac{\delta^{k} \mathcal{F}}{ \delta \rho^{k}}  d\chi_{1} \cdots d\chi_{k}.
\]
And using these, one can collect appropriate terms by their order. By assumption (A5)(i) and (A5)(ii), the functional $J$ is twice Fréchet differentiable on an open set $U$ with variations bounded by the bilinear form $D^{2}J(\rho)$. This allows us to take a variational expansion in the higher-variations \cite{Courant_Hilbert_2009} as
\[
\mathcal{F}(\rho+\eta \chi) = \mathcal{F}(\rho) +
\eta\int_{\Omega} \frac{\delta \mathcal{F}(\rho)}{\delta\rho} d\chi + \eta^{2}\iint_{\Omega}\frac{1}{2}\frac{\delta^{2} \mathcal{F}(\rho)}{\delta\rho \,\delta\rho'} d\chi d\chi' + \cdots
\]
Along the functional direction $\chi = \chi_{1} + \eta \chi_{2}$ this becomes
\begin{align*}
\mathcal{F}(\rho+\eta \chi) &= \mathcal{F}(\rho) +
\eta\int_{\Omega} \frac{\delta \mathcal{F}(\rho)}{\delta\rho} d(\chi_{1} + \eta \chi_{2}) + \eta^{2}\iint_{\Omega}\frac{1}{2}\frac{\delta^{2} \mathcal{F}(\rho)}{\delta\rho \,\delta\rho'} d(\chi_{1} + \eta \chi_{2}) d(\chi_{1} + \eta \chi_{2})' \\
&=\mathcal{F}(\rho) +
\eta\int_{\Omega} \frac{\delta \mathcal{F}(\rho)}{\delta\rho} d\chi_{1} + 
\eta^{2} \left( \int_{\Omega} \frac{\delta \mathcal{F}(\rho)}{\delta\rho} d \chi_{2} + \frac{1}{2}\iint_{\Omega}\frac{\delta^{2} \mathcal{F}(\rho)}{\delta\rho \,\delta\rho'} d\chi_{1}  d\chi_{1}' \right).
\end{align*}
Now, let us apply the $\eta$-expansion above to each term of $\Phi_{\eta}$. For the energy term $J$, one has the Taylor expansion of the functional given in its variational derivatives
\begin{align*}
J(\rho^{(n+1, \eta)}) &= J\left(\rho^{(n,\eta)} + \eta \frac{\partial }{\partial t} \rho^{\wass}_{t} + \eta^{2} \left( \jkoreg + \frac{1}{2} \frac{\partial^{2} }{\partial t^{2}} \rho'  \right) + \eta^{3} g_{2}(\rho) + o(\eta^{2}) \right) \\
&= J(\rho^{(n,\eta)}) + \eta \int \frac{\delta J (\rho^{(n,\eta)})}{\delta \rho} d\left(\frac{\partial }{\partial t} \rho^{\wass}_{t}\right) \\  &+
\frac{\eta^{2}}{2} \left(
2\int \frac{\delta J}{ \delta \rho} d\left( \jkoreg + \frac{1}{2} \frac{\partial^{2} }{\partial t^{2}} \rho'
\right) +
\int \frac{\delta^{2} J}{ \delta \rho(x) \delta \rho(x')} d\left(\frac{\partial }{\partial t} \rho^{\wass}_{t}\right)(x)d\left(\frac{\partial }{\partial t} \rho^{\wass}_{t}\right)(x') 
\right)
+ o(\eta^{2}).
\end{align*}
Likewise, for the squared Wasserstein distance we have
\begin{align*}
&W_{2}^{2}\left(\rho^{(n, \eta)}, \rho^{(n,\eta)} + \eta \frac{\partial }{\partial t} \rho^{\wass}_{t} + \eta^{2} \left( \jkoreg + \frac{1}{2} \frac{\partial^{2} }{\partial t^{2}} \rho'  \right) + \eta^{3} g_{2}(\rho) + o(\eta^{2})  \right) \\
& \quad = \underbrace{W_{2}^{2}(\rho^{(n, \eta)}, \rho^{(n, \eta)})}_{=0} + \eta \int \frac{\delta W_{2}^{2} (\rho^{(n,\eta)}, \rho)}{\delta \rho} d\left(\frac{\partial }{\partial t} \rho^{\wass}_{t}\right) \\
&\quad  + \frac{\eta^{2}}{2} \left(
2 \int \frac{\delta W_{2}^{2}(\rho^{(n, \eta)}, \rho)}{ \delta \rho} d\left( \jkoreg + \frac{1}{2} \frac{\partial^{2} }{\partial t^{2}} \rho'
\right) +
\int \frac{\delta^{2} W_{2}^{2}(\rho^{(n, \eta)}, \rho)}{ \delta \rho(x) \delta \rho(x')} d\left(\frac{\partial }{\partial t} \rho^{\wass}_{t}\right)(x)d\left(\frac{\partial }{\partial t} \rho^{\wass}_{t}\right)(x') 
\right).
\end{align*}
Thus, collecting the terms, we have the following
\begin{align*}
    \Phi_{\eta} &= J(\rho^{(n,\eta)})  + \eta \left(
    \int \frac{\delta J (\rho^{(n,\eta)})}{\delta \rho} + \frac{1}{2 \eta}\frac{\delta W_{2}^{2} (\rho^{(n,\eta)}, \rho)}{\delta \rho} d\left(\frac{\partial }{\partial t} \rho^{\wass}_{t}\right)
    \right) \\
    & + \frac{\eta^{2}}{2} \biggl(
    \int 2\frac{\delta J}{ \delta \rho} + \frac{1}{\eta} \frac{\delta W_{2}^{2}(\rho^{(n, \eta)}, \rho)}{ \delta \rho} d\left( \jkoreg + \frac{1}{2} \frac{\partial^{2} }{\partial t^{2}} \rho'
    \right) \\
    & +
    \int \left( \frac{\delta^{2} J}{ \delta \rho(x) \delta \rho(x')} + \frac{1}{2 \eta}\frac{\delta^{2} W_{2}^{2}(\rho^{(n, \eta)}, \rho)}{ \delta \rho(x) \delta \rho(x')} \right) d\left(\frac{\partial }{\partial t} \rho^{\wass}_{t}\right)(x)d\left(\frac{\partial }{\partial t} \rho^{\wass}_{t}\right)(x') 
    \biggr).
\end{align*}
When distributing orders of $\eta$ this becomes
\begin{align*}
\Phi_{\eta} &= J(\rho^{(n,\eta)}) + \frac{1}{2} \int \frac{\delta W_{2}^{2} (\rho^{(n,\eta)}, \rho)}{\delta \rho} d(\partial_{t} \rho^{\wass}_{t}) \\
&+ \eta \biggl( \int \frac{\delta J (\rho^{(n,\eta)})}{\delta \rho} d(\partial_{t} \rho^{\wass}_{t}) + \frac{1}{2} \int \frac{\delta W_{2}^{2}(\rho^{(n, \eta)}, \rho)}{ \delta \rho} d\left( \jkoreg + \frac{1}{2} \frac{\partial^{2} }{\partial t^{2}} \rho'
    \right) \\
    &\qquad \qquad + \frac{1}{4} \iint  \frac{\delta^{2} W_{2}^{2}(\rho^{(n, \eta)}, \rho)}{ \delta \rho(x) \delta \rho(x')}  d(\partial_{t} \rho^{\wass}_{t})(x) d(\partial_{t} \rho^{\wass}_{t})(x') \biggr) + o(\eta^{2}).
\end{align*}
By Lemma~\ref{lemma:Wasserstein_SecondVar}, one finds that the second-variation is expressed, for $\wagrad \in L^{2}[\rho_{t}]$, as the following quadratic integral over the tangent vectors to $\rho_{t}$;
\begin{align*}
    \frac{1}{4}\iint  \frac{\delta^{2} W_{2}^{2}(\rho^{(n, \eta)}, \rho)}{ \delta \rho(x) \delta \rho(x')}  d(\partial_{t} \rho^{\wass}_{t})(x) d(\partial_{t} \rho^{\wass}_{t})(x') & = \int \left\lVert \wagrad \right\rVert_{2}^{2} d \rho^{(n,\eta)} \\
    &= \int \left\langle \wagrad, \wagrad
    \right\rangle d \rho^{(n,\eta)}\\
    &= -\int \wadel \nabla \cdot (\rho_t \wagrad )\\
    &=  -\int \wadel -\nabla \cdot (\rho_t -\wagrad ) \\
    &= -\int \wadel (\rho^{(n,\eta)}) \partial_{t} \rho_{t}^{\wass}.
\end{align*}
The integration of parts used in the third equality is justified by the fast decay of $\rho_{t}$ at infinity from Assumption (A4). In addition, by Assumption (A5) $\nabla \psi \in L^{2}(\rho_{t})$ so that the integral inner product is finite. Noting the energy-dissipation identity for gradient flows
\begin{align}
    \frac{d}{dt} J(\rho_{t}) =
\int \frac{\delta J}{\delta \rho} \partial_{t} \rho^{\wass}=
- \left\lVert \nabla \frac{\delta J}{\delta \rho } \right\rVert_{L^{2}(\rho_{t})}^{2}
\end{align}
this term directly cancels with the second variation of the Wasserstein distance, so that we find
\begin{align*}
\Phi_{\eta} &= J(\rho^{(n,\eta)}) + \frac{1}{2} \int \frac{\delta W_{2}^{2} (\rho^{(n,\eta)}, \rho)}{\delta \rho} d(\partial_{t} \rho^{\wass}_{t}) \\
&+ \eta \biggl( \int \frac{\delta J (\rho^{(n,\eta)})}{\delta \rho} d(\partial_{t} \rho^{\wass}_{t}) + \frac{1}{2} \int \frac{\delta W_{2}^{2}(\rho^{(n, \eta)}, \rho)}{ \delta \rho} d\left( \jkoreg + \frac{1}{2} \frac{\partial^{2} }{\partial t^{2}} \rho'
    \right) \\
    &-\int \wadel d(\partial_{t} \rho_{t}^{\wass}) + o(\eta^{2}) \\
& = J(\rho^{(n,\eta)}) + \frac{1}{2} \int \frac{\delta W_{2}^{2} (\rho^{(n,\eta)}, \rho)}{\delta \rho} d(\partial_{t} \rho^{\wass}_{t}) + \frac{\eta}{2} \int \frac{\delta W_{2}^{2}(\rho^{(n, \eta)}, \rho)}{ \delta \rho} d\left( \jkoreg + \frac{1}{2} \frac{\partial^{2} }{\partial t^{2}}\rho'
    \right) + o(\eta^{2})
\end{align*}
The stationarity condition implies that up to a constant $C$ we have
\begin{align*}
    \delta\Phi_{\eta}[\rho^{(n+1)},\chi] = 0, \quad \text{for all admissible }\chi.
\end{align*}
Analogously we have $\rho^{(n+1)} = \rho^{(n)} + \eta \partial_{t} \rho' + (\eta^{2}/2) \partial_{tt} \rho' + o(\eta^{2})$ which by our $C^{2}$ variational regularity gave us an expansion about $\chi = - \nabla \cdot (\rho \xi)$ of the form
\begin{align*}
    &\delta \Phi_{\eta}[\rho^{(n+1)},\chi] =  A_{0}[\chi] + \eta A_{1}[\chi] + \eta^{2} A_{2} [\chi] + \cdots = 0.
\end{align*}
Since this must hold for all $\eta > 0$ and admissible $\chi$, one requires all $A_{k}(\chi) = 0$. At $o(\eta^{0})$, we use the Euler-Lagrange stationary condition that $\nabla \frac{\delta}{\delta \rho} \Phi_{\eta} = 0$, where from \cite{Jordan1998} one has, for $\rho^{n}$ assumed to solve the prior JKO step, that
\[
\frac{1}{\eta}\,\varphi^{n}+\psi(\rho^{n})=\text{const},\qquad
\nabla\varphi^{n}+\eta\,\nabla\psi(\rho^{n})=0\quad\text{a.e.}
\]
and identify the term
\begin{align*}
    \nabla \frac{\delta J(\rho^{(n,\eta)})}{\delta \rho} + \frac{1}{2} \nabla \frac{\delta}{\delta \rho} \int \frac{\delta W_{2}^{2} (\rho^{(n,\eta)}, \rho)}{\delta \rho} d(\partial_{t} \rho^{\wass}_{t}).
\end{align*}
Proposition~\ref{prop:grad_first_var}, a standard result \cite{Lott2007,AGS2008,chewi2024statistical}, asserts that
\begin{align*}
    \frac{1}{2} \nabla \frac{\delta}{\delta \rho} \int \frac{\delta W_{2}^{2} (\rho^{(n,\eta)}, \rho)}{\delta \rho} -\nabla \cdot (\rho_{t} \mathbf{v}_{t}) = \frac{\left(2 \mathbf{v}_{t} \right) }{2} = \mathbf{v}_{t}
\end{align*}
where $\mathbf{v}_{t} := - \wagrad (\rho_{t})(x) := - \wagrad (x, t)$, so we have that
\begin{align*}
    &\nabla \frac{\delta J(\rho^{(n,\eta)})}{\delta \rho} + \frac{1}{2} \nabla \frac{\delta}{\delta \rho} \int \frac{\delta W_{2}^{2} (\rho^{(n,\eta)}, \rho)}{\delta \rho} d(\partial_{t} \rho^{\wass}_{t}) = \nabla \frac{\delta J(\rho^{(n,\eta)})}{\delta \rho} \,\, - \,\, \nabla \frac{\delta J(\rho^{(n,\eta)})}{\delta \rho} =0
\end{align*}
So that, centered on the modified flow $\rho_t'$, the lower-order term vanishes, consistent with the statement of Theorem~\ref{thrm:JKO_first_order} \cite{Jordan1998}. What remains, at order $\eta$ and $\eta^{2}$ in the velocity, are
\begin{align*}
    &\frac{\eta}{2} \nabla \frac{\delta}{\delta \rho} \int \frac{\delta W_{2}^{2}(\rho^{(n, \eta)}, \rho)}{ \delta \rho} d\left( \jkoreg + \frac{1}{2} \frac{\partial^{2} }{\partial t^{2}}\rho'
    \right) \\
    &+ \frac{\eta^{2}}{2} \nabla \frac{\delta}{\delta \rho}\left( \iint \frac{\delta^{2} J}{\delta\rho(x) \delta \rho(x')} \partial_{t} \rho^{\wass}(x) \partial_{t} \rho^{\wass}(x') + 2\int \frac{\delta J}{ \delta \rho} d\left( \jkoreg + \frac{1}{2} \frac{\partial^{2} }{\partial t^{2}}\rho'
\right) 
    \right) + o(\eta^{2}).
\end{align*}
The implicit regularizer is defined up to order $\eta$ by the first term, and we require the $o(\eta)$ component to be stationary. Thus, we consider the condition
\begin{align*}
\frac{\eta}{2} \nabla \frac{\delta}{\delta \rho} \int \frac{\delta W_{2}^{2}(\rho^{(n, \eta)}, \rho)}{ \delta \rho} d\left( \jkoreg + \frac{1}{2} \frac{\partial^{2} }{\partial t^{2}} \rho'
    \right) \simeq_{o(\eta^{2})} \frac{\eta}{2} \nabla \frac{\delta}{\delta \rho} \int \frac{\delta W_{2}^{2}(\rho^{(n, \eta)}, \rho)}{ \delta \rho} d\left( \jkoreg + \frac{1}{2} \frac{\partial^{2} }{\partial t^{2}}\rho^{\wass}
    \right).
\end{align*}
We now integrate with respect to $\chi = \jkoreg + \frac{1}{2} \partial_{tt} \rho' \in M(\mathbb{R}^{d})$ as the signed measure, and observe that to the order relevant for the JKO residual $\partial_{tt} \rho' \simeq_{o(\eta^{2})} \partial_{tt} \rho^{\wass}$. By Assumption (A5)(iii), $\psi(\rho_{t}) \in C^{1,2}_{x,t}$ and $\nabla \psi(\rho_{t}), \partial_{t}\nabla \psi(\rho_{t}) \in L^{2}(\rho_{t})$. Thus, from Proposition~\ref{prop:accn_HJ}, we know that
\begin{align}
    \frac{1}{2}\nabla\frac{\delta}{\delta\rho} \int \frac{\delta W_{2}^{2}(\rho^{(n,\eta)}, \rho) }{\delta \rho} - \nabla\cdot\left( \partial_{t}  \left( - \wagrad (t, x(t)) \rho_{t}\right) \right) = - \partial_{t} \wagrad + \frac{1}{2}\nabla \left\lVert \wagrad \right\rVert_{2}^{2}
\end{align}
and likewise, from Proposition~\ref{prop:grad_first_var} that
\begin{align*}
    \nabla \frac{\delta}{\delta \rho} \int \frac{\delta W_{2}^{2}(\rho^{(n,\eta)}, \rho) }{\delta \rho} -\nabla\cdot \left( \rho_{t} \bm{j}(x,t) \right) = 2 \,\bm{j}(x,t).
\end{align*}
Thus, we have that
\begin{align*}
    &\frac{\eta}{2} \nabla \frac{\delta}{\delta \rho} \int \frac{\delta W_{2}^{2}(\rho^{(n, \eta)}, \rho)}{ \delta \rho} d\left( \jkoreg + \frac{1}{2} \frac{\partial^{2} }{\partial t^{2}}\rho^{\wass}
    \right) \\
    &=\frac{\eta}{2} \nabla \frac{\delta}{\delta \rho} \int \frac{\delta W_{2}^{2}(\rho^{(n, \eta)}, \rho)}{ \delta \rho} - \nabla \cdot (\rho^{\wass} \bm{j})  + \frac{\eta}{4} \nabla \frac{\delta}{\delta \rho} \int \frac{\delta W_{2}^{2}(\rho^{(n, \eta)}, \rho)}{ \delta \rho}  - \nabla \cdot \left( \partial_{t}\left( - \wagrad(\rho_{t})(x) \rho^{\wass} \right) 
    \right) \\
    &= \frac{\eta}{2} \left( 2 \bm{j} \right) + \frac{\eta}{2} \left(
    - \partial_{t} \wagrad + \frac{1}{2}\nabla \left\lVert \wagrad \right\rVert_{2}^{2}
    \right) = \eta \left( \bm{j} \,\, + \frac{1}{2}\left(
    - \partial_{t} \wagrad + \frac{1}{2}\nabla \left\lVert \wagrad \right\rVert_{2}^{2}
    \right)
    \right) = 0
\end{align*}
And we identify $\bm{j}$ to be the quantity
\begin{align*}
    \bm{j} =  \frac{1}{2}\left(
     \partial_{t} \wagrad - \frac{1}{2}\nabla \left\lVert \wagrad \right\rVert_{2}^{2}
    \right)
\end{align*}
The modified flow that matches the solution of JKO for small $\eta > 0$ is therefore given by the PDE
\begin{align}
    \partial_{t} \rho_{t}' &= - \nabla \cdot \left(\rho_{t}' \left( -\wagrad (\rho'_{t}) + \eta \, \bm{j}_{t} 
\right) \right)\,\, \\
&=  - \nabla \cdot \left(\rho_{t}' \left( -\wagrad (\rho'_{t}) + \, \frac{\eta}{2}\nabla\left(
     \frac{\partial}{\partial t} \wadel(\rho'_{t}) - \frac{1}{2} \left\lVert \wagrad (\rho'_{t})\right\rVert_{2}^{2}
    \right)
\right) \right)\,\,\label{eq:mod_flow_JKO}
\end{align}

\begin{remark}
If $(\rho_{t},\phi_{t})$ solves the Benamou-Brenier geodesic in reverse time, and $\partial_{t} \phi_{t} - \frac{1}{2} |\mathbf{v}_{t}|^{2} = 0$, $\nabla\phi$ is a displacement geodesic velocity and thus $\partial_{t}\mathbf{v}_{t} - \frac{1}{2} \nabla |\mathbf{v}_{t}|^{2} = 0$ and $\nabla \mathcal{H}[\phi] = 0$.
\end{remark}
\end{proof}

\begin{remark}[On the Velocity-field Sign Convention.]
As the JKO scheme is an implicit (Backward Euler) integrator in Wasserstein space, the residual velocity field $\bm{j}_{t}$ would be conventionally signed positive in the reverse-time direction. To maintain consistency with the standard optimization convention (e.g. \cite{smith2021on}) which defines explicit Euler velocity as being positively signed in the forward-time direction, we invert the signage $\bm{j}_{t} \to -\bm{j}_{t}$ below.
\end{remark}

We next seek to identify an explicit functional whose gradient flow recovers \eqref{eq:mod_flow_JKO}. 

\begin{corollary}
    Supposing the assumptions of Proposition~\ref{prop:implicit_JKO} hold, one has that the modified flow \eqref{eq:mod_flow_JKO} matches a single JKO step to second-order so that
    \begin{align}
    W_{2} (\rho^{(n+1,\eta)}, \,\rho_{\eta}') \leq C\eta^{3} + o(\eta^{3}) = O(\eta^{3})
    \end{align}
\end{corollary}

\begin{proof}
    Suppose $\Phi_{\eta}$ is variationally stationary to $o(\eta^{2})$ with respect to both the left and right side of the Ansatz \eqref{eq:Ansatz}
    \begin{align}\label{eq:Ansatz}
    \rho^{(n+1)} = \rho^{(n)} + \eta \frac{\partial }{\partial t}\rho' + \frac{\eta^{2}}{2} \frac{\partial^{2} }{\partial t^{2}} \rho' + \xi
    \end{align}
    With a.c. curve $\rho'$ on $t \in [n\eta, (n+1)\eta]$ satisfying boundary $\rho'(n\eta) = \rho^{(n)}$. From Proposition~\ref{prop:implicit_JKO} we have the velocity gap of the JKO flow to the truncated second-order flow is uniformly bounded as
    \begin{align*}
    &\left\lVert \bm{v}_{\mathrm{JKO}}(t) - \bm{v}_{\mathrm{JKO}}^{(2)}(t) \right\rVert_{L^{2}(\rho_{t})}\leq C \eta^{2}, \quad \text{for all }t\in [n\eta, (n+1)\eta]
    \end{align*}
    Thus, recalling Lemma~\ref{lem:stability} we have
    \begin{align*}
    &W_{2} (\rho_{t}, \rho_{t}') \leq \left( W_{2} (\rho_{t_0}, \rho_{t_0}' )  + \frac{C\eta^{2}}{L} \right) e^{L(t-t_{0})} - \frac{C\eta^{2}}{L}
\end{align*}
For $L$ the Lipschitz constant of the flow. And for time $t=\eta$, and common initial datum $\rho^{(n)}$ one has
 \begin{align*}
    W_{2} (\rho^{(n+1,\eta)}, \rho_{\eta}') &\leq \frac{C\eta^{2}}{L}  e^{L\eta} - \frac{C\eta^{2}}{L} \\
    &= 
\frac{C\eta^{2}}{L}\left(1 + (L\eta) + \frac{L^{2}\eta^{2}}{2} + o(\eta^{2})
\right) - \frac{C\eta^{2}}{L} \\
&\simeq \frac{C\eta^{2}}{L}(L\eta) =C\eta^{3}.
\end{align*}
\end{proof}

\begin{theorem}[The Implicit Regularizer of the JKO Scheme.]\label{thrm:velocity_to_loss} Let $\phi := \frac{\delta J}{\delta \rho}$ denote the first variation of the energy functional $J: {P}_{ac}(M) \to \mathbb{R}$. Suppose that the assumptions of Proposition~\ref{prop:implicit_JKO} hold. Then, Wasserstein gradient flow on the modified loss function
\begin{align}\label{eq:JKO_mod_loss}
J^{\eta} = J -
    \frac{\eta}{4} \int \lVert \nabla \phi(\rho) \rVert_{2}^{2} d \rho
\end{align}
Coincides with the solution of JKO to order $\eta^{2}$.
\end{theorem}

\begin{proof}
    We hope to show that variations of \eqref{eq:JKO_mod_loss} yield as an Euler-Lagrange condition the PDE of Proposition~\ref{prop:implicit_JKO}. To do so, let us consider the variation $\delta \left[ \int \lVert \nabla \phi(\rho) \rVert_{2}^{2} d \rho \right]$ by considering compact variations of the form $\rho_{(s)}=\rho + s \xi$. One then finds
\begin{align}\label{eq:var_sq_velocity}
    &\int \lVert \nabla \phi(\rho + s \xi) \rVert_{2}^{2} d (\rho + s \xi)
\end{align}
where, for $\phi = \wadel$, we assume sufficient regularity in Proposition~\ref{prop:implicit_JKO} to formally take a Gateux expansion of the form
\[
    \nabla \phi(\rho + s \xi) = \nabla \phi (\rho) + s \nabla \int  \frac{\delta \phi}{\delta \rho} d\xi + o(s)
\]
Thus, expanding \eqref{eq:var_sq_velocity} and taking $\frac{d}{ds} S[\rho_{(s)}] \mid_{s=0}$ by the variational principle, we find the first variation in $\rho$ as the integrand across the test function $\xi$
\begin{align*}
    &\int 
      \lVert \nabla \phi (\rho) \rVert_{2}^{2} d \xi + \int 2 \left\langle \nabla \phi (\rho), 
    \nabla \int  \frac{\delta \phi}{\delta \rho} d\xi
     \right\rangle d\rho\\
     &\qquad = \int \left( 
      \lVert \nabla \phi (\rho) \rVert_{2}^{2} + \int 2 \left\langle \nabla \phi (\rho), 
    \nabla \frac{\delta \phi}{\delta \rho} 
     \right\rangle \, d\rho \, \right) d\xi
\end{align*}
Since $\xi$ is an arbitrary test function, this implies that the first variation in $\rho$ is
\begin{align*} 
      \lVert \nabla \phi (\rho) \rVert_{2}^{2} + \int 2 \langle \nabla \phi (\rho), 
    \nabla \frac{\delta \phi}{\delta \rho} 
     \rangle \, d\rho \, 
\end{align*}
Now, recall that $\nabla \frac{\delta }{\delta \rho}[\phi] = \nabla \frac{\delta^{2} J}{\delta \rho^{2}}$, so that this becomes $\left\lVert \nabla \frac{\delta J}{\delta \rho} (\rho) \right\rVert_{2}^{2} + \int 2 \left\langle \nabla \frac{\delta J}{\delta \rho} (\rho), 
    \nabla \frac{\delta^{2} J}{\delta \rho^{2}} 
     \right\rangle \, d\rho \,$. With an integration by parts with vanishing boundaries we have
\begin{align*}
    \int \left\langle \nabla \frac{\delta J}{\delta \rho} (\rho), 
    \nabla \frac{\delta^{2} J}{\delta \rho^{2}} 
     \right\rangle \, d\rho = - \int \frac{\delta^{2} J}{\delta \rho^{2}} -\nabla \cdot \left(
    \rho -\nabla \frac{\delta J}{\delta \rho} (\rho)
     \right) = -\int \frac{\delta^{2} J}{\delta \rho^{2}} \partial_{t} \rho_{t}.
\end{align*}
By the standard identity for gradient flows that $\partial_{t} \mathcal{F}(\rho_{t}) = \int \delta \mathcal{F} (\rho_{t}) \partial_{t} \rho_{t}$ \cite{AGS2008, chewi2024statistical}, we have that for $\rho_{t}$ the gradient flow density $-\int \frac{\delta}{\delta \rho} \left( \frac{\delta J}{\delta \rho} \right) \partial_{t} \rho_{t} = - \partial_{t}\frac{\delta J}{\delta \rho}(\rho_{t})$. And thus conclude for $\phi = \frac{\delta J}{\delta \rho}$ that
\begin{align*}
    \frac{\delta}{\delta \rho} \left( \frac{1}{2}\int \lVert \nabla \phi(\rho) \rVert_{2}^{2} d \rho \right) &= -\partial_{t} \phi_{t} + \frac{1}{2} \lVert \nabla \phi_{t} \rVert_{2}^{2} \\
    \nabla\frac{\delta}{\delta \rho} \left( \frac{1}{2}\int \lVert \nabla \phi(\rho) \rVert_{2}^{2} d \rho \right) &= -\nabla \partial_{t} \phi_{t} + \frac{1}{2} \nabla \lVert \nabla \phi_{t} \rVert_{2}^{2}
\end{align*}
Thus, the correction in \eqref{eq:mod_flow_JKO} arises from an energy whose gradient of first-variation has the form
\begin{align}
    \notag \nabla\frac{\delta H}{\delta \rho} &=
    \frac{\eta}{2} \left( -\nabla \partial_{t} \phi_{t} + \frac{1}{2} \nabla \lVert \nabla \phi_{t} \rVert_{2}^{2} \right) = \nabla \frac{\delta}{\delta \rho} \left( \frac{\eta}{4} \int \lVert \nabla \phi(\rho) \rVert_{2}^{2} d \rho
    \right)  \label{eq:JKO_PDE_to_Loss}
\end{align}
So we conclude the implicit regularization term of the JKO scheme to be the quantity:
\begin{align}
    H(\rho)=\frac{\eta}{4} \int \lVert \nabla \phi(\rho) \rVert_{2}^{2} d \rho
\end{align}
\end{proof}

\paragraph{Riemannian setting.}
Let $(M,g)$ be Riemannian.
For densities $\rho \in{P}_{ac}(M)$ with smooth positive densities w.r.t.\ the Riemannian volume-form $\mathrm{vol}_g$ defined differentially by
\begin{align*}
    \mathrm{d}\mathrm{vol}_g (x) = \sqrt{\det g} \,\,x^{1} \land x^{2} \land \cdots  \land x^{n}
\end{align*}
Denote smooth test functions $f$ and smooth tangent vector fields $X$. Suppose in addition that the boundary conditions which kill boundary fluxes. The Riemannian analogue of the continuity equation and integration by parts is given by
\begin{align*}
    \partial_t \rho_t &= \nabla_g \cdot\!\left( \rho_t \, \nabla_g \frac{\delta J}{\delta \rho} \right), \\
    \displaystyle \int g(\nabla_g f, X)\,d\mathrm{vol}_g &= -\int f\, \nabla_g\!\cdot X\, d\mathrm{vol}_g.
\end{align*}
By exchanging in Proposition~\ref{prop:implicit_JKO} the Euclidean continuity equation and integration by parts identity for the Riemannian analogues, since the identities in Proposition~\ref{prop:grad_first_var}, Proposition~\ref{prop:accn_HJ}, and
Lemma~\ref{lemma:Wasserstein_SecondVar} extend directly to the Riemannian setting: Proposition~\ref{prop:implicit_JKO} holds with $\lVert \cdot \rVert_{2}^{2}$ substituted for $\lVert \cdot \rVert_{g}^{2}$.

\subsection{Other Cases of the Implicit Bias.}\label{sec:JKO_IB_other}

We recounted in the main text the implicit regularization of JKO on the entropy functional, the KL divergence, and the classical linear potential energy. Below, we observe a number of other common functionals and their associated implicit regularization under the JKO scheme.

\subsubsection{Interaction Energy.}

If we consider the mean-field interacting particle energy functional
\begin{align*}
J(\rho) = \frac{1}{2}\iint \rho(dx) \rho(dy) K(x-y) , \quad \frac{\delta J}{\delta \rho} = \int \rho(dy) K(x-y) , \quad \nabla \frac{\delta J}{\delta \rho} = \int \nabla K(x-y)\,\rho(dy) 
\end{align*}
defined with respect to a symmetric kernel $K: \mathbb{R}^{d} \to \mathbb{R}$ satisfying $K(x) = K(-x)$, then 
\begin{align}H(\rho)&=\frac{\eta}{4} \int \left\langle \int \nabla K(x-y)\,\rho(dy), \int \nabla K(x-z)\,\rho(dz) \right\rangle d \rho(x) \\
    &=\frac{\eta}{4} \iiint \left\langle \nabla K(x-y)\,, \nabla K(x-z)\, \right\rangle \rho(dx) \rho(dy) \rho(dz) 
\end{align}

Thus, JKO penalizes non-alignment between the pairwise interaction forces.

\subsubsection{Internal Energy and Porous-Medium Equation.}

Recall the interaction energy, its variation, and its Wasserstein gradient
\begin{align*}
J(\rho) =  \int U(\rho(x))dx , \quad \frac{\delta J}{\delta \rho} = U'(\rho(x)), \quad \nabla \frac{\delta J}{\delta \rho} = \nabla U'(\rho(x))
\end{align*}
Then, one has
\begin{align}\label{eq:internal-ib}
H^{\eta}(\rho)= \frac{\eta}{4} \int_M \lVert \nabla U'(\rho(x))\rVert_{2}^{2} \,\rho(dx) 
\end{align}
In the porous-medium equation, the internal energy is $U(\rho(x)) = \frac{1}{m-1} \rho^{m}$ and thus
\[
\frac{\delta J}{\delta \rho} =
\frac{m}{m-1} \rho^{m-1} , \quad  \nabla \frac{\delta J}{\delta \rho} = \frac{m}{m-1} \nabla \rho^{m-1} = m\, \rho^{m-2} \,\nabla \rho
\]
So that, in Fisher-form, one has
\begin{align}\label{eq:porous-ib}
H^{\eta}(\rho)= \frac{\eta m^{2}}{4} \int_M \lVert \rho^{m-2}\nabla \rho\rVert_{2}^{2} \,\rho(dx) &=\frac{\eta m^{2}}{4} \int_M \lVert \rho^{m-2}\nabla \rho\rVert_{2}^{2} \,\rho(dx) =\frac{\eta m^{2}}{4} \int_M \lVert \nabla \rho\rVert_{2}^{2} \,\rho^{2m-3}(dx)
\end{align}
Meanwhile, rewriting into Dirichlet form, one finds
\begin{align*}
    &\frac{\eta m^{2}}{4} \int_M \lVert \nabla \rho\rVert_{2}^{2} \,\rho^{2m-3}(dx) = \frac{\eta m^{2}}{4} \int_M \lVert \nabla \rho \,\rho^{(2m-3)/2}\rVert_{2}^{2} \,(dx) \\
    &= \frac{\eta m^{2}}{4} \int_M \left\lVert \frac{2}{2m-1} \nabla (\rho^{(2m-1)/2} )\right\rVert_{2}^{2} \,(dx) \\
    & H^{\eta}(\rho)= \frac{\eta m^{2}}{(2m-1)^{2}} \int_M \left\lVert \nabla (\rho^{(2m-1)/2} )\right\rVert_{2}^{2} \,(dx)
\end{align*}
This is a classical Dirichlet energy of a power $\rho^{(2m-1)/2}$.

\subsection{Wasserstein Barycenter.}

Recall the $K$-Wasserstein barycenter energy functional for $\lambda_{k}:\sum_{k}\lambda_{k} =1$
\begin{align*}
    J(\rho) = \sum_{k=1}^{K} \lambda_{k} W_{2}^{2}(\rho , \mu_{k})
\end{align*}
We have the canonical first variation and $W_2$ gradient \cite{AGS2008}
\begin{align*}
\frac{\delta J}{\delta \rho} = 2\sum_{k=1}^{K} \lambda_{k}\,\varphi^{\rho \to \mu_{k}}, \quad \nabla \frac{\delta J}{\delta \rho} = 2 \sum_{k=1}^{K} \lambda_{k} \,
(T^{\rho \to \mu_{k}}(x) - x)
\end{align*}
For $\varphi^{\rho \to \mu_{k}}$ the Kantorovich potential and $T^{\rho \to \mu_{k}}$ the Monge map between $\rho$ and each measure $\mu_{k}$. Then, the implicit bias of JKO is written as
\begin{align*}
&H^{\eta}(\rho) = \frac{\eta}{4} \int \Big\lVert  2 \sum_{k=1}^{K} \lambda_{k} \,
(T^{\rho \to \mu_{k}}(x) - x)\Big\rVert_{2}^{2} \rho(dx) =\eta \int \Big\lVert   \sum_{k=1}^{K} \lambda_{k} \,
(T^{\rho \to \mu_{k}}(x) - x) \Big\rVert_{2}^{2} \rho(dx)
\end{align*}
Thus JKO penalizes the expected squared-norm of the barycentric displacement of $\rho$ from the target measures.

\begin{remark}[First Variation of the Fisher Functional.]\label{rem:Fisher_var}
Recall the Fisher Information Functional \cite{Amari2016}
    \begin{align*}
    \int \lVert \nabla \log \rho \rVert_{2}^{2} \,\rho(dx) 
\end{align*}
The first variation and Wasserstein gradient are given by Theorem 4.2 of \cite{Gianazza2008} to be
\begin{equation}\label{eq:fv_fish}
\frac{\delta I}{\delta \rho} [\rho] = -4\, \frac{\Delta \sqrt{ \rho}}{\sqrt{\rho}} , \quad \nabla\frac{\delta I}{\delta \rho} [\rho] = -4\, \nabla \left( \frac{\Delta \sqrt{ \rho}}{\sqrt{\rho}} \right)
\end{equation}
\end{remark}
\begin{proof}\label{proof:fv_fisher}
    We briefly show ~\ref{rem:Fisher_var}. Let us suppose, as standard, a compact support $\Omega$ with vanishing boundary conditions. Further, suppose we take compact variations of the form $\rho_{(s)} = \rho + s \dot{\phi}$. One first simply takes the $s$-derivative
    \begin{align*}
        \frac{d}{ds} \int \lVert \nabla \log{\rho_{(s)}} \rVert_{2}^{2} \,\rho_{(s)}(dx)& =  \int \frac{d}{ds}\lVert \nabla \log{\rho_{(s)}} \rVert_{2}^{2} \,\rho_{(s)}(dx) + \int \lVert \nabla \log{\rho_{(s)}} \rVert_{2}^{2} \,\frac{d}{ds}\rho_{(s)}(dx) \\
        &= 2\int \langle \nabla \log{\rho_{(s)}},  \frac{d}{ds} \left( \frac{\nabla \rho_{(s)}}{\rho_{(s)}} \right)\rangle \,\rho_{(s)}(dx) + \int \lVert \nabla \log{\rho_{(s)}} \rVert_{2}^{2} \,\dot{\phi}(dx) \\
        &= 2\int \langle \nabla \log{\rho_{(s)}},   \left( \frac{\nabla \dot{\phi}}{\rho_{(s)}} - \frac{\nabla \rho_{(s)} \dot{\phi}}{\rho_{(s)}^{2}}\right)\rangle \,\rho_{(s)}(dx) + \int \lVert \nabla \log{\rho_{(s)}} \rVert_{2}^{2} \,\dot{\phi}(dx) \\
        \end{align*}
        Then, after identifying the norm of the score function from the ratio
        \begin{align*}
         \int \dot{\phi} \left(
        \lVert \nabla \log \rho_{(s)} \rVert_{2}^{2}  -
        2 \left\langle \nabla \log \rho_{(s)}, \nabla \log \rho_{(s)}
        \right\rangle \right)(dx) + 2\int \left\langle \nabla \log \rho_{(s)}, \nabla \dot{\phi} \right\rangle dx, \end{align*}
         and canceling terms, we find it equals
        \begin{align*}
         -\int \dot{\phi} \left(
        \lVert \nabla \log \rho_{(s)} \rVert_{2}^{2}  \right) dx + 2\int \left\langle \nabla \log \rho_{(s)}, \nabla \dot{\phi} \right\rangle dx.
    \end{align*}
    Lastly, with an integration by parts on the rightmost term, we identifying the Laplacian of the log-density
    \[
    \int_{\Omega} \left\langle \nabla \log \rho_{(s)}, \nabla \dot{\phi} \right\rangle dx = - \int \Delta \log \rho_{(s)}\dot{\phi}.
    \]
    Thus, from
    \begin{align*}
    &-\int \dot{\phi} \left(
        \lVert \nabla \log \rho_{(s)} \rVert_{2}^{2}   +2  \Delta \log \rho_{(s)} \right) dx \,\,\bigg|_{s=0} = -\int \dot{\phi} \left(
        \lVert \nabla \log \rho \rVert_{2}^{2}   +2  \Delta \log \rho \right) dx,
    \end{align*}
    One identifies that the variation against test functions $\dot\phi$ is given by
    \begin{align*}
        \frac{\delta I}{\delta \rho} [\rho] = -\lVert \nabla \log \rho \rVert_{2}^{2}   -2  \Delta \log \rho
    \end{align*}
    To identify with the Bohmian form \eqref{eq:fv_fish} of Quantum-Drift-Diffusion \cite{Gianazza2008}, note the identity
    \begin{align*}
        \frac{\Delta \sqrt{ \rho}}{\sqrt{\rho}} &= \frac{1}{\sqrt{\rho}} \tr \nabla^{2} \sqrt{\rho}\\
        &= \frac{1}{2\sqrt{\rho}} \tr \left(  \frac{\nabla^{2} \rho}{\sqrt{\rho}} - \frac{1}{2} \frac{\nabla \rho \nabla \rho^{\top}}{\rho^{3/2}}   \right) \\
        &= \frac{1}{2\sqrt{\rho}} \tr \left(  \frac{\nabla^{2} \rho}{\sqrt{\rho}} - \frac{1}{2} \frac{\nabla \rho \nabla \rho^{\top}}{\rho^{3/2}}   \right)\\
        &=\frac{1}{2\rho} \tr \nabla^{2} \rho - \frac{1}{4}  \left\lVert \frac{\nabla \rho}{\rho} \right\rVert_{2}^{2} \\
        &= \frac{1}{2} \tr \frac{\nabla^{2} \rho}{\rho} - \frac{1}{4}  \left\lVert \nabla \log \rho \right\rVert_{2}^{2} = \frac{1}{2} \Delta \log \rho + \frac{1}{2} \left\lVert \nabla \log \rho \right\rVert_{2}^{2} - \frac{1}{4}  \left\lVert \nabla \log \rho \right\rVert_{2}^{2}
        \end{align*}
        Thus, 
    \begin{align*}
        &\frac{\delta I}{\delta \rho} [\rho] = -2 \Delta \log \rho -  \left\lVert \nabla \log \rho \right\rVert_{2}^{2} = -4\, \frac{\Delta \sqrt{ \rho}}{\sqrt{\rho}}
    \end{align*}
\end{proof}

\section{Bures-Wasserstein Flows.}\label{sec:BW-flows}

While the JKO scheme \eqref{eq:JKO} rarely admits analytic solutions, the work \cite{Halder2017} solves the special case of linear filtering problems in Bures-Wasserstein space, $\mathrm{BW}(\mathbb{R}^{d})\subset{P}_{ac}(\mathbb{R}^{d})$. This offers a unique test-bed for the case where the update to JKO is known analytically. In particular for a symmetric Hurwitz matrix $A$, \cite{Halder2017} consider the following energy functional 
\begin{align}\label{eq:quadratic_well}
&J(\rho) = - \frac{1}{2} \int x^{\top} A\, x\,\, \rho(dx) + \beta^{-1} \int \rho \log \rho
\end{align}
and observe the Wasserstein gradient flow of $J(\rho)$ yields the overdamped Langevin equation on a quadratic well potential $- (1/2)\cdot  x^{\top} A x$
\[
dX_{t} =  A X_{t} + \sqrt{2\beta^{-1}} dB_{t}
\]
These linear dynamics can be equivalently represented as a gradient flow on the mean and covariance in $\mathrm{BW}(\mathbb{R}^{d})$, given by the standard Lyapunov ordinary differential equation of filtering
\begin{align}\label{eq:Lyapunov}
&\dot{\bm{\mu}}^{\wass}_{t} = A \bm{\mu}_{t}, \quad \dot{\Sigma}_{t}^{\wass} = A \Sigma_{t} + \Sigma_{t} A + \frac{2}{\beta}\, \mathbf{I}.
\end{align}
Remarkably, \cite{Halder2017} compute exactly the solution to JKO for \eqref{eq:quadratic_well} with a finite step $\eta>0$ and find it admits an analytically tractable mean and covariance update:
\begin{align}\label{eq:JKO_analytic}
    \bm{\mu}_{n+1} &= (\mb{I} - \eta A )^{-1} \bm{\mu}_{n},  \\
    \label{eq:JKO_analytic_2} \left(
   \Sigma_{n}^{-1/2} \Sigma_{n+1}^{-1} \Sigma_{n}^{-1/2}
   \right)^{1/2} &= \frac{\beta}{2 \eta} \left( - \id + \left(
    \mathbf{I} + \frac{4 \eta}{\beta} \Sigma_{n}^{-1/2} (\id - \eta A ) \Sigma_{n}^{-1/2}
   \right)^{1/2}
   \right)
\end{align}
We analytically derive the correction to JKO for this special case using \eqref{eq:JKO_analytic} from the result of \cite{Halder2017}. We show that this correction exactly coincides with the correction for arbitrary functionals $J$ when applied to the special case of linear Gaussian filtering.

\begin{prop}[Implicit Bias of JKO for Linear Fokker-Planck.]\label{prop:JKO_FP}
    Suppose $\eta > 0$ is a step-size and $\beta > 0$ is an inverse-temperature. Then, minimizing the energy~\ref{eq:quadratic_well} over $\rho \in \mathrm{BW}(\mathbb{R}^{d})$ with \eqref{eq:JKO} coincides with the following Bures-Wasserstein gradient-flow up to $O_{T}(\eta^{2})$
    \begin{align}
        &\dot{\bm{\mu}}_{t}^{\eta} = 
        \dot{\bm{\mu}}_{t}^{\wass} + \frac{\eta}{2} A^{2}\bm{\mu}_{t} , \quad \dot{\Sigma}_{t}^{\eta} = \dot{\Sigma}_{t}^{\wass} + \frac{\eta}{2}  \left(
        A^{2} \Sigma_{t}^{\eta} + \Sigma_{t}^{\eta} A^{2} - \beta^{-2} \Sigma_{t}^{\eta,-1}
        \right) .
    \end{align}
\end{prop}
We derive these corrections analytically from the explicit form of the JKO solution given in \eqref{eq:JKO_analytic} and \eqref{eq:JKO_analytic_2} in ~\ref{proof:JKO_BW_analytic}. While this derivation is tedious from the form of \cite{Halder2017} (\eqref{eq:JKO_analytic} and \eqref{eq:JKO_analytic_2}), observe that the second-order JKO flow may be written immediately using \ref{thm:main} applied to the variation for the quadratic case.

\subsubsection{Derivation of the Bures-Wasserstein Correction}

In the following section, we verify the proposed implicit regularizer in this simple case of linear stochastic differential equations in the symmetrized frame of \cite{Halder2017}. In particular, consider the linear It\^o stochastic differential equation
\begin{align}\label{eq:SDE_FokkPlanck}
    dX_{t} = A X_{t} \,dt + \sqrt{2\beta^{-1}} dB_{t}
\end{align}
Where $B_{t}$ is a Brownian motion the dynamics are symmetric $A = A^{\top}$. The result of \cite{Halder2017} identified that the JKO problem which recovers the solution of \eqref{eq:SDE_FokkPlanck} in the limit of $\eta \downarrow 0$ is given by
\begin{align*}
    &\argmin_{\rho \in {P}_{ac}(\mathbb{R}^{d})} \left( \frac{1}{2 \eta} W_{2}^{2}(\rho, \rho_{0} := \mathcal{N}(\mu_{0}, P_{0}) ) \,\, + J(\rho) \right)\\
    &\qquad = \argmin_{\rho \in {P}_{ac}(\mathbb{R}^{d})} \left( \frac{1}{2 \eta} W_{2}^{2}(\rho, \rho_{0} ) \,\, + \mathcal{E}(\rho) + \beta^{-1} \mathcal{S}(\rho) \right) \\
    &\qquad = \argmin_{\rho \in {P}_{ac}(\mathbb{R}^{d})} \left( \frac{1}{2 \eta} W_{2}^{2}(\rho, \rho_{0} ) \,\, + \int U(x) d \rho(x) + \beta^{-1} \int \rho(x) \log\rho(x) d x \right)
\end{align*}
Where $U: \mathbb{R}^{d} \to \mathbb{R}$ denotes a quadratic potential dependent on $\Gamma = -A$ which drives the linear dynamics, $U(x) = (1/2) x^{\top}\Gamma x$, and where $\mathcal{S}(\rho)$ denotes the negative differential entropy. For 
\begin{align*}
J(\rho) = \int U(x) d \rho(x) + \beta^{-1} \int \rho(x) \log\rho(x) d x
\end{align*}
Up to constants, one has the first-variation is equal to
\begin{align*}
    \frac{\delta J}{\delta \rho} &= U(x) + \frac{1}{\beta} \log\rho(x) \\
    &= \frac{1}{2} x^{\top} \Gamma x \,\, + \beta^{-1} \log\mathcal{N}\left(x \mid \mu(t), P(t) \right)
\end{align*}
For the mean and covariance of $\rho(t)$ guaranteed to be a Gaussian $\mathcal{N}\left(x \mid \mu(t), P(t) \right)$ by the derivation of \cite{Halder2017}. 
Observe that the Gaussian score and Hessian of the log-likelihood are analytically given as
\begin{align*}
    \nabla \log \rho = - P^{-1} (x - \mu), \quad \nabla^{2} \log\rho = - P^{-1}
\end{align*}
Thus, we have that the Wasserstein gradient and Hessian of the first variation may be expressed analytically as
\begin{align*}
    \nabla \delta J &= \Gamma x - \beta^{-1} P^{-1} (x - \mu) \\
    \nabla^{2} \delta J &= \mathbb{E}_{\rho} \nabla^{2} \delta J = \Gamma - \beta^{-1} P^{-1}
\end{align*}

\begin{proof}[Proof of Proposition~\ref{prop:JKO_FP}]\label{proof:JKO_BW_analytic}

We begin by verifying the analytic implicit regularization for the simple case of the linear Fokker-Planck SDE considered in \cite{Halder2017}. We split this derivation of the analytical Bures-Wasserstein special case into the mean $\mu$ and the covariance $P$.

\textbf{Step 1: Mean Update.} In this case, one has
\begin{align*}
    \mu(t + \eta) &= \mu(t) + \eta \partial_{t} \mu'(t) + \frac{\eta^{2}}{2} \partial_{tt} \mu'(t)\\
    &= \mu(t) + \eta \partial_{t} m(t) + \eta^{2} \mathcal{R} + \frac{\eta^{2}}{2} \partial_{tt} \mu'(t) \\
    &= \mu(t) + \eta A \mu(t) + \eta^{2} \mathcal{R} + \frac{\eta^{2}}{2} \partial_{tt} m(t) + o(\eta^{2})\\
    &= \mu(t) + \eta A \mu(t) + \eta^{2} \left( \mathcal{R} + \frac{1}{2} A \partial_{t}m(t) \right) + o(\eta^{2}) \\
    &= \mu(t) + \eta A \mu(t) + \eta^{2} \left( \mathcal{R} + \frac{1}{2} A^{2} m(t) \right) + o(\eta^{2})
\end{align*}
In closed-form, the update for the mean is given as
\begin{align*}
    \mu(t + \eta) = (\mathbbm{1} - \eta A )^{-1} \mu(t)
\end{align*}
Which can be expressed as a Von Neumann series for appropriately small $\eta$, so that
\begin{align*}
    \mu(t + \eta) &= (\mathbbm{1} - \eta A )^{-1} \mu(t) = \sum_{i=0}^{\infty} (\eta A)^{i} \mu(t) \\
    &= \mu(t) + \eta A \mu(t) + \eta^{2} A^{2} \mu(t) + o(\eta^{2})
\end{align*}
Thus, we have the relations
\begin{align*}
    \mu(t) + \eta A \mu(t) + \eta^{2} \left( \mathcal{R} + \frac{1}{2} A^{2} \mu(t) \right) + o(\eta^{2})
    &= \mu(t) + \eta A \mu(t) + \eta^{2} A^{2} \mu(t) + o(\eta^{2}) \\
     \eta^{2} \left( \mathcal{R} + \frac{1}{2} A^{2} \mu(t) \right) &=  \eta^{2} A^{2} \mu(t) \\
      \mathcal{R} &= A^{2} \mu(t) - \frac{1}{2} A^{2} \mu(t) = +\frac{1}{2} A^{2} \mu(t)
\end{align*}
\textbf{Step 2: Covariance Update.} Changing convention briefly, to match the notation of \cite{Halder2017}, we have
\begin{align*}
    P^{\jko}(t+\eta) := P, \quad P^{\jko}(t) := P_{0}
\end{align*}
from which the following equations define the analytical update for the true solution for the JKO step in the covariance $P^{\jko}(t+\eta)$
\begin{align*}
   \left(
   P_{0}^{-1/2} P^{-1} P_{0}^{-1/2}
   \right)^{1/2} = Z &= \frac{\beta}{2 \eta} \left( - \id + \left(
    \mathbf{I} + \frac{4 \eta}{\beta} P_{0}^{-1/2} (\id + \eta \Gamma ) P_{0}^{-1/2}
   \right)^{1/2}
   \right)
\end{align*}
Noting the power series of the matrix square-root,
\begin{align*}
    A^{1/2} = \sum_{n=0}^{\infty} (-1)^{n} {{1/2}\choose{n}} \left( \id - A \right)^{n}, 
\end{align*}
we have that
\begin{align*}
    & \left(
    \mathbf{I} + \frac{4 \eta}{\beta} P_{0}^{-1/2} (\id + \eta \Gamma ) P_{0}^{-1/2}
   \right)^{1/2} \\
   & = \sum_{n=0}^{\infty} (-1)^{n} {{1/2}\choose{n}} \left( \id -  \mathbf{I} - \frac{4 \eta}{\beta} P_{0}^{-1/2} (\id + \eta \Gamma ) P_{0}^{-1/2} \right)^{n} \\
   &= \sum_{n=0}^{\infty} (-1)^{n} {{1/2}\choose{n}} \left( -\frac{4 \eta}{\beta} \right)^{n} \left( P_{0}^{-1/2} (\id + \eta \Gamma ) P_{0}^{-1/2} \right)^{n} \\
   &= \sum_{n=0}^{\infty}  {{1/2}\choose{n}} \left(\frac{4 \eta}{\beta} \right)^{n} \left( P_{0}^{-1/2} (\id + \eta \Gamma ) P_{0}^{-1/2} \right)^{n} \\
\end{align*}
From the definition of the binomial coefficient, we have the fractional coefficients of
\begin{align}
    {{1/2}\choose{n}} := \frac{ \prod_{k=0}^{n-1} \frac12 - k}{n!}; \qquad  {{1/2}\choose{0}} = 1, \quad {{1/2}\choose{1}} = \frac{1}{2}, \quad {{1/2}\choose{2}} = - \frac{1}{8}, \quad {{1/2}\choose{3}} = \frac{1}{16}, \,\, \cdots
\end{align}
So that, denoting $(\id + \eta \Gamma ) = M$, we have

\begin{align*}
    \sum_{n=0}^{\infty}  {{1/2}\choose{n}} \left(\frac{4 \eta}{\beta} \right)^{n} \left( P_{0}^{-1/2} M P_{0}^{-1/2} \right)^{n} = \id & + \frac{1}{2} \frac{4 \eta }{\beta} P_{0}^{-1/2} M P_{0}^{-1/2} \\
    & - \frac{1}{8} \left( \frac{16 \eta^{2}}{\beta^{2}}
    \right) P_{0}^{-1/2} M P_{0}^{-1} M P_{0}^{-1/2} \\
    & + \frac{1}{16} \left( \frac{64 \eta^{3}}{\beta^{3}} P_{0}^{-1/2} M P_{0}^{-1} M P_{0}^{-1} M P_{0}^{-1/2} \right) + o(\eta^{4})
\end{align*}
Evaluating this expansion in the definition of $Z$, we have that
\begin{align*}
   Z &= \frac{\beta}{2 \eta} \biggl( - \id + \id  + \frac{1}{2} \frac{4 \eta }{\beta} P_{0}^{-1/2} M P_{0}^{-1/2}  - \frac{1}{8} \left( \frac{16 \eta^{2}}{\beta^{2}}
    \right) P_{0}^{-1/2} M P_{0}^{-1} M P_{0}^{-1/2} \\
    & + \frac{1}{16} \left( \frac{64 \eta^{3}}{\beta^{3}} P_{0}^{-1/2} M P_{0}^{-1} M P_{0}^{-1} M P_{0}^{-1/2} \right) + o(\eta^{4}) \biggr) \\
    &= \frac{\beta}{2 \eta} \biggl(  \frac{2 \eta }{\beta} P_{0}^{-1/2} M P_{0}^{-1/2}  - \left( \frac{2 \eta^{2}}{\beta^{2}}
    \right) P_{0}^{-1/2} M P_{0}^{-1} M P_{0}^{-1/2} \\
    & + \left( \frac{4 \eta^{3}}{\beta^{3}} P_{0}^{-1/2} M P_{0}^{-1} M P_{0}^{-1} M P_{0}^{-1/2} \right) + o(\eta^{4}) \biggr) \\
    &= \biggl( P_{0}^{-1/2} M P_{0}^{-1/2}  - \frac{\beta}{2 \eta} \frac{2 \eta^{2}}{\beta^{2}}
    P_{0}^{-1/2} M P_{0}^{-1} M P_{0}^{-1/2} \\
    & + \frac{\beta}{2 \eta}  \frac{4 \eta^{3}}{\beta^{3}} P_{0}^{-1/2} M P_{0}^{-1} M P_{0}^{-1} M P_{0}^{-1/2}  + o(\eta^{4}) \biggr) \\
    &= \biggl( P_{0}^{-1/2} M P_{0}^{-1/2}  - \frac{ \eta}{\beta}
    P_{0}^{-1/2} M P_{0}^{-1} M P_{0}^{-1/2} + \frac{2 \eta^{2}}{\beta^{2}} P_{0}^{-1/2} M P_{0}^{-1} M P_{0}^{-1} M P_{0}^{-1/2}  + o(\eta^{3}) \biggr) 
\end{align*}
Thus, we have the identification of $Z$ 
\begin{align*}
    Z = 
    P_{0}^{-1/2}\biggl( M  - \frac{ \eta}{\beta}
    M P_{0}^{-1} M  + \frac{2 \eta^{2}}{\beta^{2}}  M P_{0}^{-1} M P_{0}^{-1} M   + o(\eta^{3}) \biggr)  P_{0}^{-1/2} = \left(
   P_{0}^{-1/2} P^{-1} P_{0}^{-1/2}
   \right)^{1/2}
\end{align*}
To solve for the closed-form expression for the JKO solution covariance $P$, let us denote the $\eta$-dependent expansion with $V$ as $Z = P_{0}^{-1/2} V P_{0}^{-1/2}$ and equate to the form involving $P^{-1}$;
\begin{align*}
    &\left(
   P_{0}^{-1/2} P^{-1} P_{0}^{-1/2}
   \right)^{1/2} =  P_{0}^{-1/2} V P_{0}^{-1/2} \\
   & P_{0}^{-1/2} P^{-1} P_{0}^{-1/2}
    =  P_{0}^{-1/2} V P_{0}^{-1} V P_{0}^{-1/2} \\
    & P^{-1} 
    = V P_{0}^{-1} V    
    \implies P 
    =  \left( V P_{0}^{-1} V \right)^{-1} = V^{-1} P_{0} V^{-1}
\end{align*}
Where, evaluating the definition of $Z$ and plugging in the values for the matrix $M = (\id + \eta \Gamma )$, we can finally express $V$ in orders of $\eta$, evaluating beyond the first-order of \cite{Halder2017} to isolate the implicit regularization effect;
\begin{align*}
    V &=  M  - \frac{ \eta}{\beta}
    M P_{0}^{-1} M  + \frac{2 \eta^{2}}{\beta^{2}}  M P_{0}^{-1} M P_{0}^{-1} M   + o(\eta^{3}) \\
    &= (\id + \eta \Gamma ) - \frac{ \eta}{\beta}
    (\id + \eta \Gamma ) P_{0}^{-1} (\id + \eta \Gamma )  + \frac{2 \eta^{2}}{\beta^{2}}  (\id + \eta \Gamma ) P_{0}^{-1} (\id + \eta \Gamma ) P_{0}^{-1} (\id + \eta \Gamma )   + o(\eta^{3}) \\
    &= \id + \eta \Gamma - \frac{\eta}{\beta} \left( P_{0}^{-1} + \eta \Gamma P_{0}^{-1} + \eta P_{0}^{-1} \Gamma + \eta^{2} \Gamma^{2} 
    \right) + \frac{2 \eta^{2}}{\beta^{2}} P_{0}^{-2} + o(\eta^{2}) \\
    &= \id + \eta \Gamma - \left( \frac{\eta}{\beta} P_{0}^{-1} + \frac{\eta^{2}}{\beta} \Gamma P_{0}^{-1} + \frac{\eta^{2}}{\beta} P_{0}^{-1} \Gamma  
    \right) + \frac{2 \eta^{2}}{\beta^{2}} P_{0}^{-2} + o(\eta^{2})
\end{align*}
Now, regrouping in orders of $\eta$ we find
\begin{align*}
     V &= \id + \eta W_{1} + \eta^{2} W_{2}  + o(\eta^{2}) \\
     &= \id + \eta \left( \Gamma - \frac{1}{\beta} P_{0}^{-1}
     \right) + \eta^{2}
      \left(
      \frac{2 }{\beta^{2}} P_{0}^{-2} - \frac{1}{\beta} \Gamma P_{0}^{-1} - \frac{1}{\beta} P_{0}^{-1} \Gamma  
    \right)  + o(\eta^{2})
\end{align*}
Returning to the expansion of $V^{-1}$, we observe that for a Neumann series one has
\begin{align*}
    \left(
    \id - T
    \right)^{-1} = \sum_{k=0}^{\infty} T^{k}
\end{align*}
so that, expressing up to orders in $\eta^{2}$, we have
\begin{align*}
    V^{-1} &= \left( \id - \left(- \eta W_{1} - \eta^{2}
      W_{2}  - o(\eta^{2})\right) \right)^{-1}  \\
      &=\id \,\, + \,\,(- \eta W_{1} - \eta^{2}
      W_{2}  - o(\eta^{2})) \,\, + (- \eta W_{1} - \eta^{2}
      W_{2}  - o(\eta^{2}))^{2} + o(\eta^{2}) \\
      &= \id \,\, + \,\,(- \eta W_{1} - \eta^{2}
      W_{2}  - o(\eta^{2})) \,\, + (- \eta W_{1} )^{2} + o(\eta^{2}) \\
      &= \id - \eta W_{1} + \eta^{2} \left( W_{1}^{2} -W_{2} \right) + o(\eta^{2})
\end{align*}
Evaluating the difference $W_{1}^{2} - W_{2}$, we find the simplification
\begin{align*}
    W_{1}^{2} - W_{2} &= \left( \Gamma - \frac{1}{\beta} P_{0}^{-1}
     \right)^{2} - \left(
      \frac{2 }{\beta^{2}} P_{0}^{-2} - \frac{1}{\beta} \Gamma P_{0}^{-1} - \frac{1}{\beta} P_{0}^{-1} \Gamma  
    \right) \\
    &= \Gamma^{2} - \frac{1}{\beta} \Gamma P_{0}^{-1} - \frac{1}{\beta} P_{0}^{-1} \Gamma + \frac{1}{\beta^{2}} P_{0}^{-2} 
    - \frac{2 }{\beta^{2}} P_{0}^{-2} + \frac{1}{\beta} \Gamma P_{0}^{-1} + \frac{1}{\beta} P_{0}^{-1} \Gamma  \\
    &= \Gamma^{2} - \beta^{-2} P_{0}^{-2}
\end{align*}
So that we conclude, up to $\eta^{2}$, the form of $V^{-1}$ to be
\begin{align*}
    V^{-1} = \id \,\, - \eta \left( \Gamma - \frac{1}{\beta} P_{0}^{-1}
     \right) + \eta^{2} \left(
     \Gamma^{2} - \beta^{-2} P_{0}^{-2}
     \right)
\end{align*}
Now, we return to $P$:
\begin{align*}
    P &= V^{-1} P_{0} V^{-1}  \\
    &= \left( \id \,\, - \eta \left( \Gamma - \frac{1}{\beta} P_{0}^{-1}
     \right) + \eta^{2} \left(
     \Gamma^{2} - \beta^{-2} P_{0}^{-2}
     \right) \right) P_{0} \left( \id \,\, - \eta \left( \Gamma - \frac{1}{\beta} P_{0}^{-1}
     \right) + \eta^{2} \left(
     \Gamma^{2} - \beta^{-2} P_{0}^{-2}
     \right) \right) + o(\eta^{2}) \\
     &= P_{0}  + \eta \left[ -\left( \Gamma - \frac{1}{\beta} P_{0}^{-1}
     \right) P_{0} - \eta P_{0} \left( \Gamma - \frac{1}{\beta} P_{0}^{-1}
     \right) \right] \\
     &+ \eta^{2} \left[ \left( \Gamma - \frac{1}{\beta} P_{0}^{-1}
     \right) P_{0} \left( \Gamma - \frac{1}{\beta} P_{0}^{-1}
     \right) + \left(
     \Gamma^{2} - \beta^{-2} P_{0}^{-2}
     \right) P_{0} +  P_{0} \left(
     \Gamma^{2} - \beta^{-2} P_{0}^{-2}
     \right) \right] + o(\eta^{2})
\end{align*}
As standard, the $o(\eta)$ term carries the usual Lyapunov ordinary differential equation for the covariance, where for $A = - \Gamma$ one has in the symmetrized coordinates of \cite{Halder2017} where $B Q B^{T} = 2 \beta^{-1} \id$:
\begin{align*}
    \eta \left[ -\left( \Gamma - \frac{1}{\beta} P_{0}^{-1}
     \right) P_{0} -  P_{0} \left( \Gamma - \frac{1}{\beta} P_{0}^{-1}
     \right) \right] = \eta \left[ A P_{0} + P_{0} A + \frac{2}{\beta} \id
     \right]
\end{align*}
For the order $\eta^{2}$ terms, we find
\begin{align*}
    & \left( \Gamma - \frac{1}{\beta} P_{0}^{-1}
     \right) P_{0} \left( \Gamma - \frac{1}{\beta} P_{0}^{-1}
     \right) + \left(
     \Gamma^{2} - \beta^{-2} P_{0}^{-2}
     \right) P_{0} +  P_{0} \left(
     \Gamma^{2} - \beta^{-2} P_{0}^{-2}
     \right) \\
     &= \left( \Gamma P_{0} \Gamma - \frac{1}{\beta} \Gamma - \frac{1}{\beta} \Gamma + \beta^{-2} P_{0}^{-1} \right)
     + \left(
     \Gamma^{2} - \beta^{-2} P_{0}^{-2}
     \right) P_{0} +  P_{0} \left(
     \Gamma^{2} - \beta^{-2} P_{0}^{-2}
     \right) \\
     &= \left( \Gamma P_{0} \Gamma - \frac{2}{\beta} \Gamma  + \beta^{-2} P_{0}^{-1} \right)
     + 
     \left( \Gamma^{2} P_{0} - \beta^{-2} P_{0}^{-1} \right) +   \left(
     P_{0}\Gamma^{2} - \beta^{-2} P_{0}^{-1}
     \right) \\
     &= \left( \Gamma P_{0} \Gamma - \frac{2}{\beta} \Gamma  + \beta^{-2} P_{0}^{-1} \right)
     + 
     \left( \Gamma^{2} P_{0} - \beta^{-2} P_{0}^{-1} \right) +   \left(
     P_{0}\Gamma^{2} - \beta^{-2} P_{0}^{-1}
     \right) \\
     &= \Gamma^{2} P_{0} +  P_{0}\Gamma^{2} + \Gamma P_{0} \Gamma - \frac{2}{\beta} \Gamma - \beta^{-2} P_{0}^{-1}
\end{align*}
Thus, for the modification differential equation of the form
\begin{align*}
    P(t+\eta) &= P_{0} + \eta (\dot{P}_{0} + \eta \mathcal{R} ) + \frac{\eta^{2}}{2} \Ddot{P}_{0} + o(\eta^{2}) \\
    &= P_{0} + \eta (A P_{0} + P_{0} A + 2 \beta^{-1} \id) + \eta^{2} (\mathcal{R} + \frac{1}{2} \Ddot{P}_{0}) 
\end{align*}
We may cancel the zero and first order terms directly, and equate the $\eta^{2}$ components to solve for the implicit regularization required for the continuous flow to coincide with the solution to the JKO problem;
\begin{align*}
    \mathcal{R} &= \Gamma^{2} P_{0} +  P_{0}\Gamma^{2} + \Gamma P_{0} \Gamma - \frac{2}{\beta} \Gamma - \beta^{-2} P_{0}^{-1} - \frac{1}{2} \left( A \Dot{P}_{0} + \Dot{P}_{0} A \right) \\
    &= \Gamma^{2} P_{0} +  P_{0}\Gamma^{2} + \Gamma P_{0} \Gamma - \frac{2}{\beta} \Gamma - \beta^{-2} P_{0}^{-1} - \frac{1}{2} \left( A (A P_{0} + P_{0} A + 2 \beta^{-1} \id) + (A P_{0} + P_{0} A + 2 \beta^{-1} \id) A \right) \\
    &= \Gamma^{2} P_{0} +  P_{0}\Gamma^{2} + \Gamma P_{0} \Gamma - \frac{2}{\beta} \Gamma - \beta^{-2} P_{0}^{-1} - \frac{1}{2} \left( A^{2} P_{0} + A P_{0} A + 2 \beta^{-1} A + A P_{0} A + P_{0} A^{2} + 2 \beta^{-1} A \right) \\
    &= \Gamma^{2} P_{0} +  P_{0}\Gamma^{2} + \Gamma P_{0} \Gamma - \frac{2}{\beta} \Gamma - \beta^{-2} P_{0}^{-1} - \frac{1}{2} \left( \Gamma^{2} P_{0} + 2\Gamma P_{0} \Gamma + P_{0} \Gamma^{2} - 4 \beta^{-1} \Gamma \right) \\
    &= \Gamma^{2} P_{0} +  P_{0}\Gamma^{2} + \Gamma P_{0} \Gamma - \frac{2}{\beta} \Gamma - \beta^{-2} P_{0}^{-1}  - \frac{1}{2} \Gamma^{2} P_{0} - \Gamma P_{0} \Gamma - \frac{1}{2} P_{0} \Gamma^{2} + 2 \beta^{-1} \Gamma \\
    &= \frac{1}{2}\Gamma^{2} P_{0} +  \frac{1}{2} P_{0}\Gamma^{2} - \beta^{-2} P_{0}^{-1}
\end{align*}
Thus, we identify the implicit regularizer to be the simple expression $\frac{1}{2}\Gamma^{2} P_{0} +  \frac{1}{2} P_{0}\Gamma^{2} - \beta^{-2} P_{0}^{-1}$.
\end{proof}

The implicit regularization notably deviates from the acceleration $(1/2)\Ddot{P}_{0}$ alone as in Euclidean gradient descent \cite{barrett2021implicit} and captures curvature-effects in Bures-Wasserstein space.

\begin{remark}[Verifying the implicit Regularization from the JKO-PDE.]
\end{remark}\label{rem:BW_from_PDE}

Note that 
\[
\dot{\mu}=\frac{d}{dt}\mathbb{E}_{\rho}x = \int x \partial_{t} \rho = -\int x \nabla \cdot (\rho \mathbf{v}) = \int \mathbf{v} \rho \, dx,\]
so that we can track the expected evolution of the implicit regularization of Proposition~\ref{prop:implicit_JKO}.
\begin{align*}
    -\frac{\eta}{2} \left( \frac{\partial}{\partial t} \wagrad\ - \frac{1}{2} \nabla \left\lVert \wagrad \right\rVert_{2}^{2} \right)
\end{align*}
from the Wasserstein gradient $\wagrad\ =\Gamma x - \beta^{-1} P^{-1} (x - \mu)$, we have
\begin{align*}
 &\mathbb{E}\left[\frac{\partial}{\partial t} \wagrad\ \right] = \mathbb{E}\left[- \beta^{-1} \partial_{t}P^{-1} (x - \mu) - \beta^{-1} P^{-1} \partial_{t} (-\mu) \right] \\
    &= \mathbb{E}\left[+ \beta^{-1} P^{-1} \dot{P} P^{-1} (x - \mu) - \beta^{-1} P^{-1} \Gamma \mu \right] = -\beta^{-1}P^{-1 }\Gamma\mu
\end{align*}
Meanwhile, for $H = (\Gamma - \beta^{-1}P^{-1})$ and $b=+\beta^{-1} P^{-1} \mu$, one finds
\begin{align*}
    &\mathbb{E} \left[\frac{1}{2}\nabla \left\lVert \wagrad \right\rVert_{2}^{2} \right] = \frac{1}{2} \nabla\left(
    x^{\top} H^{2} x + 2 x^{\top} H b
    \right) = H^{2}x + Hb \\
    &= \mathbb{E} \left[H(Hx + b) \right] \\
    &= H\mathbb{E} \left[((\Gamma - \beta^{-1}P^{-1})x + \beta^{-1} P^{-1} \mu) \right] = H \left( \Gamma \mu - \beta^{-1} P^{-1} \mathbb{E} \left[(x-\mu)\right]  \right) \\
    &= H \Gamma \mu = (\Gamma - \beta^{-1}P^{-1}) \Gamma \mu = \Gamma^{2}\mu - \beta^{-1} P^{-1} \Gamma \mu
\end{align*}
Now, taking the expectation of the convective acceleration (Proposition~\ref{prop:implicit_JKO}) one has
\begin{align*}
    &-\frac{\eta}{2} \left( \mathbb{E}\left[\frac{\partial}{\partial t} \wagrad \right] - \mathbb{E}\left[\frac{1}{2} \nabla \left\lVert \wagrad \right\rVert_{2}^{2} \right] \right) \\
    &= -\frac{\eta}{2} \left( 
    -\beta^{-1}P^{-1 }\Gamma\mu - \Gamma^{2}\mu + \beta^{-1} P^{-1} \Gamma \mu
    \right) \\
    &= + \frac{\eta}{2} \Gamma^{2} \mu = \frac{\eta}{2} (-A)^{2} \mu = \frac{\eta}{2} A^{2} \mu
\end{align*}
This recovers the analytical form of the implicit regularizer. We additionally derive the dynamics of the covariance matrix directly from the implicit regularization result.

It is an established result that the gradient-flow dynamics evolve according to the differential equation \cite{chewi2024statistical}
\begin{align}\label{eq:BW_dyn}
    \dot{P} = - P \mathbb{E}_{\rho} [\nabla^{2} \delta J] - \mathbb{E}_{\rho} [ \nabla^{2} \delta J ] P
\end{align}
Thus, we compute the Hessian of the first variation
\begin{align*}
    &\nabla^{2} \left(
    \frac{1}{2} x^{\top} \Gamma x \,\, + \frac{1}{\beta} \log\rho
    \right) \\
    &= \Gamma \,\, + \beta^{-1} \nabla^{2} \log\rho
\end{align*}
And find the Bures-Wasserstein gradient flow \cite{chewi2024statistical} of the covariance recovers the standard Lyapunov differential equation \cite{Halder2017}
\begin{align*}
    &\dot{P} = -P(\Gamma - \beta^{-1} P^{-1}) - (\Gamma - \beta^{-1} P^{-1})P \\
    &= PA + AP + 2\beta^{-1} \id
\end{align*}
For the proposed correction component $H^{\eta}(\rho)$ we require $\nabla^{2} \frac{\delta H^{\eta}}{\delta \rho}$, and thus focus on the correction of Proposition~\ref{prop:implicit_JKO}, $(1/2) ( \nabla \partial_{t} \delta J \, \, - (1/2) \nabla \lVert \nabla \delta J \rVert_{2}^{2} )$. For each term, we find
\begin{align*}
    &\partial_{t} \nabla^{2} \delta J = \partial_{t} ( \Gamma - \beta^{-1} P^{-1}) = - \beta^{-1} (- P^{-1} \dot{P} P^{-1})  \\
    &= \beta^{-1} P^{-1} (-\Gamma P - P \Gamma + 2 \beta^{-1} \id) P^{-1} \\
    &= -\beta^{-1} P^{-1} \Gamma - \beta^{-1} \Gamma P^{-1} + 2\beta^{-2} P^{-2}
\end{align*}
and, likewise, observe
\begin{align*}
    \nabla^{2} (1/2) \lVert \nabla \delta J \rVert_{2}^{2} &= \nabla^{2} \frac{1}{2} \left\langle (\Gamma -\beta^{-1} P^{-1})x + \beta^{-1} P^{-1} \mu, (\Gamma -\beta^{-1} P^{-1})x + \beta^{-1} P^{-1} \mu\right\rangle \\
    &= (\Gamma -\beta^{-1} P^{-1})^{\top} (\Gamma -\beta^{-1} P^{-1})
\end{align*}
In the symmetrized frame of \cite{Halder2017}, this is simply
\begin{align*}
    (\Gamma -\beta^{-1} P^{-1}) (\Gamma -\beta^{-1} P^{-1}) = \Gamma^{2} - \beta^{-1} \Gamma P^{-1} -\beta^{-1} P^{-1} \Gamma + \beta^{-2} P^{-2}
\end{align*}
Since these exhibit no $x$-dependence, the Hessian of the variation coincides with its expectation. Thus, we return to the gradient of the convective acceleration (Proposition~\ref{prop:implicit_JKO}) to find
\begin{align*}
    & \frac12 ( \nabla^{2} \partial_{t} \delta J \, \, - (1/2) \nabla^{2} \lVert \nabla \delta J \rVert_{2}^{2} ) \\
    &= \frac{1}{2} \left( -\beta^{-1} P^{-1} \Gamma - \beta^{-1} \Gamma P^{-1} + 2\beta^{-2} P^{-2} \, \, - \left( \Gamma^{2} - \beta^{-1} \Gamma P^{-1} -\beta^{-1} P^{-1} \Gamma + \beta^{-2} P^{-2} \right) \right) \\
    &= \frac{1}{2} \left( -\beta^{-1} P^{-1} \Gamma - \beta^{-1} \Gamma P^{-1} + 2\beta^{-2} P^{-2} \, \, - \Gamma^{2} + \beta^{-1} \Gamma P^{-1} + \beta^{-1} P^{-1} \Gamma - \beta^{-2} P^{-2} \right) \\
    &= -\frac{1}{2} \left( \, \Gamma^{2} - \beta^{-2} P^{-2} \right)
\end{align*}
Pre and post-multipying by $-P$ on both sides then yields the associated Bures-Wasserstein differential equation on the covariance \eqref{eq:BW_dyn}, which give the final form of the implicit regularization;
\begin{align*}
    &(1/2) (\Gamma^{2} - \beta^{-2} P^{-2}) P + (1/2) P(\Gamma^{2} - \beta^{-2} P^{-1}) \\
    &= \frac{1}{2}\Gamma^{2} P - \frac{1}{2} \beta^{-2} P^{-1} + \frac{1}{2} P \Gamma^{2} - \frac{1}{2}\beta^{-2} P^{-1} \\
    &= \frac{1}{2}\Gamma^{2} P + \frac{1}{2} P \Gamma^{2} - \beta^{-2} P^{-1}
\end{align*}
Thus, computing the implicit regularization directly from the variation exactly recovers the implicit regularization implied by the analytical update for the JKO problem with respect to the covariance $P$.

\section*{Supplemental Results}\label{sec:supp_res}

\begin{definition}{Otto-Metric}
    Let $\rho \in {P}_{ac}(\mathbb{R}^{d})$. The tangent space to ${P}_{ac}(\mathbb{R}^{d})$ is defined by
    \[
    T_{\rho}{P}_{ac}(\mathbb{R}^{d}) = \overline{ \{
    \nabla \psi : \psi: \mathbb{R}^{d} \to \mathbb{R} \quad \text{ Compactly supported and smooth}
    \} }^{L^{2}(\rho)}
    \]
    $T_{\rho}{P}_{ac}(\mathbb{R}^{d})$ has the associated metric
    \[
    \left\langle \xi_{1}, \xi_{2} \right\rangle_{\rho} = \int \langle \nabla \psi_{1}, \nabla \psi_{2} \rangle d \rho,
    \]
    where $\psi_{1}, \psi_{2}$ both solve the equations $\xi_{1} = - \nabla \cdot (\rho \nabla \psi_{1}), \xi_{2} = - \nabla \cdot (\rho \nabla \psi_{2})$. 
\end{definition}

In Proposition~\ref{prop:grad_first_var}, we recall a standard identity which identifies the gradient of the first variation of the Wasserstein distance. This identity is standard in the Otto calculus; see e.g. Lemma 1 of \cite{Lott2007}.

\begin{prop}\label{prop:grad_first_var}
Let $\rho \in {P}_{ac}(\mathbb{R}^{d})$ be a solution to the continuity equation $\partial_{t} \rho =- \nabla \cdot (\rho(x,t) \mathbf{v}(x,t))$ for a given velocity field $\mathbf{v}_{t}(x): [0,1] \times \mathbb{R}^{d}  \to \mathbb{R}^{d}$ in $T_{\rho} P_{ac}(\mathbb{R}^{d})$ (i.e. a gradient field) and a fixed measure $\rho^{(n,\eta)}$. Then;
\begin{align*}
    \nabla \frac{\delta}{\delta \rho} \int \frac{\delta W_{2}^{2}(\rho^{(n,\eta)}, \rho) }{\delta \rho} -\nabla\cdot \left( \rho_{t} \mathbf{v}_t \right) = 2 \,\mathbf{v}_t
\end{align*}
\end{prop}
\begin{proof}
Note the well-known \cite{Jordan1998,AGS2008,chewi2024statistical} result (given in \ref{proof:first_var} for completeness) that 
\[
\int \frac{\delta W_{2}^{2}(\rho^{(n,\eta)}, \rho)}{\delta \rho} d\chi = \frac{d}{d\epsilon} \left.  W_{2}^{2}(\rho^{(n,\eta)}, \rho + \epsilon \chi) \,\, \right|_{\epsilon = 0} = 2 \int \phi_{*}(x) d\chi(x)
\]
For $\phi_{*} \in \{ f : f\,\, \mathbf{cvx}, f: \mathbb{R}^{d}  \to \mathbb{R} \}$ the optimal Kantorovich potential from $\rho^{(n,\eta)} \to \rho$ whose push-forward satisfies $(\mathrm{id} - \nabla \phi_{*} )_{\sharp}\rho^{(n,\eta)} = \rho$ infinitesimally. In this case, as $\rho$ is the solution to the continuity equation
\[
\chi = - \nabla \cdot (\rho \mathbf{v}_t)
\]
The first-variation is with respect to the perturbation direction given by $\chi$, and thus
\begin{align*}
&\int \frac{\delta W_{2}^{2}(\rho^{(n, \eta)}, \rho)}{ \delta \rho} d\chi = 2 \int \phi_{*}(x) d \chi(x) = - 2 \int \phi_{*}(x)  \nabla \cdot (\rho \mathbf{v}_{t})dx = 2 \int\langle
\nabla\phi_{*}(x) , \mb{v}_{t} \rangle d\rho
\end{align*}
So that $2\nabla\phi_{*}(x)$ is the Riesz representer of the first-variation for all perturbations $\chi$, and thus defines the gradient. By identification of the first variation, with $\lim_{\epsilon \downarrow 0} \frac{1}{\epsilon} \left( \mathcal{F}(\rho+\epsilon \chi) - \mathcal{F}(\rho) \right) = \int \frac{\delta \mathcal{F}}{\delta \rho} d\chi$ we see that
\begin{align*}
\nabla\frac{\delta}{\delta \rho}\int \frac{\delta W_{2}^{2}(\rho^{(n,\eta)}, \cdot) }{\delta \rho} -\nabla\cdot \left( \rho_{t} \mathbf{v}_{t} \right) =
2 \nabla \frac{\delta}{\delta \rho}  \int\langle
\nabla\phi_{*}(x) , \mb{v}_{t} \rangle d\rho = 2 \nabla\phi_{*}(x) 
\end{align*}
thus, along a direction $\chi = - \nabla \cdot (\rho \mathbf{v}_t)$, one concludes
\[
\nabla \frac{\delta}{\delta\rho}  \int \frac{\delta W_{2}^{2}(\rho^{(n,\eta)}, \rho) }{\delta \rho} -\nabla\cdot \left( \rho_{t} \mathbf{v}_{t} \right) = 2 \mathbf{v}_{t}
\]
\end{proof}

By \cite{McCann2001}, one identifies that the transport map is given as $t(x) = \exp_{x} (- \nabla_{g} \phi_{*})$ $\rho$-a.e. $x$ for densities in ${P}_{ac}(M)$. Then one necessarily has for a functional direction $\chi = - \nabla_{g} \cdot (\rho_{t} \mb{v}_{t})$ that
\begin{align*}
    &\int \frac{\delta W_{2}^{2}(\rho^{(n,\eta)}, \rho)}{\delta \rho} \chi d \mathrm{vol}_{g} = \frac{d}{d\epsilon} \left.  W_{g}^{2}(\rho^{(n,\eta)}, \rho + \epsilon \chi) \,\, \right|_{\epsilon = 0} = 2 \int \phi_{*}(x) \chi d\vol \\
    &= 2 \int g( \nabla_{g}\phi_{*}(x), \mb{v}_{t}) \rho_{t} d\vol = 2\langle \nabla_{g}\phi_{*}(x), \mb{v}_{t} \rangle_{L^{2}(\rho_{t})}
\end{align*}
By identification of the first variation with the Kantorovich potential which yields the Monge map through $t(x) = \exp_{x} (- \nabla_{g} \phi_{*})$, one identifies from above that
\begin{align*}
    \frac{\delta}{\delta \rho}\int \frac{\delta W_{g}^{2}(\rho^{(n,\eta)}, \rho)}{\delta \rho} \chi d \mathrm{vol}_{g} = 2 \langle \nabla_{g}\phi_{*}(x), \mb{v}_{t} \rangle_{L^{2}(\rho_{t})}
\end{align*}
For any choice of $\chi$. Thus, from the Riesz representation theorem in the Otto-metric we have that $\nabla_{g}\phi_{*}(x)$ is the unique Riesz representer, so that 
\begin{align*}
    \nabla_{g}\frac{\delta}{\delta \rho}\int \frac{\delta W_{g}^{2}(\rho^{(n,\eta)}, \cdot)}{\delta \rho} \chi d \mathrm{vol}_{g} = 2\nabla_{g}\phi_{*}(x)
\end{align*}
And for a direction $\chi = - \nabla_{g} \cdot (\rho_{t} \mb{v}_{t})$ one finds that \begin{align*}
    \nabla_{g}\frac{\delta}{\delta \rho}\int \frac{\delta W_{g}^{2}(\rho^{(n,\eta)}, \rho)}{\delta \rho} \chi d \mathrm{vol}_{g} = 2\mb{v}_{t}
\end{align*}
As before. Note that the Riemannian generalization is the case shown in Lemma 1 of \cite{Lott2007}.

\begin{prop}
    Let $\rho^{(n,\eta)} \in {P}_{ac}(\mathbb{R}^{d})$ be a fixed measure, and suppose $\mathbf{v}_{t}: [0,1] \times \mathbb{R}^{d}  \to \mathbb{R}^{d}$ is a time-varying vector field. Assume the following boundary conditions holds at the domain boundaries
    \[ \left. \left( \partial_{t} ( \rho_{t} \mathbf{v}_{t} ) \cdot \bm{n}\right) \right|_{\partial \Omega}  = 0, \quad \left. \rho_{t} \right|_{\partial \Omega} = 0\]
so that not only does the density $\rho$ vanish on the boundary of the domain, but also the normal component $\bm{n}$ of the time-change of the mass-flux $\rho \mathbf{v}_{t}$. Then, one has the following integral over the partial in time $\partial_{t}$ of the mass-flux $(\rho_{t} \mathbf{v}_{t}  
    )$;
    \[
    \frac{1}{2}\nabla\frac{\delta}{\delta\rho} \int \frac{\delta W_{2}^{2}(\rho^{(n,\eta)}, \rho) }{\delta \rho} - \nabla\cdot\left( \partial_{t} (\rho_{t} \mathbf{v}_{t}  
    ) \right) = \partial_{t} \mathbf{v}_{t} +\frac{1}{2}\nabla \left\lVert \mathbf{v}_{t} \right\rVert_{2}^{2}
    \]
    and if $\mathbf{v}_{t} = - \wagrad (\rho_{t})(x(t)) := - \wagrad (t, x(t))$ is the Wasserstein gradient flow velocity, then
    \[
    \frac{1}{2}\nabla\frac{\delta}{\delta\rho} \int \frac{\delta W_{2}^{2}(\rho^{(n,\eta)}, \rho) }{\delta \rho} - \nabla\cdot\left( \partial_{t}  \left( - \wagrad (t, x(t)) \rho_{t}\right) \right) = - \partial_{t} \wagrad + \frac{1}{2}\nabla \left\lVert \wagrad \right\rVert_{2}^{2}
    \]
\end{prop}\label{prop:accn_HJ}
\begin{proof}
First, observe that
\begin{align*}
\nabla\frac{\delta}{\delta\rho} \int \frac{\delta W_{2}^{2}(\rho^{(n,\eta)}, \rho) }{\delta \rho} - \nabla\cdot\left( \partial_{t} (\rho_{t} \mathbf{v}_{t}  
    ) \right) = \nabla\frac{\delta}{\delta\rho} \int \frac{\delta W_{2}^{2}(\rho^{(n,\eta)}, \rho) }{\delta \rho} - \nabla\cdot\left( \partial_{t} \rho_{t} \mathbf{v}_{t}  + \rho_{t} \partial_{t} \mathbf{v}_{t}
     \right)
\end{align*}
By linearity of the divergence this equals the two terms
\begin{align*}
    = \nabla\frac{\delta}{\delta\rho} \int \frac{\delta W_{2}^{2}(\rho^{(n,\eta)}, \rho) }{\delta \rho} - \nabla\cdot\left( \rho_{t} \partial_{t} \mathbf{v}_{t}
     \right) + \nabla\frac{\delta}{\delta\rho} \int \frac{\delta W_{2}^{2}(\rho^{(n,\eta)}, \rho) }{\delta \rho} - \nabla\cdot\left( \partial_{t} \rho_{t} \mathbf{v}_{t} 
     \right)
\end{align*}
by Proposition~\ref{prop:grad_first_var}, we have that since $- \nabla\cdot\left( \rho_{t} \partial_{t} \mathbf{v}_{t}
     \right)$ is simply a continuity equation driven by the velocity $\partial_{t} \mathbf{v}_{t}$, the first term immediately becomes
\begin{align*}
    \nabla\frac{\delta}{\delta\rho} \int \frac{\delta W_{2}^{2}(\rho^{(n,\eta)}, \rho) }{\delta \rho} - \nabla\cdot\left( \rho_{t} \partial_{t} \mathbf{v}_{t}
     \right) = 2 \partial_{t} \mathbf{v}_{t}
\end{align*}
The other term we analyze independently, as it does not reduce to a simple integral against a continuity equation. The boundary conditions on the flux directly imply that we may apply integration by parts on terms of the form $\partial_{tt} \rho_{t}  = - \partial_{t}( \nabla \cdot (\rho_{t} \mathbf{v}_{t} ))  = -\nabla \cdot (\partial_{t}(\rho_{t} \mathbf{v}_{t}))$ and thus $-\nabla \cdot (-\nabla \cdot\left(\rho_{t} \mathbf{v}_{t}\right) \mathbf{v}_{t} ) $ as the condition
\[
\partial_{t} (\rho_{t} \mathbf{v}_{t}) \cdot \bm{n} = \underbrace{-\nabla \cdot\left(\rho_{t} \mathbf{v}_{t}\right) \mathbf{v}_{t} \cdot \bm{n} + \rho_{t} \partial_{t}\mathbf{v}_{t} \cdot \bm{n} = -\nabla \cdot\left(\rho_{t} \mathbf{v}_{t}\right) \mathbf{v}_{t} \cdot \bm{n} +0 = 0}_{\text{on }\partial \Omega}
\]
holds at the boundary $\partial \Omega$. Thus, returning to the second term
\begin{align*}
    & \nabla\frac{\delta}{\delta\rho} \int \frac{\delta W_{2}^{2}(\rho^{(n,\eta)}, \rho) }{\delta \rho} - \nabla\cdot\left( \partial_{t} \rho_{t} \mathbf{v}_{t}) \right) 
\end{align*}
we apply integration by parts to find
\begin{align*}
    &= +\nabla\frac{\delta}{\delta\rho} \int \left\langle \nabla\frac{\delta W_{2}^{2}(\rho^{(n,\eta)}, \rho) }{\delta \rho}, \mathbf{v}_{t} \right\rangle \partial_{t} \rho_{t} \\
    &= -\nabla\frac{\delta}{\delta\rho} \int \left\langle \nabla\frac{\delta W_{2}^{2}(\rho^{(n,\eta)}, \rho) }{\delta \rho}, \mathbf{v}_{t} \right\rangle \nabla \cdot (\rho_{t} \mathbf{v}_{t} )
    \\
    &= + \nabla\frac{\delta}{\delta\rho} \int \left\langle \mathbf{v}_{t}, \nabla\left\langle \nabla\frac{\delta W_{2}^{2}(\rho^{(n,\eta)}, \rho) }{\delta \rho}, \mathbf{v}_{t} \right\rangle \right\rangle d\rho_{t} \\
    &= \nabla\frac{\delta}{\delta\rho} \int \left\langle \mathbf{v}_{t}, \nabla^{2}\frac{\delta W_{2}^{2}(\rho^{(n,\eta)}, \rho) }{\delta \rho} \mathbf{v}_{t} + \nabla \mathbf{v}_{t} \nabla \frac{\delta W_{2}^{2}(\rho^{(n,\eta)}, \rho) }{\delta \rho} \right\rangle  d\rho_{t} \\
    &= \nabla\frac{\delta}{\delta\rho} \int \left\langle \mathbf{v}_{t}, \nabla^{2}\frac{\delta W_{2}^{2}(\rho^{(n,\eta)}, \rho) }{\delta \rho} \mathbf{v}_{t}  \right\rangle  d\rho_{t} + \nabla\frac{\delta}{\delta\rho} \int \left\langle \mathbf{v}_{t},  \nabla \mathbf{v}_{t} \nabla \frac{\delta W_{2}^{2}(\rho^{(n,\eta)}, \rho) }{\delta \rho} \right\rangle  d\rho_{t} 
    \end{align*}
    From Proposition~\ref{prop:grad_first_var}, we have that this simply becomes
    \begin{align*}
        &\nabla\frac{\delta}{\delta\rho} \int \left\langle \mathbf{v}_{t}, \nabla\mathbf{v}_{t} \mathbf{v}_{t}  \right\rangle  d\rho_{t} + \nabla\frac{\delta}{\delta\rho} \int \left\langle \mathbf{v}_{t},  \nabla \mathbf{v}_{t} \mathbf{v}_{t} \right\rangle  d\rho_{t} = 2 \nabla\frac{\delta}{\delta\rho} \int \left\langle \mathbf{v}_{t}, \frac{1}{2}\nabla \lVert \mathbf{v}_{t} \rVert_{2}^{2} \right\rangle  d\rho_{t} \\
        &=  \nabla\frac{\delta}{\delta\rho} \int \left\langle \mathbf{v}_{t}, \nabla \lVert \mathbf{v}_{t} \rVert_{2}^{2} \right\rangle  d\rho_{t}  =\nabla\frac{\delta}{\delta\rho} \int \lVert \mathbf{v}_{t} \rVert_{2}^{2}  d \left(- \nabla \cdot (\rho_{t} \mathbf{v}_{t} ) \right) = \nabla\frac{\delta}{\delta\rho} \int \lVert \mathbf{v}_{t} \rVert_{2}^{2}  d (\partial_{t} \rho_{t}) \\
        &= \nabla \lVert \mathbf{v}_{t} \rVert_{2}^{2}
    \end{align*}
    Thus, we have that 
    \begin{align*}
        \nabla\frac{\delta}{\delta\rho} \int \frac{\delta W_{2}^{2}(\rho^{(n,\eta)}, \rho) }{\delta \rho} - \nabla\cdot\left( \partial_{t} (\rho_{t} \mathbf{v}_{t}  
    ) \right) = 2 \left( \partial_{t} \mathbf{v}_{t} +\frac{1}{2}\nabla \left\lVert \mathbf{v}_{t} \right\rVert_{2}^{2} \right)
    \end{align*}
    And for $\mathbf{v}_{t} = - \wagrad(\rho_{t}) (x) := - \wagrad(t, x)$ this becomes
    \begin{align*}
        \nabla\frac{\delta}{\delta\rho} \int \frac{\delta W_{2}^{2}(\rho^{(n,\eta)}, \rho) }{\delta \rho} - \nabla\cdot\left( \partial_{t} (\rho_{t} -\wagrad  
    ) \right) = 2 \left( - \partial_{t} \wagrad + \frac{1}{2}\nabla \left\lVert \wagrad \right\rVert_{2}^{2} \right)
    \end{align*}
    So that we conclude.
\end{proof}

Supposing standard boundary conditions on the flux, all integration by parts identities above may be replaced with their Riemannian analogue. Moreover, by the extension of Proposition~\ref{prop:grad_first_var} to the Riemannian case (i.e. from Lemma 1 of \cite{Lott2007}), one finds that the Riemannian gradient of the first variation coincides with the velocity of the tangent $\chi$, so that the proof goes through directly with $\nabla \to \nabla_{g}$ and $\left\lVert  \right\rVert_{2}^{2} \to \left\lVert  \right\rVert_{g}^{2}$.

In the next Lemma, we identify the second variation of the Wasserstein distance $W_{2}^{2}(\rho^{(n, \eta)}, \rho)$ for a $\rho \in {P}_{ac}$ for the particular case where one reference measure $\rho^{(n, \eta)}$ is held fixed and is the initial condition for $\rho: \rho(0) = \rho^{(n, \eta)}$.

\begin{lemma}\label{lemma:Wasserstein_SecondVar}
Consider a Wasserstein gradient flow on $\rho_{t}$ defined by 
\begin{align*}
\partial_{t}\rho_{t} = +\nabla \cdot (\rho_{t} \wagrad (\rho_{t}))
\end{align*}
from a fixed initial datum $\rho_{0} =\rho^{(n, \eta)}$. Then, the second variation of the squared Wasserstein-distance $W_{2}^{2}(\rho^{(n, \eta)}, \rho)$ along the flow, against this fixed initial measure, is given by
\begin{align*}
& \frac{1}{2}\int \frac{\delta^{2} W_{2}^{2}(\rho^{(n, \eta)}, \rho)}{ \delta \rho(x) \delta \rho(x')} d\left(\frac{\partial \rho_{t}}{\partial t}\right)(x)d\left(\frac{\partial \rho_{t}}{\partial t}\right)(x') = \left\lVert \nabla \frac{\delta J}{\delta \rho} \right\rVert_{L^{2}(\rho_{t})}^{2}
\end{align*}
In the case that the flow is Riemannian on $(M, g)$, so that $\partial_{t}\rho_{t} = +\nabla_{g} \cdot (\rho_{t} \nabla_{g} \frac{\delta J}{\delta \rho} (\rho_{t}))$, one similarly identifies the second variation as
\begin{align*}
    = \int \left\lVert \,\rgrad \frac{\delta J}{\delta \rho} \right\rVert_{g}^{2} d \rho_{t}
\end{align*}
\end{lemma}
\begin{proof}
Let us begin in the Euclidean case by analyzing the term
\begin{align*}
& \delta^{2}W_{2}^{2}[\partial_{t} \rho_{t}, \partial_{t} \rho_{t}] = \int \frac{1}{2} \frac{\delta^{2} W_{2}^{2}(\rho^{(n, \eta)}, \rho)}{ \delta \rho(x) \delta \rho(x')} d\left(\frac{\partial \rho_{t}}{\partial t}\right)(x)d\left(\frac{\partial \rho_{t}}{\partial t}\right)(x') \\
&= \frac{1}{2}\int \frac{\delta^{2} W_{2}^{2}(\rho^{(n, \eta)}, \rho)}{ \delta \rho(x) \delta \rho(x')} \nabla \cdot \left(\rho_{t} \nabla \frac{\delta J}{\delta \rho} \right) \nabla \cdot \left(\rho_{t} \nabla \frac{\delta J}{\delta \rho}\right) dx dx'
\end{align*}

Note that the second variation, for a slot $\rho$ fixed, we define $\mathcal{F}(\cdot) := W_{2}^{2}(\cdot, \rho^{(n,\eta)} )$ and take the expansion in variational derivatives as
\begin{align*}
    W_{2}^{2}(\rho + \eta \chi, \rho ) := \mathcal{F}(\rho + \eta \chi) = \mathcal{F}(\rho) + \eta \delta \mathcal{F} \mid_{\rho} (\chi) + \eta^{2} \delta^{2} \mathcal{F} \mid_{\rho}(\chi) + o(\eta^{2})
\end{align*}
Clearly, $\mathcal{F}(\rho) = W_{2}^{2}(\rho, \rho ) = 0$. For the second variation under consideration, suppose we consider a tangent $\chi = - \nabla \cdot (\rho \mathbf{V}) = - \nabla \cdot (\rho \nabla \varphi)$. For small $\alpha$, one has that $\rho + \alpha \chi = (\mathbf{i} + \alpha \nabla \varphi)_{\sharp} \rho$, so that
\begin{align*}
    \mathcal{F} (\rho + \alpha \chi) &= \inf_{\gamma \in \Gamma(\rho, \rho + \alpha \chi)} \int_{\mathsf{X} \times \mathsf{X}} \lVert x_{1} - x_{2} \rVert_{2}^{2} d\gamma(x_{1}, x_{2}) = \int_\mathsf{X} \lVert x_{1} - T^{\star}_{\alpha}(x_{1}) \rVert_{2}^{2} d \rho \\
    & = \int \lVert x_{1} - \left( x_{1} + \alpha \nabla \varphi(x_{1}) \right) \rVert_{2}^{2} d \rho = \alpha^{2} \int \lVert \nabla \varphi(x_{1}) \rVert_{2}^{2} d \rho
\end{align*}
Thus, the second variation is given as
\begin{align*}
    \frac{d}{d\alpha^{2}} \mathcal{F}(\rho + \alpha \chi) \mid_{\alpha = 0} = \frac{d}{d\alpha^{2}}\alpha^{2} \int \lVert \nabla \varphi(x_{1}) \rVert_{2}^{2} d \rho = 2\int \lVert \nabla \varphi(x_{1}) \rVert_{2}^{2} d \rho \implies \delta^{2} \mathcal{F} \mid_{\rho} = 2\int \lVert \nabla \varphi(x_{1}) \rVert_{2}^{2} d \rho
\end{align*}
For the gradient-flow tangent $\partial_{t}  \rho^{\wass}$ it holds that $\mathbf{V} = \nabla \varphi(x_{1})$ \cite{AGS2008}, so we conclude the Euclidean case.

For the case with symmetric, bilinear Riemannian metric $g: T_{p}M \, \times \, T_{p}M \to \mathbb{R}_{+}$ and distance $D^{2} W_{g}^{2}$ we similarly have $W_{g}^{2}(\rho, \rho ) = 0$ and for step $\alpha$ have the analogue of the Euclidean update for generic $M$ that $\rho_{t}+\alpha\chi =\mathrm{exp}_{x}(\alpha \nabla \varphi)_{\sharp} \rho_{t}$ \cite{McCann2001}, since
\begin{align*}
    \mathcal{F} (\rho + \alpha \chi) &= \inf_{\gamma \in \Gamma(\rho, \rho + \alpha \chi)} \int_{\mathsf{X} \times \mathsf{X}} d_{g}^{2}( x_{1} , x_{2} ) d\gamma(x_{1}, x_{2}) = \int_\mathsf{X} d_{g}^{2}(x_{1} , T^{\star}_{\alpha}(x_{1}) ) d \rho \\
    & = \int d_{g}^{2}( x_{1} ,\left( \mathrm{exp}_{x_{1}}(\alpha \nabla \varphi) \right) ) d \rho 
\end{align*}
By \cite{Monera2013}, we have an expansion of the exponential map in orders of $\alpha$ given by
\begin{align*}
  &\int d_{g}^{2}( x_{1} , \left( \mathrm{exp}_{x_{1}}(\alpha \nabla \varphi) \right) ) d \rho = \int d_{g}^{2}( x_{1} , \left( 
  x_{1}
  + \alpha \gamma_{v}'(0) + (\alpha^{2}/2) \gamma_{v}''(0) + \cdots
   \right) ) d \rho \\
   &= \int g\left( 
  \alpha \gamma_{v}'(0) + (\alpha^{2}/2) \gamma_{v}''(0) + \cdots, \,\, \alpha \gamma_{v}'(0) + (\alpha^{2}/2) \gamma_{v}''(0) + \cdots
   \right) d \rho
\end{align*}
Using that the first derivative of the geodesic $\gamma_{v}$ coincides with the gradient of the potential $\gamma_{v}'(0)=\nabla_{g}\varphi$, and using bilinearity and homogenity of $g$, we have that the previous display equals
\begin{align*}
   & \int g\left( 
  \alpha \gamma_{v}'(0), \alpha \gamma_{v}'(0)\right) + g\left(\alpha \gamma_{v}'(0), (\alpha^{2}/2) \gamma_{v}''(0) \right)  \\
  &\qquad +g\left((\alpha^{2}/2) \gamma_{v}''(0), \alpha \gamma_{v}'(0)\right)
  + g\left( (\alpha^{2}/2) \gamma_{v}''(0),  (\alpha^{2}/2) \gamma_{v}''(0)
  \right) d \rho + o(\alpha^{2}) \\
  &= \int \left( \alpha^{2} g\left( 
  \rgrad \varphi,  \rgrad \varphi\right) + \alpha^{3}g\left( \rgrad \varphi,  \gamma_{v}''(0) \right) 
  + (\alpha^{2}/2)^{2}g\left( \gamma_{v}''(0), \gamma_{v}''(0)
  \right) \right) d \rho + o(\alpha^{2}) 
\end{align*}
Thus, it evidently holds that
\begin{align*}
    \frac{d}{d\alpha^{2}} \mathcal{F}(\rho + \alpha \chi) \mid_{\alpha = 0} = \frac{d}{d\alpha^{2}}\alpha^{2} \int \lVert \rgrad \varphi(x_{1}) \rVert_{g}^{2} d \rho = 2\int \lVert \rgrad \varphi(x_{1}) \rVert_{g}^{2} d \rho \implies \delta^{2} \mathcal{F} \mid_{\rho} = 2\int \lVert \rgrad \varphi(x_{1}) \rVert_{g}^{2} d \rho
\end{align*}
So the conclusion generalizes to the Riemannian case with minimal change.
\end{proof}

\begin{proof}[First Variation of the Wasserstein-Distance.]\label{proof:first_var}
\begin{align*}
    \frac{\delta}{\delta \rho} W_{2}^{2}(\rho, \rho^{(n,\eta)}) = \lim_{\epsilon \downarrow 0} \frac{W_{2}^{2}(\rho + \epsilon \chi, \rho^{(n,\eta)})}{\epsilon} = 2 \lim_{\epsilon \downarrow 0} \frac{1}{\epsilon} \sup_{\psi \in \Psi^{c}} \int \psi d(\rho + \epsilon \chi) + \int \psi^{c} d\rho^{(n,\eta)}
\end{align*}
Where $\psi^{c}(y) = \inf\{ \lVert x - y \rVert_{2}^{2} - \psi(x) \}$ is the c-conjugate of $\psi$ and $\Psi^{c}$ denotes the set of $\psi$ such that $\psi(x) + \psi^{c}(y) \leq \lVert x - y \rVert_{2}^{2}$. For $\psi^{*}$ optimal for the problem above, we have
\begin{align*}
    = 2 \lim_{\epsilon \downarrow 0} \frac{1}{\epsilon} \int \psi^{*} d(\rho + \epsilon \chi) + \int \psi^{*,c} d\rho^{(n,\eta)} = 2 \int \psi^{*} d\chi
\end{align*}
So that $2 \psi^{*}$, for $\psi^{*}$ the optimal potential from $\rho \to \rho^{(n,\eta)}$, is equal to $\delta W_{2}^{2}(\rho, \rho^{(n,\eta)})/\delta \rho$, recapitulating the standard result \cite{AGS2008, Jordan1998, chewi2024statistical}.
\end{proof}

\section{Implicit Regularization of Riemannian Gradient Descent}\label{sec:Riemann_IR}

\subsection{Background on Optimization over Riemannian Manifolds.} Let $M$ denote a Riemannian manifold. For a point $x \in M$ we denote the tangent space of $M$ at $x$ by $T_{x} M$ with cotangent space $T_{x}^{\star} M$. Let the Riemannian metric $\langle \cdot, \cdot\rangle \mid_{T_{x}M} := g(\cdot,\cdot)$ denote the norm on $T_{x} M$ given by $g: T_{x} M \times T_{x}^{\star} M \to \mathbb{R}_{+}$, which may be expressed as the two-form $g_{x} = \sum_{ij} g_{ij}(x)\,\, dx^{i} \mid_{x} \otimes\, dx^{j} \mid_{x}$. Consider a smooth loss function $J: M \to \mathbb{R}$ over the manifold. Riemannian optimization concerns finding a minimum $x^{\star}$ for the problem
\begin{align}\label{eq:manifold_loss}
    x^{\star}:=\argmin_{x \in M} J(x)
\end{align}
Riemannian gradient descent and Riemannian gradient flow arise as two natural approaches arise for this problem. In both cases, one relies on the \emph{Riemannian gradient}, $ \rgrad J \in T_{x} M $, which is the vector uniquely defined by $D J(x) [V] = g( V, \rgrad J )$ for the differential $DJ (x) [V]$ the linear map defined by $DJ(x) [V] = (J \circ \gamma)'(0)$ for a smooth curve $\gamma \in M$ with $\gamma(0) = x, \,\, \gamma'(0) = V$. In Riemannian gradient flow, one minimizes $J$ through the following ordinary differential equation on $M$
\begin{align}
    \dot{x} = - \rgrad J (x) \label{eq:RGF-1}
\end{align}
While Riemannian gradient flow is guaranteed to remain on $M$, unlike in the case of Euclidean gradient descent, the discretization of \ref{eq:RGF-1} given by $x^{(k+1)} = x^{(k)} - \eta \rgrad J$ generally fails to remain. To formalize a discretization which ensures feasibility, for any $p,q \in M$ denote the Riemannian geodesic distance between the points by $d_{M}(p,q)$ by the infimal length of a curve $\gamma \in C^{\infty}[0,1]$ parametrized by arc length which satisfies $\gamma(0) = p$ and $\gamma(1) = q$.
\begin{align}
    d(p,q) \,\,:=\inf_{\substack{\gamma \in C^{\infty}[0,1], \\
     \gamma(0) = p, \,\gamma(1) = q}} \int_{0}^{1} g( \dot{\gamma}(s), \dot{\gamma}(s)) ds
\end{align}
In \emph{Riemannian gradient descent} one minimizes \ref{eq:manifold_loss} by choosing a step-size $\eta \in \mathbb{R}_{+}$ and solving the following proximal-point implicit step for iterations $k \in [T]$
\begin{align}\label{eq:RGD_Implicit}
    x^{(k+1)} \gets \argmin_{x \in M} J(x) + \frac{1}{2\eta} d_{M}(x, x^{(k)})
\end{align}
Which ensures that the sequence of iterates $(x^{(t)})_{t=1}^{T} \subset M$. An equivalent \emph{explicit} form for this step is given by the \emph{exponential map}. Let $(x, v) \in TM$ denote a point in the tangent bundle. The exponential map $\mathrm{exp}_{x}(v): T M \to M$ is defined by $\mathrm{exp}_{x}(v) = \gamma_{v}(1)$ where $\gamma_{v}: [0,1] \to M$ is the unique geodesic p.a.l. satisfying $\gamma_{v}(0) = x$ and $\gamma_{v}'(0) = v$. Given this definition, the equivalent explicit step for \ref{eq:RGD_Implicit} is given by the update
\begin{align}\label{eq:RGD_Explicit}
    x^{(k+1)} \gets \mathrm{exp}_{x^{(k)}} \left(
    -\eta \rgrad J (x)
    \right)
\end{align}
Which constitutes the Riemannian analogue of the explicit Euclidean gradient descent update: in Euclidean space, $\mathrm{exp}_{x^{(k)}}(-\eta v) = x^{(k)} - \eta v$.

\subsection{Derivation of the Forward-Euler Case}

As Riemannian gradient descent offers explicit updates, demonstrating its implicit bias follows a similar recipe to that of gradient descent over the Euclidean manifold $\mathbb{R}^{d}$.

Suppose that we are given a Riemannian manifold $(M, g)$, where $M \subset \mathbb{R}^{d}$ is a smooth manifold and $g$ denotes its Riemannian metric. In particular, as standard, for $L^{2}(T_{x} M; \mathbb{R})$ the space of bilinear functions $\phi: T_{x} M \times T_{x} M \to \mathbb{R}$ one can express any such function in terms of basis elements of the bilinear forms $dx^{i} \mid_{x} \otimes \, d x^{j} \mid_{x}$ where each $dx^{i}$ forms a basis of the dual space $(T_{x} M)^{*}$ to the tangent space $T_{x} M$. This is to say, for any such $\phi$ one may write $\phi = \sum_{i,j} \alpha_{i,j} \, dx^{i} \otimes dx^{j}$. The Riemannian metric $g$ defines, for all points $x \in M$, a bilinear form $x \in M \to g_{x} \in L^{2}(T_{x} M; \mathbb{R})$ such that (1) symmetry holds $g_{x}(X,Y) = g_{x}(Y,X)$, (2) the metric is positive-definite $g_{x}(X,X) > 0$ for all non-zero $X$, and written in the basis of bilinear forms $dx^{i} \mid_{x} \otimes \, d x^{j} \mid_{x}$ one has the Riemannian metric at $x$, $g_{x} = \sum_{i,j} g_{i,j} (x)\, dx^{i} \mid_{x} \otimes \,\, dx^{j} \mid_{x}$, is differentiable so that $\frac{\partial}{\partial x^{k}} g_{ij}(x)$ are well-defined.

For a gradient-descent procedure on this Riemannian manifold, we require a smooth loss function defined on $M$, which we denote by $J: M \to \mathbb{R}$. The unique Riemannian gradient of $J$, $\rgrad J$, is defined by:
\[
\rgrad J:  \quad D J(X)[V] = g(V, \rgrad J (X))
\qquad \forall \, (V,X) \in T_{x} M
\]
Supposing we operate in local coordinates $(x^{1}, \cdots, x^{n})$ where we have the partials $\frac{\partial J}{\partial x^{i}}$, for the co-vector (1-form) $D J(\cdot)$ we have the coordinate representation
\begin{align*}
& D J(X) = \sum_{i} \frac{\partial J}{\partial x^{i}}(X) dx^{i}  = g(\cdot, \rgrad J(X))
\end{align*}
Focusing on the basis vectors $\partial_{j}$ of the tangent space, where any $v \in T_{p}M \implies v = v^{i} \partial_{i}$, we have
\begin{align*}
& D J(X) [\partial_{j}] = \sum_{i} \frac{\partial J}{\partial x^{i}}(X) dx^{i}  [\partial_{j}] = \sum_{i} \frac{\partial J}{\partial x^{i}}(X) \delta_{ij} = \frac{\partial J}{\partial x^{j}} 
\end{align*}
and
\begin{align*}
&g(\rgrad J(X), \partial_{j}) =
\left( \sum_{i,k} g_{i,k} \, dx^{i} \otimes dx^{k}  \right)[\rgrad J(X), \partial_{j}] \\
&= \sum_{i,k} \delta_{k,j} g_{i,k} dx^{i} [\rgrad J (X)] 
= \sum_{i} g_{i,j} dx^{i} [\rgrad J (X)]
=
g_{i,j} \rgrad J(X)^{i}
\end{align*}
so that in local coordinates, one may write the Riemannian gradient as
\begin{align*}
    g_{i,j} \rgrad J(X)^{i} = \frac{\partial J}{\partial x^{j}} \implies \rgrad J(X)^{i} = \sum_{j} g^{ij} \frac{\partial J}{\partial x^{j}}
\end{align*}
where the raised indices $g^{ij}$ correspond to the inverse metric-tensor; thus, one may write the Riemannian gradient in local coordinates succinctly in terms of the standard gradient as
\[
\rgrad J := (G)^{-1} \nabla J (X)
\]
where $G$ makes explicit the matrix-form of the 2-tensor at the point of interest. Riemannian gradient flow is given by the expression
\begin{align}\label{eq:RGF}
\Dot{X} = - \rgrad J (X)
\end{align}
whose flow can be shown to remain on $M$ exactly. For discrete steps over the manifold, discrete steps on velocities in the tangent space may lead to iterates $x_{k} \in \mathbb{R}^{d} \setminus M$, requiring the notion of either a retraction or the exponential map. For $x \in M$, the exponential map, denoted by $\exp_{x}: T_{x}M \to M$, is defined as
\begin{align*}
\exp_{x} (v) = \gamma_v(1) 
\end{align*}
for $v \in T_{x} M$ a tangent vector at point $x$, and $\gamma_{v}(t)$ the unique geodesic p.a.l. satisfying $\gamma_{v}(0) = x$, $\Dot{\gamma_{v}}(0) = v$. Given this definition of $\expp$, in Riemannian gradient descent one takes the canonical step for a step-size $\eta > 0$ of
\begin{align*}
    x_{n+1} = \exp_{x_{n}} (- \eta \rgrad J (x_{n})),
\end{align*}
where $\eta$ is sufficiently small so $- \eta \rgrad J (x_{n})$ remains in the injectivity radius of the exponential map about $x_{n}$. In Euclidean space, $\exp_{x_{n}}(-\eta \nabla J (x_{n})) := x_{n} - \eta \nabla J (x_{n})$, which reproduces standard gradient descent.

\begin{proof}[Proof of Implicit Bias of Riemannian Gradient Descent (Forward-Euler Case).]
Let us seek a continuous solution to a modified flow
\[
\Dot{\tilde{x}} = \mathcal{F}(\tilde{x})
\]
whose solution matches Riemannian gradient up to second-order in the step-size $\eta$. In particular, in a parallel manner to \cite{barrett2021implicit} we propose the Ansatz for the solution form for this modified flow of
\[
\Dot{\tilde{x}} = -\rgrad J (\tilde{x}) + \eta \Fimp (\tilde{x}) + o(\eta)
\]
where we scale $n \eta = t, (n+1)\eta = t+\eta$ and suppose the initial condition of the modified flow at time $t$ equals $\tilde{x}(t) :=x_{n}$. By definition, for $\lambda \in (0,1)$ $\tilde{x}(t+\lambda \eta)$ may not reside in $M$, but only stipulate that up to $o(\eta^{2})$ the final condition resides on the manifold, so that $\tilde{x}(t+\eta) \simeq_{o(\eta^{2})} x_{n+1} \in M$. Thus, the expansion of the corrected continuous $\tilde{x}(t + \eta)$ proceeds directly as
\begin{align*}
    & \tilde{x}(t + \eta) = \tilde{x}(t) + \eta \mathcal{F} (\tilde{x}(t)) + \frac{\eta^{2}}{2} \frac{d}{ d t} \left[ \mathcal{F} (\tilde{x}(t)) \right] + o(\eta^{2}) \\
    & = \tilde{x}(t) + \eta \mathcal{F} (\tilde{x}(t)) + \frac{\eta^{2}}{2} \nabla \mathcal{F} (\tilde{x}(t)) \Dot{\tilde{x}}(t) + o(\eta^{2}) \\
    &= \tilde{x}(t) + \eta \left( -\rgrad J (\tilde{x}(t)) + \eta \Fimp (\tilde{x}(t)) + o(\eta^{2}) \right) + \frac{\eta^{2}}{2} \nabla \mathcal{F} (\tilde{x}(t)) \mathcal{F} (\tilde{x}(t)) + o(\eta^{2}) \\
    &= \tilde{x}(t) - \eta \rgrad J (\tilde{x}) + \eta^{2} \left( \Fimp (\tilde{x}(t)) + \frac{1}{2} \nabla (-\rgrad J(\tilde{x}(t))) (- \rgrad J(\tilde{x}(t)))
    \right) + o(\eta^{2})
\end{align*}
The primary distinction is that the solution to the discrete iterate $\tilde{x}(t + \eta)$ is no longer given by the direct gradient step, but by the exponential map as
\begin{align*}
    \tilde{x}(t+\eta) + o(\eta^{2}) = x_{n+1} := \exp_{x_{n}} (- \eta \rgrad J (x_{n}))
\end{align*}
Thus, we find the identification of the implicit correction term to be
\begin{align*}
     &\exp_{x_{n}} (- \eta \rgrad J (x_{n})) 
     \\ &= \tilde{x}(t) - \eta \rgrad J (\tilde{x}) + \eta^{2} \left( \Fimp (\tilde{x}(t)) + \frac{1}{2} \nabla (-\rgrad J(\tilde{x}(t))) (- \rgrad J(\tilde{x}(t)))
    \right) + o(\eta^{2})
\end{align*}
Since $\exp_{x_{n}} (- \eta \rgrad J (x_{n}))$ coincides with the unique geodesic $\gamma_{v}:[0,1] \to M$ satisfying
\[
\gamma_{v}(0) = x_{n}, \quad \gamma_{v}(1) = x_{n+1} := \exp_{x_{n}} (- \eta \rgrad J (x_{n})), \quad \gamma_{v}'(0) = -\eta \rgrad J (x_{n})
\]
it holds that its expansion coincides with the geodesic \cite{Monera2013}, so that
\[\exp_{x_{n}} (- \eta \rgrad J (x_{n})) = \gamma_{v}(1) = \gamma_{v}(0) + \eta \gamma_{v}'(0) + \frac{\eta^2}{2} \gamma_{v}''(0) + \cdots := x_{n} - \eta \rgrad J (x_{n}) + \frac{\eta^{2}}{2} \gamma_{v}''(0) + o(\eta^{2})
\]

Moreover, \cite{Monera2013} identifies the second derivative of the geodesic at time $t=0$ to be the normal component of the directional derivative at a point $x$, $(D_{X} Y)^{\perp}$. In particular, decomposing $T_{x} \mathbb{R}^{d} = T_{x} M \oplus (T_{x} M)^{\perp} = T_{x} M \oplus N_{x} M$ one identifies this component as a bilinear map from the tangent bundle to the space of smooth sections $\Gamma(N M)$ over the normal bundle, $\alpha: T M \times T M \to \Gamma(N M)$ defined by 
\begin{align*}
    \alpha(X, Y) = (D_{X} Y)^{\perp}
\end{align*}
Denoting $\nu$ to be the Gauss map which gives the normal component of the vector field, i.e. if $M$ is given by a hypersurface element, one has that
\begin{align*}
    &(D_{X} Y)^{\perp} = \langle D_{X} Y, \nu \rangle \nu  = \II(X,Y) \nu
\end{align*}
where $\II(X,Y)$ denotes the second fundamental form and $D_{X}$ the standard covariant derivative. Thus, $\gamma_{v}''(t) = \II_{x}[v(t), v(t)] \nu(t)$ and at $t=0$ one has that $\gamma_{v}''(0) = \II_{x}[v,v] \nu$ for $v = \rgrad J(x_{n})$. Thus for the Riemannian gradient descent step one has that the correction is 
\[
\II_{x}[v,v]\nu := \II_{x_{n}}[- \rgrad J (x_{n}), - \rgrad J (x_{n})]\nu
\]
We then arrive at the identification
\begin{align*}
    &x_{n} - \eta \rgrad J (x_{n}) + \frac{\eta^{2}}{2} \II_{x_{n}}[- \rgrad J (x_{n}), - \rgrad J (x_{n})] \nu + o(\eta^{2}) \\
    &= x_{n} - \eta \rgrad J (x_{n}) + \eta^{2} \left( \Fimp (x_{n}) + \frac{1}{2} \nabla (-\rgrad J(x_{n})) (- \rgrad J(x_{n}))
    \right) + o(\eta^{2})
\end{align*}
Canceling the zero and first order terms, we arrive at the identification for $\Fimp$ (up to $o(\eta^{2})$) of
\begin{align*}
    \frac{\eta^{2}}{2} \II_{x_{n}}[- \rgrad J (x_{n}), - \rgrad J (x_{n})] \nu = \eta^{2} \left( \Fimp (x_{n}) + \frac{1}{2} \nabla (-\rgrad J(x_{n})) (- \rgrad J(x_{n}))
    \right)
\end{align*}
Thus, rearranging the terms, we see the implicit regularizer is given as the difference between the second-order component of the exponential map expansion and the naive second-order component of the expansion on the original flow
\begin{align*}
    \Fimp (x_{n})  = \frac{1}{2} \left[ \II_{x_{n}}[- \rgrad J (x_{n}), - \rgrad J (x_{n})]\nu - \nabla (\rgrad J(x_{n})) (\rgrad J(x_{n})) \right] 
\end{align*}
To simplify this and express the regularization intrinsically, and introducing $\nabla_{X}$ as the covariant derivative with respect to vector field $X$, we note that
\begin{align*}
    \II_{x_{n}}[- \rgrad J (x_{n}), - \rgrad J (x_{n})]\nu = D_{(- \rgrad J (x_{n}))} {(- \rgrad J (x_{n}))} - \nabla_{(- \rgrad J (x_{n}))} (- \rgrad J (x_{n}))
\end{align*}
by linearity of the covariant and directional derivatives, one has
\begin{align*}
    = D_{\rgrad J (x_{n})} { \rgrad J (x_{n})} - \nabla_{\rgrad J (x_{n})} \rgrad J (x_{n})
\end{align*}
and as $\gamma_{v}$ defines a geodesic, one has that the covariant derivative vanishes as $\nabla_{\dot{\gamma}_{v}} \dot{\gamma}_{v} = 0$, and identifies that $D_{\Dot{\gamma}_{v}} \Dot{\gamma}_{v} = \Ddot{\gamma}_{v}$. Thus, defining $x^{i}(t) = \gamma_{v}(t)$ as our coordinates, and writing the standard geodesic equation (for $f$ the hypersurface element for the manifold $M$), it immediately follows that
\begin{align*}
    & D_{\Dot{\gamma}} \Dot{\gamma} - \nabla_{\Dot{\gamma}} \Dot{\gamma} = \sum_{k} \Ddot{x}^{k}(t) \frac{\partial f}{\partial x^{k}} =D_{\Dot{\gamma}} \Dot{\gamma} - \sum_{k} \biggl( \Ddot{x}^{k}(t) + \sum_{ij} \Dot{x}^{i} (t) \Dot{x}^{j}(t) \Gamma_{ij}^{k}(\gamma(t)) \biggr) \frac{\partial f}{\partial x^{k}} \\
    & =\sum_{k} \Ddot{x}^{k}(t) \frac{\partial f}{\partial x^{k}} - \sum_{k} \biggl( \Ddot{x}^{k}(t) + \sum_{ij} \Dot{x}^{i} (t) \Dot{x}^{j}(t) \Gamma_{ij}^{k}(\gamma(t)) \biggr) \frac{\partial f}{\partial x^{k}} \\
    & = - \sum_{ijk} \Dot{x}^{i} (t) \Dot{x}^{j}(t) \Gamma_{ij}^{k}(\gamma(t)) \frac{\partial f}{\partial x^{k}} 
\end{align*}
Now, we evaluate this normal component to the acceleration at time $t=0$ to evaluate the Taylor expansion of the exponential map. By noting that
\[
\gamma(0) := x_{n}, \quad \gamma'(0) = \Dot{x}(0) := \rgrad J (x_{n})
\]
One finds the regularization is
\begin{align}
    \left[ - \sum_{ijk} \Dot{x}^{i} (t) \Dot{x}^{j}(t) \Gamma_{ij}^{k}(\gamma(t)) \frac{\partial f}{\partial x^{k}} \right]_{t=0} = - \sum_{ijk} {\rgrad J (x_{n})}^{i} \rgrad J (x_{n})^{j} \Gamma_{ij}^{k}(x_{n}) \partial_{k}
\end{align}
Thus, the implicit regularizer is written coordinate-wise in $k$ as
\begin{align*}
     \Fimp_{k} (x_{n}) & =   - \frac{1}{2} \sum_{ij} \Gamma_{ij}^{k} \rgrad J(x_{n})^{i} \rgrad J(x_{n})^{j} \\ & \qquad \qquad -\frac{1}{2}  (\nabla \rgrad J(x_{n})\rgrad J(x_{n}))^{k} \\
    &= - \frac{1}{2} \left( \Gamma_{ij}^{k}  \Dot{X}^{i} \Dot{X}^{j} +  (\nabla \Dot{X} )^{ki} \Dot{X}_{i} \right) 
\end{align*}
And thus, for $\mathrm{Hess}_{g} \, J$ the Hesse $(1,1)$-tensor, one has
\begin{align*}
    = - \frac{1}{2} \left( 
    \mathrm{Hess}_{g} \,J \, [ \dot{X} ] - 
    \Gamma_{ij}^{k}  \Dot{X}^{i} \Dot{X}^{j} \right)
\end{align*}
where we use the notation of Equation~\ref{eq:RGF} that $\Dot{X} = - \rgrad J (X)$ is the solution to the continuous Riemannian gradient flow, to distinguish it from the modified flow $\Dot{\tilde{x}} = \mathcal{F}(\tilde{x})$. One may rewrite this flow concisely with a simple operator which does not require evaluation of the Christoffel symbols $\Gamma_{ij}^{k}$. Note that the Christoffel symbols of the second kind, $\Gamma_{ij}^{k}$, are given as
\begin{align*}
    \Gamma_{ij}^{k} = \frac{1}{2} g^{km} \left(
    \partial_{j} g_{mi} + \partial_{i} g_{mj} - \partial_{m} g_{ij}
    \right)
\end{align*}
so that we have
\begin{align*}
    \Dot{X}^{i} \Dot{X}^{j} \Gamma_{ij}^{k} = \frac{1}{2} g^{km} \left(
    \partial_{j} g_{mi} \Dot{X}^{i} \Dot{X}^{j} + \partial_{i} g_{mj} \Dot{X}^{i} \Dot{X}^{j} - \partial_{m} g_{ij} \Dot{X}^{i} \Dot{X}^{j}
    \right)
\end{align*}
where contraction against the $ij$ indices implies the symmetry of the terms
\begin{align*}
    \partial_{j} g_{mi} \Dot{X}^{i} \Dot{X}^{j} = \partial_{i} g_{mj} \Dot{X}^{i} \Dot{X}^{j} 
\end{align*}
so the above expression becomes
\begin{align*}
    \Dot{X}^{i} \Dot{X}^{j} \Gamma_{ij}^{k} =  g^{km} \left[ \partial_{i} g_{mj} \Dot{X}^{i} \Dot{X}^{j} - \frac{1}{2} \partial_{m} g_{ij} \Dot{X}^{i} \Dot{X}^{j} \right]
\end{align*}
Returning to $(\nabla \Dot{X} )^{ki} \Dot{X}_{i}$ and applying $g^{km} g_{mj} = \delta^{k}_{j}$ we find
\begin{align*}
    (\nabla \Dot{X} )^{ki} \Dot{X}_{i} = g^{km} g_{mj} (\nabla \Dot{X} )^{ji} \Dot{X}_{i} = g^{km} g_{mj} \Ddot{X}^{j}
\end{align*}
So the implicit regularization may be expressed as
\begin{align*}
    - \frac{1}{2} \left( \Gamma_{ij}^{k}  \Dot{X}^{i} \Dot{X}^{j} +  (\nabla \Dot{X} )^{ki} \Dot{X}_{i} \right) &= - \frac{1}{2} g^{km} \left[ \partial_{i} g_{mj} \Dot{X}^{i} \Dot{X}^{j} - \frac{1}{2} \partial_{m} g_{ij} \Dot{X}^{i} \Dot{X}^{j} + g_{mj} \Ddot{X}^{j} \right] \\
    &= - \frac{1}{2} g^{km} \left[ \partial_{i} g_{mj} \Dot{X}^{i} \Dot{X}^{j} + g_{mj} \Ddot{X}^{j} - \frac{\partial}{\partial X^{m}} \left( \frac{1}{2} g_{ij} \Dot{X}^{i} \Dot{X}^{j} \right) \right] \\
    &= - \frac{1}{2} g^{km} \left[ \frac{\partial g_{mj}}{\partial X^{i}} \frac{d X^{i}}{d t} \Dot{X}^{j} + g_{mj} \frac{d}{d t} \Dot{X}^{j} - \frac{\partial}{\partial X^{m}} \left( \frac{1}{2} g_{ij} \Dot{X}^{i} \Dot{X}^{j} \right) \right] \\
    &= - \frac{1}{2} g^{km} \left[ \frac{d}{dt} \left( g_{mj} \Dot{X}^{j} \right) - \frac{\partial}{\partial X^{m}} \left( \frac{1}{2} g_{ij} \Dot{X}^{i} \Dot{X}^{j} \right) \right] \\
    &= \frac{1}{2} g^{km} \left[ 
    \frac{\partial}{\partial X^{m}} \left( \frac{1}{2} g_{ij} \Dot{X}^{i} \Dot{X}^{j} \right)
    -    \frac{d}{dt} \frac{\partial}{\partial \Dot{X}^{m}} \left( \frac{1}{2} g_{ij} \Dot{X}^{i} \Dot{X}^{j} \right)  \right] \\
    &= \frac{1}{2} g^{km} \left( 
    \frac{\partial}{\partial X^{m}} 
    -    \frac{d}{dt} \frac{\partial}{\partial \Dot{X}^{m}}  \right) \left( \frac{1}{2} g_{ij} \Dot{X}^{i} \Dot{X}^{j} \right)\\
    &:= \frac{1}{2} g^{km} \left( 
    \frac{\partial}{\partial X^{m}} 
    -    \frac{d}{dt} \frac{\partial}{\partial \Dot{X}^{m}}  \right) \mathcal{L}
\end{align*}
Where $( \partial_{X^{m}} -    \frac{d}{dt} \partial_{ \Dot{X}^{m}} )$ denotes the Euler-Lagrange operator on the $m$-th component and $g^{km}$ is the contravariant metric raise, as in standard Riemannian descent on $M$. The flow is given by a descent in the EL operator on the Lagrangian $\mathcal{L} = \frac{1}{2} g_{ij} \Dot{X}^{i} \Dot{X}^{j}$ which denotes the kinetic energy with respect to a Riemannian metric $g_{ij}$. Since we have that $\Dot{\tilde{x}}(t) \simeq_{o(\eta)} \Dot{X}(t)$, it holds that for the kinetic energy
\begin{align*}
\frac{\eta}{4} g_{ij} \Dot{\tilde{x}}^{i} \Dot{\tilde{x}}^{j} = \frac{\eta}{4} g_{ij} (\Dot{X}^{i} + O(\eta) ) (\Dot{X}^{j} + O(\eta) ) = \frac{\eta}{4} g_{ij} \Dot{X}^{i} \Dot{X}^{j} + O(\eta^{2})
\end{align*}
thus the corrected flow may be entirely given in the modified variable $\tilde{x}$ as
\begin{align*}
 -\rgrad J (\tilde{x}) &= - g^{km} \partial_{m} J = \Dot{X} \\
 \Dot{\tilde{x}}^{k} &= - g^{km} \left( \frac{\partial}{\partial \tilde{x}^{m}} J(\tilde{x}) - \frac{\eta}{2} \left( 
    \frac{\partial}{\partial \tilde{x}^{m}} 
    -    \frac{d}{dt} \frac{\partial}{\partial \Dot{\tilde{x}}^{m}}  \right) \mathcal{L} \right) \\
     \Dot{\tilde{x}}^{k} &= - g^{km} \left( \frac{\partial}{\partial \tilde{x}^{m}} J(\tilde{x}) - \left( 
    \frac{\partial}{\partial \tilde{x}^{m}} 
    -    \frac{d}{dt} \frac{\partial}{\partial \Dot{\tilde{x}}^{m}}  \right) \left[ \frac{\eta}{4} g_{ij} \Dot{\tilde{x}}^{i} \Dot{\tilde{x}}^{j} \right] \right) 
\end{align*}
Since a potential energy or loss function is independently defined only in the parameters $\tilde{x}$ as $J: M \to \mathbb{R}$ we have that
\[
\frac{\partial}{\partial \Dot{\tilde{x}}^{m}} J(\tilde{x}) = 0
\]
Thus, one may extend the Euler-Lagrange operator to both the potential and kinetic contributions
\begin{align*}
    &- g^{km} \left( \left( \frac{\partial}{\partial \tilde{x}^{m}} - \frac{d}{dt} \frac{\partial}{\partial \Dot{\tilde{x}}^{m}} \right) J(\tilde{x}) - \left( 
    \frac{\partial}{\partial \tilde{x}^{m}} 
    -    \frac{d}{dt} \frac{\partial}{\partial \Dot{\tilde{x}}^{m}}  \right) \left[ \frac{\eta}{4} g_{ij} \Dot{\tilde{x}}^{i} \Dot{\tilde{x}}^{j} \right] \right).
\end{align*}
We have 
\[
E_{m} = \left( 
    \frac{\partial}{\partial \tilde{x}^{m}} 
    -    \frac{d}{dt} \frac{\partial}{\partial \Dot{\tilde{x}}^{m}}  \right)
\]
is the Euler-Lagrange operator and a Step-Dependent Lagrangian,
\[
\mathcal{L}^{\eta}(\tilde{x}, \dot{\tilde{x}}) =  \underbrace{\frac{\eta}{4} g_{ij} \Dot{\tilde{x}}^{i} \Dot{\tilde{x}}^{j}}_{\textrm{Kinetic Bias }(T)} - \underbrace{J(\tilde{x})}_{\textrm{Potential }(V) 
}. 
\]
Here, we distribute the signs to maintain the standard notation that $\mathcal{L} = T - V$ for $T$ and $V$ representing kinetic and the potential energies. To order $\eta^{2}$, one then has that the modified continuous flow of Riemannian gradient descent is as an application of the Euler-Lagrange operator on $\mathcal{L}^{\eta}(\tilde{x}, \dot{\tilde{x}})$:
\begin{align*}
&\Dot{\tilde{x}}^{k} = \,g^{km}\,\Bigl(
\partial_{\tilde{x}^{m}}
-\tfrac{d}{dt}\,\partial_{\dot{\tilde{x}}^{m}}
\Bigr)
\Bigl[\frac{\eta}{4}\,g_{ij}\dot{\tilde{x}}^{i}\dot{\tilde{x}}^{j}\;-\;
J(\tilde{x})
\Bigr]
\;+\;O(\eta^{2}). \\
& = \,g^{km}E_{m}\mathcal{L}^{\eta}(\tilde{x}, \dot{\tilde{x}})\;+\;O(\eta^{2}).
\end{align*}
This concludes the proof.
\end{proof}

\textbf{Remark.}
This proposition recovers the gradient-descent regularization of \cite{smith2021on, barrett2021implicit}: for Euclidean space $(M, g) = (\mathbb{R}^{d}, \mathrm{id})$ the Riemannian gradient reduces to the Euclidean gradient so $\rgrad J (\tilde{x} ) =  \nabla J(\tilde{x})$ as $g_{ij} = \delta_{ij}$. 
\begin{align*}
    &- \nabla_{\mb{id}} J + \frac{\eta}{4}  (\mb{id})^{-1}\left( 
    \nabla_{\tilde{x}} - \frac{d}{dt} \nabla_{\Dot{\tilde{x}}}
    \right) \langle \dot{\tilde{x}}, (\mb{id})_{ij} \dot{\tilde{x}} \rangle + O(\eta^{2}) = - \nabla J + \frac{\eta}{4} \left( 
    \nabla_{\tilde{x}} - \frac{d}{dt} \nabla_{\Dot{\tilde{x}}}
    \right) \lVert \dot{\tilde{x}} \rVert_{2}^{2} + O(\eta^{2}) \\
    &= - \nabla J - \frac{\eta}{4}  
     \,\frac{d}{dt} \nabla_{\Dot{\tilde{x}}}\lVert \dot{\tilde{x}} \rVert_{2}^{2} + O(\eta^{2}) = - \nabla J - \frac{\eta}{4}  \nabla \lVert \nabla J \rVert_{2}^{2} + O(\eta^{2}),
\end{align*}
where we used that $\frac{d}{dt} \nabla_{\Dot{\tilde{x}}}\lVert \dot{\tilde{x}} \rVert_{2}^{2} = \frac{d}{dt} \dot{\tilde{x}}(\tilde{x}(t)) = \nabla \dot{\tilde{x}} \,\dot{\tilde{x}} = (1/2)\nabla \lVert \dot{\tilde{x}} \rVert_{2}^{2}$.

\section{Details on Numerical Validation}

\subsection{Bures-Wasserstein Dynamics}\label{sec:bw_numerics}

We numerically validate the second-order accuracy of the proposed JKO correction. We specialize to $M=\mathbb{R}^{3}$, and to avoid alignment of the drift with the coordinate axes generate a random Gaussian matrix $Z_{ij} \sim \mathcal{N}(0,1)$, orthogonalize $Z = Q R^{\top}$ via QR decomposition, and use the orthogonal factor $Q$ to rotate a diagonal drift spectrum
\[
\bm{d} =\begin{pmatrix}
    -0.2 \\ -0.6 \\ -1.2
\end{pmatrix}
\]
to form a random negative-definite drift $A \prec 0$
\[
A = Q \mathrm{diag}(\bm{d}) Q^{\top}
\]
As $A$ is real, symmetric, and strictly negative definite, we ensure the system dynamics \eqref{eq:SDE_FokkPlanck} are stable. This also prevents alignment of $A$ with the initial covariance $P_{0}$, which is more representative of a non-commuting flow. This covariance $P_{0}$ is generated by sampling a random Gaussian matrix $M_{ij} \sim \mathcal{N}(0,1)$ and constructing the positive-definite matrix
\[
P_{0} = MM^{\top} + \epsilon \mb{I}
\]
For $\epsilon = 0.5$ guaranteeing $P_{0} \succ 0$. We set $\beta = 1$ in \eqref{eq:SDE_FokkPlanck} and initialize $\bm{\mu}_{0} \sim \mathcal{N}(\mb{0}, \mb{I})$. We integrate Wasserstein-gradient flow, the second-order flow for $200$ RK4 steps and compare to the JKO step in $W_2$ distance, in $\ell_{2}$ error on the mean, and Frobenius-norm error on the covariance. Values for this experiment are reported in Table~\ref{tab:jko_error_scaling}. We additionally plot these curves in Figure~\ref{fig:jko-bw}, and show a non-symmetric and rotational set of Gaussian dynamics on JKO, Wasserstein gradient flow, and the second-order JKO flow in the Top Left sub-panel of Figure~\ref{fig:jko-bw}.

\begin{table}[ht]
\centering
\caption{Error scaling of vanilla and modified (JKO) flows with step size $\eta$.}
\begin{tabular}{c cccccc}
\toprule
$\eta$ 
 & $\|\mu_{\mathrm{v}}-\mu_{\mathrm{JKO}}\|$ 
 & $\|\mu_{\mathrm{m}}-\mu_{\mathrm{JKO}}\|$ 
 & $\|P_{\mathrm{v}}-P_{\mathrm{JKO}}\|_F$ 
 & $\|P_{\mathrm{m}}-P_{\mathrm{JKO}}\|_F$ 
 & $W_2(\mathrm{v},\mathrm{JKO})$ 
 & $W_2(\mathrm{m},\mathrm{JKO})$ \\
\midrule
0.250000 & $2.91{\times}10^{-2}$ & $5.62{\times}10^{-3}$ & $2.98{\times}10^{-1}$ & $6.60{\times}10^{-2}$ & $7.62{\times}10^{-2}$ & $1.95{\times}10^{-2}$ \\
0.125000 & $9.00{\times}10^{-3}$ & $8.68{\times}10^{-4}$ & $1.02{\times}10^{-1}$ & $1.12{\times}10^{-2}$ & $2.46{\times}10^{-2}$ & $3.35{\times}10^{-3}$ \\
0.062500 & $2.52{\times}10^{-3}$ & $1.22{\times}10^{-4}$ & $3.03{\times}10^{-2}$ & $1.67{\times}10^{-3}$ & $7.13{\times}10^{-3}$ & $5.23{\times}10^{-4}$ \\
0.031250 & $6.67{\times}10^{-4}$ & $1.62{\times}10^{-5}$ & $8.29{\times}10^{-3}$ & $2.29{\times}10^{-4}$ & $1.94{\times}10^{-3}$ & $7.52{\times}10^{-5}$ \\
0.015625 & $1.72{\times}10^{-4}$ & $2.09{\times}10^{-6}$ & $2.17{\times}10^{-3}$ & $3.02{\times}10^{-5}$ & $5.07{\times}10^{-4}$ & $1.02{\times}10^{-5}$ \\
0.007812 & $4.36{\times}10^{-5}$ & $2.66{\times}10^{-7}$ & $5.54{\times}10^{-4}$ & $3.87{\times}10^{-6}$ & $1.30{\times}10^{-4}$ & $1.33{\times}10^{-6}$ \\
\bottomrule
\end{tabular}
\label{tab:jko_error_scaling}
\end{table}

\begin{figure}
    \centering
    \includegraphics[width=1.0\linewidth]{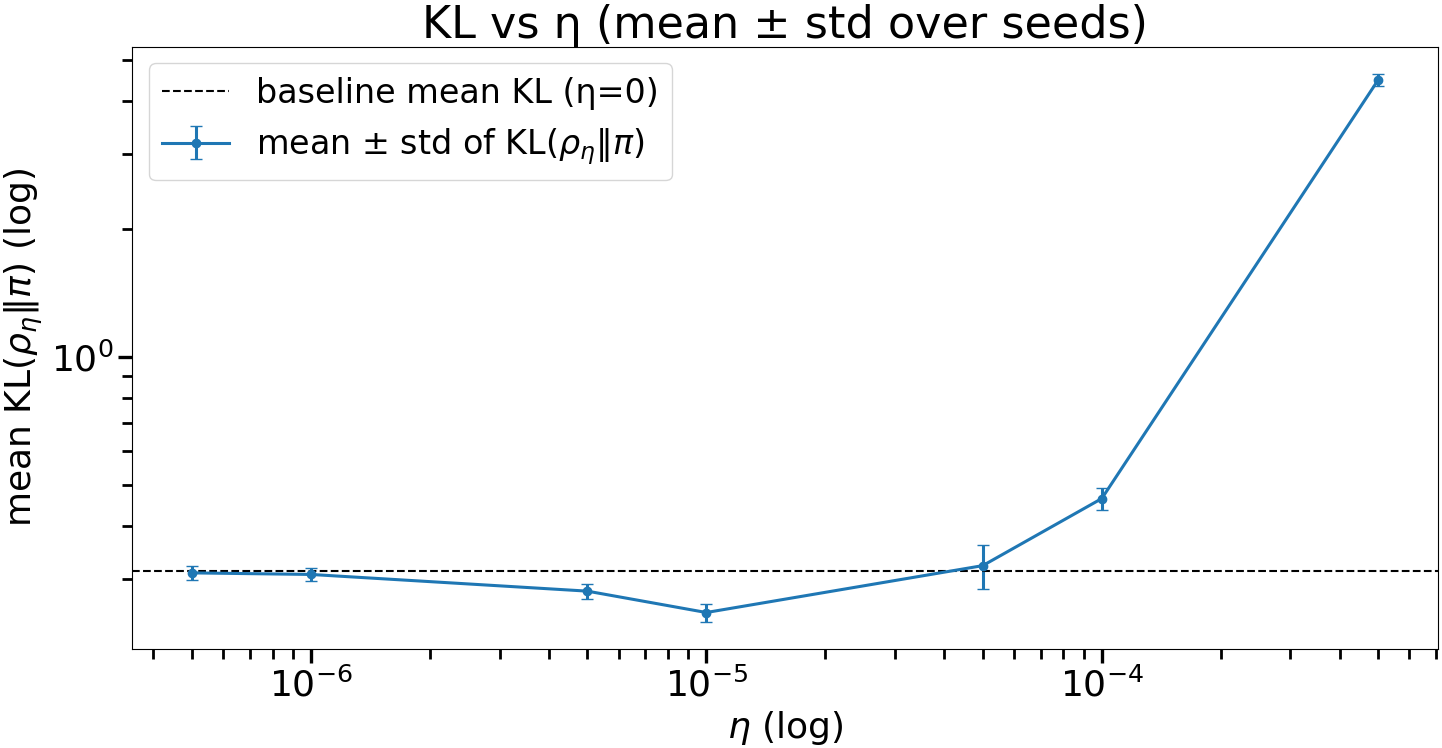}
    \caption{KL-distance between target measure and both JKO-flow and Wasserstein gradient flow numerically integrated for $2000$ steps with $h=2e-3$ and a range of $\eta$ values. Mean and standard deviation are over $500$ random seeds. Several JKO strengths $\eta$ in the range $[1e-6, 1e-4]$ show a consistent improvement over the $\eta=0$ baseline.}
    \label{fig:jko-wgf}
\end{figure}

\end{document}